\documentclass{article}

\usepackage{arxiv}

\usepackage[utf8]{inputenc}
\usepackage[T1]{fontenc}
\usepackage{amsmath,amssymb,amsfonts,mathtools}
\usepackage{amsthm}
\usepackage{booktabs}
\usepackage{placeins}
\usepackage{longtable}
\usepackage{multirow}
\usepackage{array}
\usepackage{graphicx}
\usepackage[caption=false,font=footnotesize]{subfig}
\usepackage{xcolor}
\usepackage{url}
\usepackage[hypertexnames=false]{hyperref}
\usepackage[numbers]{natbib}
\usepackage{doi}
\usepackage{algorithm}
\usepackage{algpseudocode}
\usepackage{enumitem}
\usepackage{microtype}
\usepackage{listings}
\usepackage{needspace}
\usepackage{cleveref}

\title{Metadata-Free Meta-Reweighted Direct Preference Optimization under Noisy Preference Labels}

\author{
\normalfont Hua Qu\\
Xi'an Jiaotong University\\
\texttt{qh@mail.xjtu.edu.cn}
\and
\normalfont Yifan Li\\
Xi'an Jiaotong University\\
\texttt{4123158005@stu.xjtu.edu.cn}
\and
\normalfont Xiaodong Yuan\\
Xi'an Jiaotong University\\
\texttt{xiaodongyuan@163.com}
}

\date{}

\hypersetup{
pdftitle={Metadata-Free Meta-Reweighted Direct Preference Optimization under Noisy Preference Labels},
pdfsubject={Direct Preference Optimization; noisy preferences; meta-learning},
pdfauthor={Hua Qu, Yifan Li, Xiaodong Yuan},
pdfkeywords={Direct Preference Optimization, preference noise, meta-learning, prompt augmentation consistency, finite difference, LoRA}
}

\newcommand{\Dtrain}{D_{\mathrm{train}}}
\newcommand{\Ptrain}{P_{\mathrm{train}}}
\newcommand{\Pmeta}{P_{\mathrm{meta}}}
\newcommand{\piref}{\pi_{\mathrm{ref}}}
\newcommand{\logit}{\operatorname{logit}}
\newcommand{\E}{\mathbb{E}}
\newcommand{\Prob}{\mathbb{P}}
\newcommand{\R}{\mathbb{R}}
\newcommand{\Z}{\mathcal{Z}}
\newcommand{\Acal}{\mathcal{A}}

\newtheorem{theorem}{Theorem}
\newtheorem{assumption}{Assumption}
\newtheorem{proposition}{Proposition}

\begin{document}
\maketitle

\begin{abstract}
Direct Preference Optimization (DPO) has become an important method for aligning large language models (LLMs) with human preferences because it removes the need for explicit reward modeling and reinforcement learning. However, its performance depends heavily on the quality of preference data, and noisy preference data in real-world settings can weaken alignment performance. To address this issue, we propose a bilevel optimization framework and prove, under some idealized conditions, that this framework can recover the DPO optimum under clean data. We further derive a prior form for the learnable weighting function under label-flipping noise. Considering that high-quality metadata may be difficult to obtain, we propose a prompt augmentation consistency method that enables meta-learning even when metadata is completely unavailable. To reduce the high cost of higher-order gradients in LLM meta-learning, we combine central-difference approximation with LoRA fine-tuning and develop a scalable training scheme. Experiments on TL;DR summarization and Anthropic Helpful and Harmless dialogue show that the proposed method improves alignment performance over multiple DPO baselines under different noise rates.
\end{abstract}


\section{Introduction}

Aligning large language models (LLMs) with human preferences is important \cite{touvron2023llama,team2023gemini}. Reinforcement learning from human feedback (RLHF) has been shown to be effective for aligning LLMs with human preferences \cite{ouyang2022training}. A typical RLHF pipeline first obtains an SFT policy through supervised fine-tuning, then learns a reward model (RM) from human preference data, and finally optimizes the policy with reinforcement learning algorithms such as PPO while constraining its distance from a reference model to avoid excessive policy drift.

Although RLHF plays an important role in LLM alignment, its training pipeline is relatively complex. It requires training an additional reward model and policy model, and the reinforcement learning stage involves frequent sampling, which leads to high computational and memory costs \cite{kaufmann2023survey,zheng2023secrets}. Direct Preference Optimization (DPO) provides a simpler alternative for preference alignment: it directly optimizes the LLM from human preference data without explicitly learning a reward model and avoids the complexity of reinforcement learning \cite{rafailov2023direct}.

Despite greatly simplifying the traditional RLHF pipeline, DPO still depends heavily on the quality of preference data. Prior studies have shown that preference-data noise can degrade both the stability and final performance of alignment training \cite{gao2024impact,wang2024secrets}. Therefore, effective DPO learning under noisy preference data has become an important problem in preference alignment.

Existing work has mainly addressed this problem by improving the noise robustness of DPO through loss correction, label smoothing, distributionally robust optimization, or noisy-sample filtering \cite{chowdhury2024provably,wu2025towards,bukharin2024robust,mitchell2023note,liang2024ropo,kong2024perplexity}. However, these methods often rely on predefined noise forms, fixed reweighting rules, or heuristic reliability estimates, which makes it difficult to simultaneously maintain interpretability and training effectiveness under complex noise conditions.

Recently, several studies have introduced meta-learning into preference optimization to adaptively characterize noisy preference samples and learn sample-level weights \cite{li2026aligner,li2026learning}. These methods show that meta-learning can provide a more flexible mechanism for weighting noisy preference samples than fixed heuristic rules.

However, existing meta-learning-based weighted preference optimization methods mainly learn sample weights from training dynamics, model diagnostic signals, or validation-set performance. They do not systematically explain, from the perspective of the discrepancy between noisy and clean conditional preference posteriors, why a bilevel weighting mechanism can correct the optimal-solution shift caused by noisy preferences in DPO. Moreover, these methods usually still require a validation set with relatively reliable preference labels or a small amount of clean metadata to provide the outer-level optimization signal. In real LLM alignment scenarios, constructing high-quality meta-preference data often requires additional annotation, model review, or manual cleaning, which can be costly or even unavailable.

To this end, we revisit DPO learning under noisy preferences from the perspectives of conditional risk and posterior shift, and propose a prompt-augmentation-consistency meta-reweighted DPO method that does not rely on clean meta-preference labels, termed PACMR-DPO. The contributions of this paper are summarized as follows:

\begin{itemize}[leftmargin=*]
    \item We prove that, under idealized conditions, bilevel optimization can recover the DPO optimum under clean data from noisy training data, thereby providing a theoretical basis for learnable weight parameterization.
    \item We propose a prompt-augmentation-consistency meta-reweighted DPO (PACMR-DPO) method that does not require clean metadata. Instead of relying on clean meta-preference labels, the method injects task-agnostic prompt augmentation consistency as meta-knowledge into the outer-level optimization.
    \item To address the high cost of higher-order gradients in LLM meta-learning, we propose a scalable training scheme that combines central-difference approximation with perturbations in the LoRA parameter space, reducing the memory and computation overhead of bilevel optimization.
    \item We evaluate the proposed method on TL;DR summarization and Anthropic HH dialogue, and compare it with multiple DPO baselines under different random-flip noise rates. Experiments show that PACMR-DPO achieves more pronounced improvements under medium and high noise rates.
\end{itemize}

\section{Related Work}

\subsection{Direct Preference Optimization}

RLHF has become an important approach for aligning LLMs with human preferences \cite{ouyang2022training}, but its training process usually consists of SFT, RM learning, and reinforcement-learning fine-tuning, making it complex and computationally demanding. Against this background, DPO provides a simpler alternative to RLHF \cite{rafailov2023direct}. For preference data $(x,y_a,y_b)$, the DPO loss is defined as
\begin{align}
\mathcal{L}_{\mathrm{DPO}}(\omega)
&=-\E_{\substack{(x,y_a,y_b)\\\sim \Dtrain}}
\Bigg[\log\sigma\Bigg(
\beta \log\frac{\pi_{\omega}(y_a\mid x)}{\piref(y_a\mid x)}
\nonumber\\[-0.25ex]
&\hspace{7.2em}
-\beta \log\frac{\pi_{\omega}(y_b\mid x)}{\piref(y_b\mid x)}
\Bigg)\Bigg].
\end{align}
where $\sigma(r)=1/(1+\exp(-r))$ is the sigmoid function and $\beta>0$ is the DPO temperature parameter.

DPO avoids explicit reward modeling and reinforcement learning, thereby substantially reducing the training cost of preference alignment.

\subsection{Preference Learning with Noisy Labels}

Prior studies have shown that preference data used for LLM training can contain a non-negligible amount of noise, and that alignment performance decreases as the noise rate increases \cite{gao2024impact}. To address this issue, cDPO smooths the hard labels in DPO into soft labels using the noise rate, reducing the damage caused by incorrectly labeled preference samples \cite{mitchell2023note}; rDPO constructs a debiased weighted loss to offset the effect of incorrect preference annotations on DPO training \cite{chowdhury2024provably}; Dr.DPO combines distributionally robust optimization with DPO to improve robustness to noisy data \cite{wu2025towards}; ROPO suppresses gradients from high-uncertainty samples and alternates between noise-tolerant training and noisy-sample filtering to improve robustness \cite{liang2024ropo}; R3M models corrupted preference labels as sparse outliers \cite{bukharin2024robust}; and PerpCorrect detects and corrects noisy labels through perplexity differences within preference pairs \cite{kong2024perplexity}.

Unlike these methods, we model sample reliability in noisy preference learning as a learnable sample-level weighting function, and use bilevel optimization to adaptively adjust the influence of training samples on LLM updates according to an outer-level objective.

\subsection{Meta Learning for Preference Reweighting}

MetaPO\cite{li2026learning} argues that sample value is not determined solely by the current model state, but changes dynamically during training. It therefore uses three temporal dynamic features - reward-margin evolution, learning fluctuation, and deviation from the reference model - to learn time-aware weights for each preference sample . Aligner, Diagnose Thyself\cite{li2026aligner} argues that preference-sample reliability cannot be fully captured by a single heuristic indicator. It constructs diagnostic vectors from preference consistency, learning difficulty, and generation confidence, and uses a meta-learning network to fuse these intrinsic feedback signals into sample-level weights.

Unlike these methods, this paper does not primarily rely on empirical dynamic features or diagnostic signals to define sample reliability. Instead, it starts from the posterior shift between noisy and clean preferences and analyzes why bilevel optimization can correct the optimal-solution shift caused by noisy preferences in DPO. Based on this analysis, we further design a prior for the learnable weighting function and replace the clean-metadata-dependent outer supervision objective with a prompt-augmentation-consistency objective.

\subsection{Meta Learning Sample Reweighting under Noisy Supervision}

In noisy-label learning, meta-learning-based sample reweighting is an important class of methods. Meta-Weight-Net (MWN)\cite{shu2019meta} adaptively reduces the influence of noisy or difficult samples by learning an explicit mapping from sample loss to sample weight . CMW-Net\cite{shu2023cmw} extends MWN by learning class-aware sample-weighting functions to explicitly model class-dependent heterogeneous noise and imbalance . DAC-MR\cite{shu2026dac} addresses the difficulty of obtaining high-quality metadata, which may be noisy or even unavailable, by introducing data augmentation consistency as an alternative meta-objective and using task-agnostic meta-knowledge to compensate for traditional metadata-driven meta-learning . 

We transfer this idea to DPO-based preference optimization and construct a prompt augmentation consistency objective around preference pairs.

\section{Problem Formulation}

\subsection{Noisy Pairwise Preference Setting}

Let $\mathcal{Z}$ denote the space of preference examples, where
each element is a triplet $z=(x,y_a,y_b)$ consisting of a prompt
$x$ and two candidate responses $y_a$ and $y_b$. Let $Z$ be a
random variable taking values in $\mathcal{Z}$, and let $z$
denote a realization of $Z$. Let the latent clean preference label be $Y\in\{0,1\}$, where $Y=0$ means $y_a\succ y_b$ and $Y=1$ means $y_a\prec y_b$. In practice, the model often observes a noisy label $\widetilde{Y}\in\{0,1\}$ rather than the clean label.

To describe the sampling process for the training and metadata sets, we use $m=0$ to indicate that a sample is selected into the metadata set, and $t=0$ to indicate that a sample is selected into the training set.

\begin{assumption}[Label-independent sampling]
\label{assump:sample-selection}
We assume that the probability of a sample entering the metadata set or training set is independent of the latent clean preference label, and that every sample has positive probability of being selected into both sets. Specifically, for all $z\in\Z$ and $y\in\{0,1\}$,
\begin{align}
\Prob(m=0\mid Y=y,Z=z)&=\Prob(m=0\mid Z=z)>0,\\
\Prob(t=0\mid Y=y,Z=z)&=\Prob(t=0\mid Z=z)>0.
\end{align}
\end{assumption}

Define
\begin{align}
\eta(z)&=\Prob(Y=0\mid Z=z)\in(0,1),\\
\eta_{\mathrm{meta}}(z)&=\Prob(Y=0\mid Z=z,m=0)\in(0,1),\\
\eta_{\mathrm{train}}(z)&=\Prob(Y=0\mid Z=z,t=0)\in(0,1).
\end{align}
Under \cref{assump:sample-selection}, the clean preference posterior is identical under the population distribution, training distribution, and meta distribution.

\begin{proposition}[Posterior invariance under label-independent sampling]
\label{prop:posterior-invariance}
Under \cref{assump:sample-selection}, for any $z\in\Z$, we have
\begin{equation}
\eta_{\mathrm{train}}(z)=\eta_{\mathrm{meta}}(z)=\eta(z).
\end{equation}
\end{proposition}

\begin{proof}
By Bayes' rule and \cref{assump:sample-selection},
{\footnotesize
\begin{align}
\eta_{\mathrm{meta}}(z)
&=\Prob(Y=0\mid Z=z,m=0)\\
&=\frac{\Prob(m=0\mid Y=0,Z=z)\Prob(Y=0\mid Z=z)}
{\sum_{y\in\{0,1\}}\Prob(m=0\mid Y=y,Z=z)\Prob(Y=y\mid Z=z)}\\
&=\frac{\Prob(m=0\mid Z=z)\eta(z)}
{\Prob(m=0\mid Z=z)[\eta(z)+1-\eta(z)]}\\
&=\eta(z).
\end{align}
}
Similarly, replacing $m$ with $t$ gives $\eta_{\mathrm{train}}(z)=\eta(z)$. Therefore, $\eta_{\mathrm{train}}(z)=\eta_{\mathrm{meta}}(z)=\eta(z)$.
\end{proof}

We further consider a general noisy conditional probability $q:\Z\to(0,1)$, where
\begin{equation}
q(z)=\Prob(\widetilde{Y}=0\mid Z=z,t=0),
\end{equation}
which denotes the conditional probability that sample $z$, after being selected into the training set and processed by the noise mechanism, is observed with label $\widetilde{Y}=0$. In preference learning, $\eta(z)$ is the conditional preference distribution over sample pair $z$, and both reward modeling and DPO can be viewed as estimating this conditional preference distribution.

\subsection{DPO as Conditional Preference Risk}
Let the implicit reward difference in DPO be 
\begin{equation}
    u(z)=\beta(\log\frac{\pi(y_a\mid x)}{\piref(y_a\mid x)}-\log\frac{\pi(y_b\mid x)}{\piref(y_b\mid x)})\in\R.
\end{equation}
Under clean preference training data, the DPO conditional risk can be written as (because $\eta_{\mathrm{train}}(z)=\eta(z)$)
\begin{align}
\mathcal{L}_{\mathrm{DPO}}(u(z),\eta(z))
={}&-\eta(z)\log\sigma(u(z))\nonumber\\
&-(1-\eta(z))\log\sigma(-u(z)).
\label{eq:dpo-conditional-risk}
\end{align}
The conditional risk is the pointwise risk obtained by fixing the sample pair $z$ and taking the expectation only over the conditional distribution of the latent preference label $Y$. Minimizing this risk yields the pointwise optimal implicit reward difference under the clean preference distribution,
\begin{equation}
 u^{*}_{\mathrm{DPO\text{-}clean}}(z)=\log\frac{\eta(z)}{1-\eta(z)}.
 \label{eq:clean-dpo-optimum}
\end{equation}
Correspondingly, under the noisy posterior $q(z)$, the pointwise optimum of standard DPO becomes
\begin{equation}
 u^{*}_{\mathrm{DPO\text{-}noisy}}(z)=\log\frac{q(z)}{1-q(z)}.
 \label{eq:noisy-dpo-optimum}
\end{equation}
This shows that when $q(z)\neq\eta(z)$, the implicit reward difference learned by standard DPO under the noisy distribution is shifted relative to that under the clean distribution. This result is similar to previous studies\cite{liang2024ropo}. The meaning of the conditional risk and the derivation of the optimum are given in Appendix~\ref{app:conditional-risk}.

\subsection{Sample Weighted DPO}

To correct this shift, we introduce a sample-dependent positive weighting function $g(z)$. If the observed label is $\widetilde{Y}=0$, we denote the weight as $g^+(z)$; if the observed label is $\widetilde{Y}=1$, we denote the weight as $g^-(z)$. The weighted DPO conditional risk is then
\begin{align}
&\mathcal{L}_{\mathrm{W\text{-}DPO}}
\bigl(u(z),q(z),g^+(z),g^-(z)\bigr)
\nonumber\\
&\quad=-q(z)g^+(z)\log\sigma(u(z))
\nonumber\\
&\qquad -(1-q(z))g^-(z)\log\sigma(-u(z)).
\end{align}
Its optimum is
\begin{equation}
u^{*}_{\mathrm{W\text{-}DPO\text{-}noisy}}(z)=\log\frac{q(z)}{1-q(z)}+\log\frac{g^+(z)}{g^-(z)}.
\end{equation}
Therefore, if the weight ratio $g^+(z)/g^-(z)$ is properly designed, it can offset the optimal-solution shift induced by noise.

\section{Theoretical Analysis}

\subsection{Clean Optimum Recovery by Bilevel Reweighting}

Based on the above observation, we formulate noisy DPO as a bilevel optimization problem: the inner level learns a weighted DPO solution under the noisy training distribution, while the outer level constrains the weighting function through a meta-objective.
\begin{align}
\min_{g^+>0,g^->0}\quad
&\E_{z'\sim\Pmeta}\Big[
-\eta_{\mathrm{meta}}(z')
\log\sigma\bigl(u^*(g^+,g^-,z')\bigr)
\nonumber\\
&\quad -(1-\eta_{\mathrm{meta}}(z'))
\log\sigma\bigl(-u^*(g^+,g^-,z')\bigr)\Big]
\nonumber\\
\mathrm{s.t.}\quad
&u^*(g^+,g^-,z)=\arg\min_u
\E_{z\sim\Ptrain}\Big[
\nonumber\\[-0.25ex]
&\quad -q(z)g^+(z)\log\sigma(u(z))
\nonumber\\[-0.25ex]
&\quad -(1-q(z))g^-(z)\log\sigma(-u(z))\Big].
\end{align}
The following theorem shows that, under idealized conditions, the optimal weights obtained by this bilevel optimization can recover the pointwise optimum of clean DPO.
\begin{theorem}[Clean-optimum recovery by ideal bilevel reweighting]
\label{thm:clean-recovery}
Under \cref{assump:sample-selection}, let $(g^{*,+},g^{*,-})$ be an optimal solution to the bilevel optimization problem, and let $u^*(g^{*,+},g^{*,-},z)$ be the inner-level pointwise optimal implicit reward difference induced by it. Then
\begin{equation}
 u^*(g^{*,+},g^{*,-},z)=
 u^{*}_{\mathrm{DPO\text{-}clean}}(z)=
 \log\frac{\eta(z)}{1-\eta(z)}.
\end{equation}
Equivalently, the log-ratio of the optimal weights obtained by bilevel optimization satisfies
\begin{equation}
\log\frac{g^{*,+}(z)}{g^{*,-}(z)}=
\log\frac{\eta(z)}{1-\eta(z)}-
\log\frac{q(z)}{1-q(z)}.
\label{eq:main-optimal-weight-ratio}
\end{equation}
\end{theorem}

\noindent
\Cref{thm:clean-recovery} shows that, in the ideal case, the log-ratio of the optimal weights obtained by bilevel optimization can cancel the noise-induced shift. The complete bilevel objective, the pointwise separability argument, and the proof are given in Appendix~\ref{app:clean-recovery}.

\subsection{Weight-Prior Construction under a General Label-Flipping Model}

Although the noisy preference posterior $q(z)$ and the clean preference posterior $\eta(z)$ are usually unknown, so that the corresponding optimal weighting function cannot be directly constructed, we can use the form of the theoretically optimal weight log-ratio to design a general prior for the learnable weighting function.

Under the general two-rate label-flipping model, the observed posterior is,
\begin{equation}
q(z)=(1-\varepsilon_0(z))\eta(z)+\varepsilon_1(z)(1-\eta(z)),
\label{eq:main-asym-noise}
\end{equation}
where $\varepsilon_0(z)=\Prob(\widetilde{Y}=1\mid Y=0,z)$ and $\varepsilon_1(z)=\Prob(\widetilde{Y}=0\mid Y=1,z)$. By \cref{thm:clean-recovery}, the optimality condition only requires the weight ratio $g^+(z)/g^-(z)$ to satisfy \cref{eq:main-optimal-weight-ratio}. Therefore, the optimal weight construction is not unique.

We adopt a class of bounded sigmoid weight constructions that separately extract the clean preference strength $\logit(\eta(z))=\log\frac{\eta(z)}{1-\eta(z)}$ and separate noise-rate-related terms from mixed terms that depend on both preference strength and noise rate. Specifically, let $s(z)=\logit(\eta(z)),\boldsymbol{\varepsilon}(z)
=
\bigl(\varepsilon_0(z),\varepsilon_1(z)\bigr)$. We can construct
\begin{align}
 g^{*,+}(z)&=\sigma\left(a_+(s(z),\boldsymbol{\varepsilon}(z))s(z)+b_+(\boldsymbol{\varepsilon}(z))\right),\\
 g^{*,-}(z)&=\sigma\left(-a_-(s(z),\boldsymbol{\varepsilon}(z))s(z)+b_-(\boldsymbol{\varepsilon}(z))\right),
\end{align}
so that the optimal weight ratio is satisfied. The complete construction and derivation are given in Appendix~\ref{app:asym-weight}.

\section{PACMR-DPO: Prompt Augmentation Consistency Meta Reweighted DPO}

\subsection{Learnable Weight Function}

According to the theoretical analysis, under the optimal weights we have
\begin{equation}
u^{*}_{\mathrm{W\text{-}DPO\text{-}noisy}}(z)=u^{*}_{\mathrm{DPO\text{-}clean}}(z)=s(z).
\end{equation}

Therefore, although samples are drawn from the noisy training set, we can still use the current model's implicit reward difference $u(z)$ as an approximation of the clean preference strength. In addition to the implicit reward difference $u(z)$, we also consider the implicit reward sum $\Delta(z)$. On the one hand, $\Delta(z)$ can reflect the overall quality of the two candidate responses. On the other hand, the combination of $u(z)$ and $\Delta(z)$ contains the implicit reward information of both candidate responses and may carry signals related to the sample noise rate.

Moreover, as the preference relations in the actual training data are provided, the two branches $g^+$ and $g^-$ can be unified into a single weighting network.
Based on this observation, we learn a mapping
\begin{equation}
N_{\Theta}:[\Delta(z),u(z)]\mapsto [a(z),b(z)],
\end{equation}
and express the learnable weighting function as
\begin{equation}
g(z;\Theta)=\sigma(a(z;\Theta)u(z)+b(z;\Theta)),
\end{equation}
where $N_{\Theta}$ is an MLP and $\Theta$ denotes its parameters. Since $u(z)$ is not equal to $\logit(\eta(z))$ during training, we use $a(z;\Theta)$ and $b(z;\Theta)$ to dynamically calibrate $u(z)$.

\subsection{Prompt Augmentation Consistency without Clean Meta-Preference Labels}

To apply the bilevel optimization framework in actual training, an effective outer-level optimization signal is required. Traditional meta-learning frameworks usually use a small amount of clean metadata to update the meta-network \cite{franceschi2018bilevel,shu2019meta,hospedales2021meta}. However, in LLM preference alignment, high-quality clean meta-preference data may be difficult to obtain. Therefore, we replace the outer-level objective with a prompt augmentation consistency objective, the inspiration was drawn from \cite{shu2026dac}:
\begin{equation}
\Theta^*=\arg\min_{\Theta} MR^{\mathrm{pac}}(D;\omega^*(\Theta),A),\quad A\in\Acal,
\end{equation}
where
\begin{align}
&MR^{\mathrm{pac}}(D;\omega^*(\Theta),A)
\nonumber\\[-0.25ex]
&\quad=\frac{1}{|D|}\sum_{i=1}^{|D|}\rho\Big(
 f(z_i;\omega^*(\Theta)),
\nonumber\\[-0.25ex]
&\hspace{8em}f(A(z_i);\omega^*(\Theta))\Big),
\end{align}
here $D\subset \Dtrain$ is a subset sampled from the noisy training data, $f$ denotes an output of the task model, $\omega^*(\Theta)$ denotes the model parameters obtained by the inner optimization, $A$ is a semantics-preserving augmentation transformation, and $\rho$ is a consistency metric in the output space.

For clarity, we distinguish two types of notation for response pairs. 
When a preference label is used in the inner training objective, we denote the ordered pair as 
$(x_i, y_i^c, y_i^r)\in D_{train}$, where $y_i^c$ and $y_i^r$ are the chosen and rejected responses according to the observed preference label in the training data. 
In contrast, when constructing the prompt-augmentation consistency objective, we do not rely on the preference label of the pair. 
Therefore, we denote the two candidate responses as $y_i^a$ and $y_i^b$, and write the corresponding example as 
$(x_i,y_i^a,y_i^b)\in D$. 
This notation only indicates two candidate responses associated with the same prompt, and does not assume that $y_i^a$ is preferred over $y_i^b$.

We refer to the learnable weighting function as VNet, denoted as
\[
V(\Delta(z;\omega),u(z;\omega);\Theta)
=
g(z;\Theta).
\]

Let $\ell_i^{\mathrm{train}}(\omega)=\mathcal{L}_{\mathrm{DPO}}(z_i;\omega), z_i \in D_{\mathrm{train}}$. Then the prompt augmentation consistency problem can be written as

\begin{align}
\Theta^*=\arg\min_{\Theta} MR^{\mathrm{pac}}(D;\omega^*(\Theta),A),\quad A\in\Acal\\
\mathrm{s.t.}\omega^*(\Theta)=\arg\min_{\omega}\mathcal{L}^{\mathrm{train}}(\Dtrain;\omega,V_{\Theta})
\end{align}
where
\begin{align}
&\mathcal{L}^{\mathrm{train}}(\Dtrain;\omega,V_{\Theta})
\nonumber\\[-0.25ex]
&\quad=\frac{1}{|\Dtrain|}\sum_{i=1}^{|\Dtrain|}
V\bigl(\Delta(z_i;\omega),u(z_i;\omega);\Theta\bigr)
\nonumber\\[-0.25ex]
&\hspace{11em}\cdot\ell_i^{\mathrm{train}}(\omega).
\end{align}

For a preference triplet $z_i$, we use back-translation as the text augmentation transformation \cite{sennrich2016improving,edunov2018understanding,xie2020unsupervised}, and apply it only to the prompt:
\begin{equation}
A(z_i)=(A(x_i),y_i^a,y_i^b).
\end{equation}

We use the English $\rightarrow$ Chinese $\rightarrow$ English back-translation path.
Based on the Bradley-Terry Model\cite{bradley1952rank}, Candidate responses are not back-translated because generative rewriting of candidate responses may introduce uncontrollable and substantial distribution shifts and may damage the latent preference relation .

\subsection{Pseudo-Labels and Confidence Filtering}
For a sample pair $z_i$, we can construct a binary distribution based on the implicit reward difference $u(z_i)$
\begin{equation}
    \mathbf{p}(z_i;\omega)=[\sigma(u(z_i)),1-\sigma(u(z_i))].
\end{equation}

This distribution can be understood as the model's prediction of the preference relationship between the candidate answers $y_i^a$ and $y_i^b$ in the current sample $z_i$.

Based on the above binary distribution, we construct pseudo-labels for the outer-level data as
\begin{equation}
\widehat{y}_i(\omega)=
\begin{cases}
0,& \sigma(u(z_i;\omega))\geq \tau,\\
1,& \sigma(u(z_i;\omega))\leq 1-\tau,
\end{cases}
\end{equation}
where $0.5\leq\tau<1$ is the confidence threshold. Samples that do not satisfy either condition are excluded from the outer-level meta-loss. The outer-level objective can then be written as
\begin{align}
\Theta^*=\arg\min_{\Theta}
-\frac{1}{|\widetilde{D}|}\sum_{i\in\widetilde{D}}
\Big[&\mathbf{1}\bigl(\widehat{y}_i(\widetilde{\omega}^*(\Theta))=0\bigr)
\nonumber\\[-0.25ex]
&\cdot\log\sigma\bigl(u(A(z_i);\omega^*(\Theta))\bigr)
\nonumber\\
+&\mathbf{1}\bigl(\widehat{y}_i(\widetilde{\omega}^*(\Theta))=1\bigr)
\nonumber\\[-0.25ex]
&\cdot\log\sigma\bigl(-u(A(z_i);\omega^*(\Theta))\bigr)\Big],
\end{align}
where $\widetilde{\omega}^*(\Theta)$ is a fixed copy of the current parameters $\omega^*(\Theta)$, indicating that gradients do not flow through the pseudo-label branch, $\widetilde{D}$ denotes the noisy training subset that passes confidence filtering, $\mathbf{1}(\cdot)$ denotes the indicator function. 

This objective does not supervise the meta-network using preference labels. Instead, it injects the task-agnostic meta-knowledge that the same sample should produce consistent preference judgments under semantics-preserving augmentation into the outer-level optimization.

An empirical validation of the prompt augmentation consistency signal is provided in Appendix~\ref{app:pac-analysis}.

\subsection{Central-Difference Meta Gradient Approximation}

Standard meta-learning bilevel optimization requires differentiating through the inner update steps and therefore involves higher-order meta-gradient computation. In the LLM setting, this leads to significant memory and computational burdens. We therefore use central differences to approximate the key sample-level directional derivatives in the VNet meta-gradient. For notational simplicity, we denote $\mathcal{L}^{\mathrm{PACMR}}(\widehat{\omega}^{(t)}(\Theta))$ as
\begin{align}
    \mathcal{L}^{\mathrm{PACMR}}(\omega^*(\Theta))=-\frac{1}{M}\sum_{i=1}^{M}
\Big[&\mathbf{1}\bigl(\widehat{y}_i(\widetilde{\omega}^*(\Theta))=0\bigr)
\nonumber\\[-0.25ex]
&\cdot\log\sigma\bigl(u(A(z_i);\omega^*(\Theta))\bigr)
\nonumber\\
+&\mathbf{1}\bigl(\widehat{y}_i(\widetilde{\omega}^*(\Theta))=1\bigr)
\nonumber\\[-0.25ex]
&\cdot\log\sigma\bigl(-u(A(z_i);\omega^*(\Theta))\bigr)\Big],
\end{align}
and
\begin{align}
    MR_i^{\mathrm{pac}}(\omega^*(\Theta))=
&-\mathbf{1}\bigl(\widehat{y}_i(\widetilde{\omega}^*(\Theta))=0\bigr)
\nonumber\\[-0.25ex]
&\cdot\log\sigma\bigl(u(A(z_i);\omega^*(\Theta))\bigr)
\nonumber\\
-&\mathbf{1}\bigl(\widehat{y}_i(\widetilde{\omega}^*(\Theta))=1\bigr)
\nonumber\\[-0.25ex]
&\cdot\log\sigma\bigl(-u(A(z_i);\omega^*(\Theta))\bigr),
\end{align}
omitting the subset set $D$ and the prompt augmentation transformation $A$, and use $M$ to denote the meta-batch size.
Let
\begin{equation}
    \mathcal{L}^{\mathrm{PACMR}}(\widehat{\omega}^{(t)}(\Theta))
    =\frac{1}{M}\sum_{j=1}^{M}MR_j^{\mathrm{pac}}(\widehat{\omega}^{(t)}(\Theta)),
\end{equation}
where
\begin{align}
\widehat{\omega}^{(t)}(\Theta)
={}&\omega^{(t)}-\frac{\alpha}{N}\sum_{i=1}^{N}
V\bigl(\Delta(z_i;\omega^{(t)}),u(z_i;\omega^{(t)});\Theta\bigr)
\nonumber\\[-0.25ex]
&\hspace{9em}\cdot\nabla_{\omega}\ell_i^{\mathrm{train}}(\omega^{(t)}).
\end{align}
and define the meta-gradient direction
\begin{align}
d_{\mathrm{meta}}
&=\nabla_{\widehat{\omega}^{(t)}}
\mathcal{L}^{\mathrm{PACMR}}(\widehat{\omega}^{(t)}(\Theta))
\nonumber\\
&=\frac{1}{M}\sum_{j=1}^{M}
\nabla_{\widehat{\omega}^{(t)}}
MR_j^{\mathrm{pac}}(\widehat{\omega}^{(t)}(\Theta)).
\end{align}
Using the chain rule for a one-step virtual inner update, the VNet meta-gradient can be written as
\begin{align}
&\nabla_{\Theta}\mathcal{L}^{\mathrm{PACMR}}(\widehat{\omega}^{(t)}(\Theta))
\nonumber\\
&\quad=-\frac{\alpha}{N}\sum_{i=1}^{N}c_i\,
\nabla_{\Theta}V\bigl(\Delta(z_i;\omega^{(t)}),
 u(z_i;\omega^{(t)});\Theta\bigr).
\end{align}
where the key coefficient is
\begin{equation}
    c_i=d_{\mathrm{meta}}^T\nabla_{\omega}\ell_i^{\mathrm{train}}(\omega^{(t)}).
\end{equation}
PACMR-DPO does not explicitly construct per-sample LoRA gradients. Instead, it uses central differences to approximate this directional derivative:
\begin{equation}
    c_i\approx \widetilde{c}_i=
    \frac{\ell_i^{\mathrm{train}}(\omega^{(t)}+\epsilon d_{\mathrm{meta}})-
    \ell_i^{\mathrm{train}}(\omega^{(t)}-\epsilon d_{\mathrm{meta}})}{2\epsilon},
\end{equation}
Thus, the VNet meta-gradient can be approximated by
\begin{align}
&\nabla_{\Theta}\mathcal{L}^{\mathrm{PACMR}}(\widehat{\omega}^{(t)}(\Theta))
\nonumber\\
&\quad\approx-\frac{\alpha}{N}\sum_{i=1}^{N}\widetilde{c}_i\,
\nabla_{\Theta}V\bigl(\Delta(z_i;\omega^{(t)}),
 u(z_i;\omega^{(t)});\Theta\bigr),
\end{align}
where $\alpha$ is the inner learning rate, $N$ is the training batch size, and $\epsilon$ is the central-difference perturbation scale. 

Central-difference approximation requires positive and negative perturbations to the model parameters. Applying perturbations and updates in the full parameter space would still be expensive for LLMs. Therefore, we further restrict training to the LoRA parameter space \cite{hu2022lora}. 

Under LoRA fine-tuning, central-difference perturbations, gradient storage, and parameter writing all occur in the low-rank delta parameter space, thereby reducing computational cost. A complete complexity analysis is given in Appendix~\ref{app:complexity}. The overall PACMR-DPO training procedure is summarized in \cref{alg:pacmr-dpo}.

\begin{algorithm}[t]
\caption{PACMR-DPO}
\label{alg:pacmr-dpo}
\begin{algorithmic}[1]
\Require Reference model $\piref$; initialized policy model $\pi_{\omega}\leftarrow\piref$; VNet $V(\cdot;\Theta)$; augmentation $A$; training data $\Dtrain$; outer data $D$ sampled from $\Dtrain$; threshold $\tau$; finite-difference scale $\epsilon$.
\Ensure Policy model $\pi_{\omega}$.
\For{$t=1,2,\ldots,E$}
    \State Sample training batch $B_t=\{z_i\}_{i=1}^{B_{\mathrm{train}}}$, where $z_i=(x_i,y_i^c,y_i^r)$, from $\Dtrain$.
    \State Compute implicit reward difference $u(z_i;\omega^{(t)})$ and reward sum $\Delta(z_i;\omega^{(t)})$.
    \State Compute weights $V(\Delta(z_i;\omega^{(t)}),u(z_i;\omega^{(t)});\Theta^{(t)})$.
    \State Compute weighted DPO loss and perform one virtual update $\widehat{\omega}^{(t)}(\Theta^{(t)})$.
    \State Sample outer batch $B_m=\{z_j,A(z_j)\}_{j=1}^{B_{\mathrm{meta}}}$, where $z_j=(x_j,y_j^a,y_j^b)$, $A(z_j)=(A(x_j),y_j^a,y_j^b)$, from $D$.
    \State Compute original implicit reward differences $u(z_j;\widehat{\omega}^{(t)}(\Theta^{(t)}))$ and augmented implicit reward differences $u(A(z_j);\widehat{\omega}^{(t)}(\Theta^{(t)}))$.
    \State Construct pseudo labels using threshold $\tau$ and obtain filtered meta batch $\widetilde{B}_{\mathrm{meta}}$.
    \State Compute meta loss $\mathcal{L}^{\mathrm{PACMR}}$ and meta gradient $d_{\mathrm{meta}}=\nabla_{\widehat{\omega}^{(t)}}\mathcal{L}^{\mathrm{PACMR}}(\widehat{\omega}^{(t)}(\Theta^{(t)}))$.
    \State Evaluate perturbed training losses at $\omega^{(t)}+\epsilon d_{\mathrm{meta}}$ and $\omega^{(t)}-\epsilon d_{\mathrm{meta}}$.
    \State Compute central-difference coefficients $\widetilde{c}_i$ and update VNet parameters $\Theta^{(t+1)}$.
    \State Recompute weights with $V(\cdot;\Theta^{(t+1)})$ and update policy model parameters $\omega^{(t+1)}$ by weighted DPO.
\EndFor
\State \Return $\pi_{\omega}$
\end{algorithmic}
\end{algorithm}

\section{Experiments}

\subsection{Experimental Setup}

We evaluate PACMR-DPO on summarization task (TL;DR) \cite{volske2017tl} and dialogue preference task (Anthropic HH) \cite{bai2022training}. In the summarization task, the model generates a summary for a Reddit post. We first perform supervised fine-tuning on a filtered TL;DR summarization dataset to obtain the initial reference model, and then use the human preference for preference optimization \cite{stiennon2020learning}. In the dialogue preference task, we use the Anthropic Helpful and Harmless (HH) dialogue preference dataset and conduct instruction fine-tuning on the preferred responses to obtain the reference model.

We inject random-flip noise into the datasets of the two tasks to construct noisy preference-learning scenarios. Data processing and noise-injection details are given in Appendix~\ref{app:data-noise-and-clean-meta-data}. Specifically, we consider three noise rates, $20\%$, $30\%$, and $40\%$, to systematically evaluate the performance of different methods as the noise rate increases. In all experiments, we use Llama-2-7B as the base model. For computational efficiency, LoRA fine-tuning is used in both SFT and preference alignment, with LoRA rank set to 16 and LoRA alpha set to 32. All models are trained for one epoch on the corresponding dataset, and the experiments are based on OpenRLHF \cite{hu2024openrlhf}.

\subsection{Baselines and Evaluation Metrics}

We select DPO, cDPO, IPO \cite{azar2024general}, rDPO, and Dr.DPO as baselines. Training hyperparameters and baseline implementation details are given in Appendix~\ref{app:hyper-baselines}, which also presents the clean-metadata MWN-DPO formulation and pseudocode used in the dedicated comparison. 

Standard DPO is used as the evaluation anchor, and outputs from the proposed method and other baselines are compared pairwise against DPO outputs. We randomly sample 800 prompts from the test set, generate responses using the proposed method, each baseline, and standard DPO, pair each generated response with the corresponding DPO response, and then use GPT-5.1 for pairwise comparison \cite{rafailov2023direct,zheng2023judging}. 

The evaluation protocol and judge prompt are provided in Appendix~\ref{app:evaluation}; win/tie/loss ratios are recorded. The win-score is defined as \cite{bukharin2024robust}
\begin{equation}
\mathrm{Win\text{-}score}=1+\frac{\#\mathrm{win}-\#\mathrm{lose}}{\mathrm{Total\ comparisons}}.
\end{equation}

\subsection{Results under Noisy Preference Labels}

\begin{table*}[!t]
\centering
\caption{Pairwise evaluation against standard DPO on the TL;DR summarization task under random preference-label flips.}
\label{tab:main-tldr}
\small
\begin{tabular}{lcccccc}
\toprule
\multirow{2}{*}{Method} & \multicolumn{3}{c}{Win rate} & \multicolumn{3}{c}{Win-score} \\
\cmidrule(lr){2-4}\cmidrule(lr){5-7}
& 20\% Flipped & 30\% Flipped & 40\% Flipped & 20\% Flipped & 30\% Flipped & 40\% Flipped\\
\midrule
cDPO & 14.13\% & 11.13\% & 11.00\% & 0.7813 & 0.7563 & 0.8525 \\
IPO & 10.00\% & 12.50\% & 11.25\% & 0.5850 & 0.7150 & 0.8688 \\
rDPO & 45.38\% & 53.63\% & 52.38\% & 1.3287 & 1.4138 & 1.4113 \\
Dr.DPO & 55.38\% & 60.88\% & 49.63\% & 1.4438 & 1.5175 & 1.3650 \\
PACMR-DPO & \textbf{61.63\%} & \textbf{61.13\%} & \textbf{61.13\%} & \textbf{1.4763} & \textbf{1.5263} & \textbf{1.4950} \\
\bottomrule
\end{tabular}
\end{table*}

\begin{table*}[!t]
\centering
\caption{Pairwise evaluation against standard DPO on the HH single-turn dialogue task under random preference-label flips.}
\label{tab:main-hh}
\small
\begin{tabular}{lcccccc}
\toprule
\multirow{2}{*}{Method} & \multicolumn{3}{c}{Win rate} & \multicolumn{3}{c}{Win-score} \\
\cmidrule(lr){2-4}\cmidrule(lr){5-7}
& 20\% Flipped & 30\% Flipped & 40\% Flipped & 20\% Flipped & 30\% Flipped & 40\% Flipped\\
\midrule
cDPO & 18.63\% & 14.38\% & 13.13\% & 0.7925 & 0.7913 & 0.9075 \\
IPO & 18.13\% & 15.13\% & 11.50\% & 0.6638 & 0.7675 & 0.8863 \\
rDPO & 44.50\% & 49.88\% & \textbf{41.38\%} & 1.2525 & 1.3250 & 1.2213 \\
Dr.DPO & 60.50\% & 60.63\% & 39.88\% & 1.4013 & 1.4400 & 1.2363 \\
PACMR-DPO & \textbf{65.13\%} & \textbf{74.00\%} & 40.63\% & \textbf{1.4488} & \textbf{1.6075} & \textbf{1.2400} \\
\bottomrule
\end{tabular}
\end{table*}

\begin{table*}[!t]
\centering
\caption{Mean learned VNet weights for unflipped and flipped training pairs.}
\label{tab:final-vnet-weights}
\small
\begin{tabular}{llccc}
\toprule
Task & Noise rate & Mean weight (unflipped) & Mean weight (flipped) & Weight Gap \\
\midrule
TL;DR & 20\% & 0.7152 & 0.3314 & 0.3838 \\
TL;DR & 30\% & 0.6871 & 0.3909 & 0.2962 \\
TL;DR & 40\% & 0.5649 & 0.4131 & 0.1518 \\
HH & 20\% & 0.5354 & 0.3860 & 0.1494 \\
HH & 30\% & 0.5423 & 0.4018 & 0.1406 \\
HH & 40\% & 0.5190 & 0.4976 & 0.0213 \\
\bottomrule
\end{tabular}
\end{table*}

\begin{figure*}[!t]
\centering

\subfloat[20\%]{\includegraphics[width=0.31\textwidth]{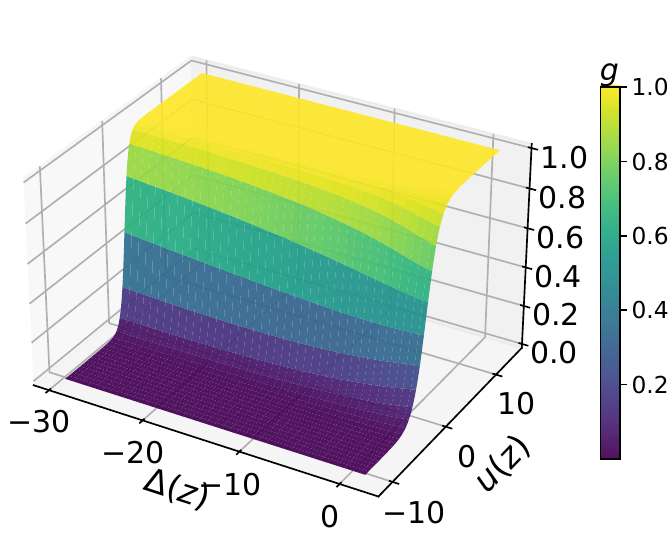}}
\hfil
\subfloat[30\%]{\includegraphics[width=0.31\textwidth]{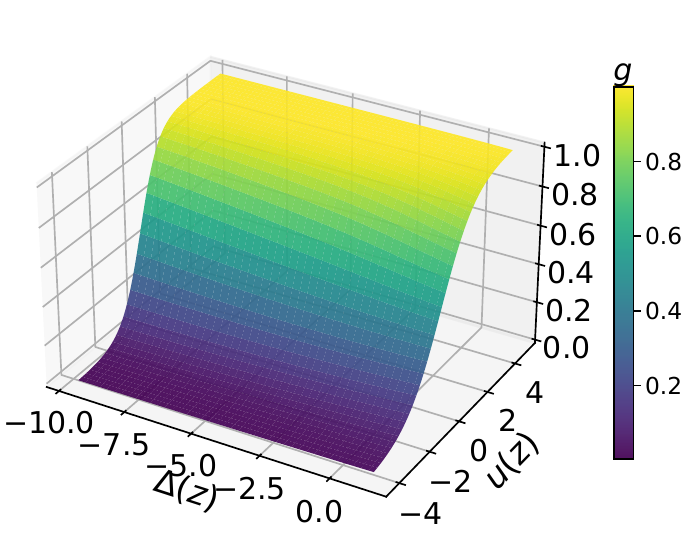}}
\hfil
\subfloat[40\%]{\includegraphics[width=0.31\textwidth]{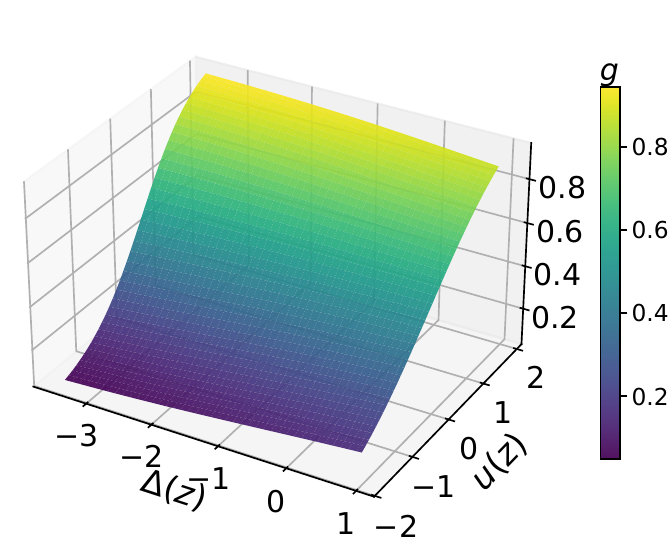}}
\caption{Learned weight functions on TL;DR under 20\%, 30\%, and 40\% random preference flips.}
\label{fig:tldr-final-weight-functions}
\end{figure*}

\begin{figure*}[!t]
\centering

\subfloat[20\%]{\includegraphics[width=0.31\textwidth]{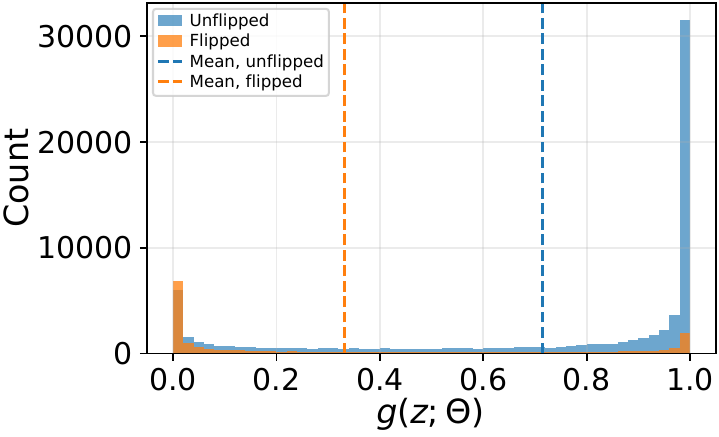}}
\hfil
\subfloat[30\%]{\includegraphics[width=0.31\textwidth]{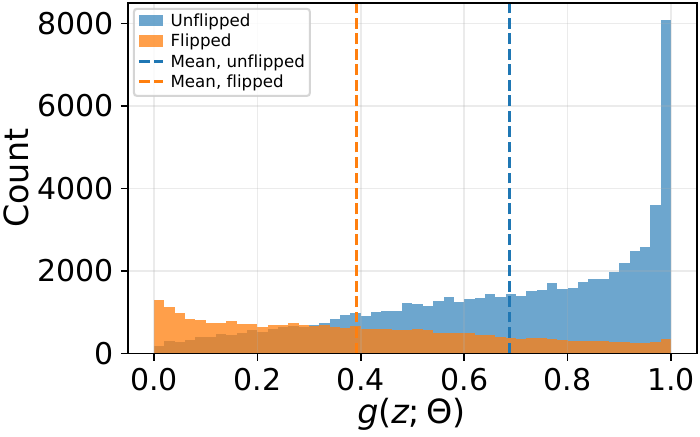}}
\hfil
\subfloat[40\%]{\includegraphics[width=0.31\textwidth]{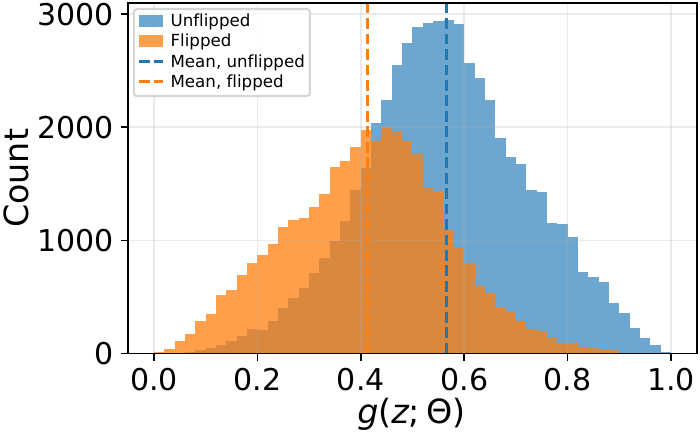}}
\caption{Learned weight distributions on TL;DR.}
\label{fig:tldr-final-weight-hists}
\end{figure*}

\begin{figure*}[!t]
\centering

\subfloat[20\%]{\includegraphics[width=0.31\textwidth]{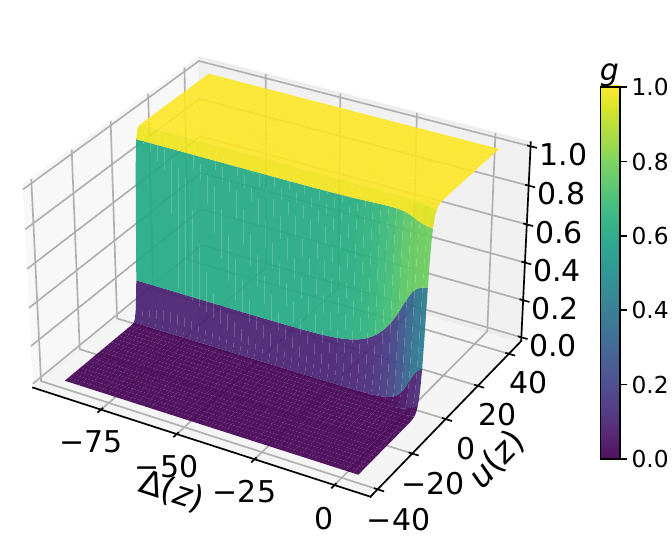}}
\hfil
\subfloat[30\%]{\includegraphics[width=0.31\textwidth]{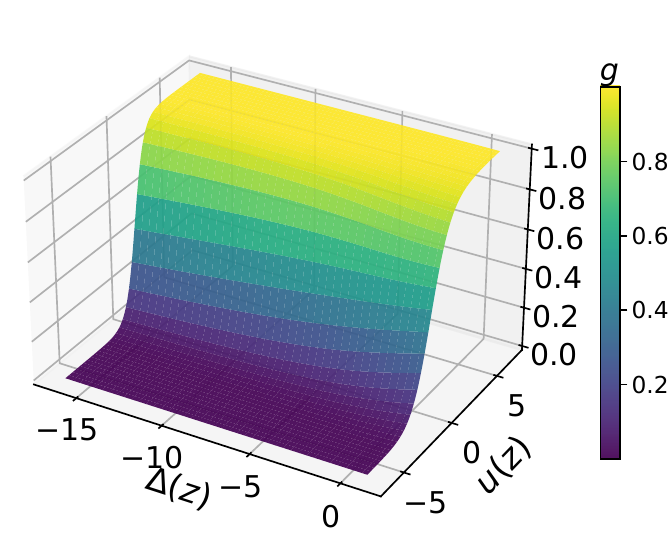}}
\hfil
\subfloat[40\%]{\includegraphics[width=0.31\textwidth]{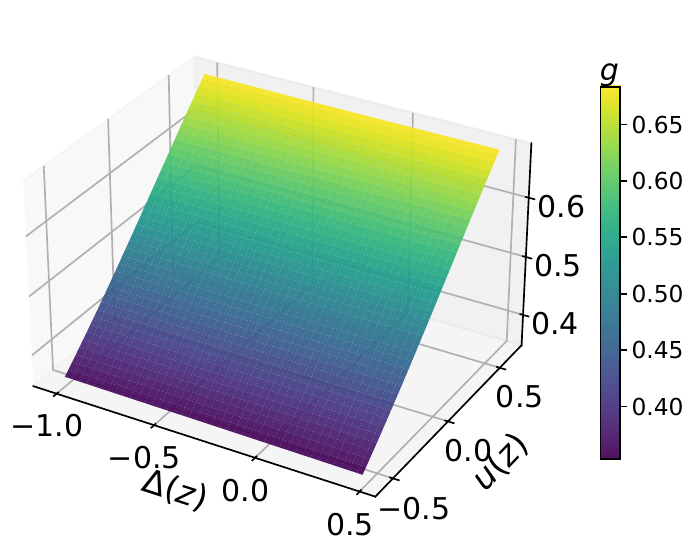}}
\caption{Learned weight functions on Anthropic HH under 20\%, 30\%, and 40\% random preference flips.}
\label{fig:hh-final-weight-functions}
\end{figure*}

\begin{figure*}[!t]
\centering

\subfloat[20\%]{\includegraphics[width=0.31\textwidth]{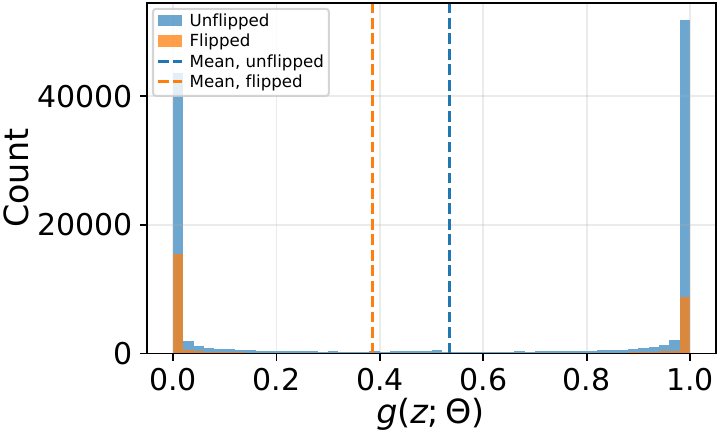}}
\hfil
\subfloat[30\%]{\includegraphics[width=0.31\textwidth]{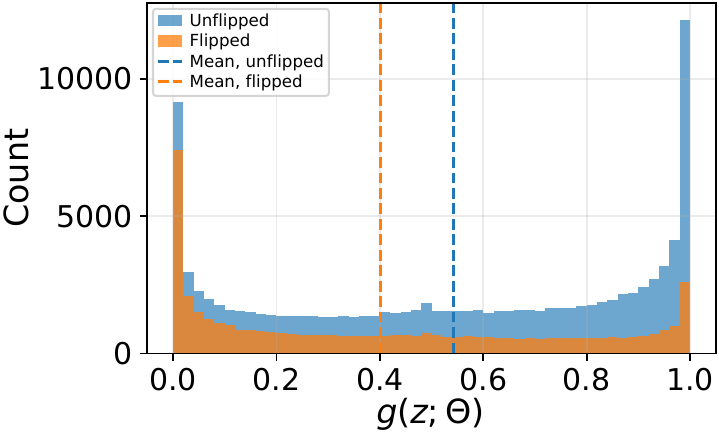}}
\hfil
\subfloat[40\%]{\includegraphics[width=0.31\textwidth]{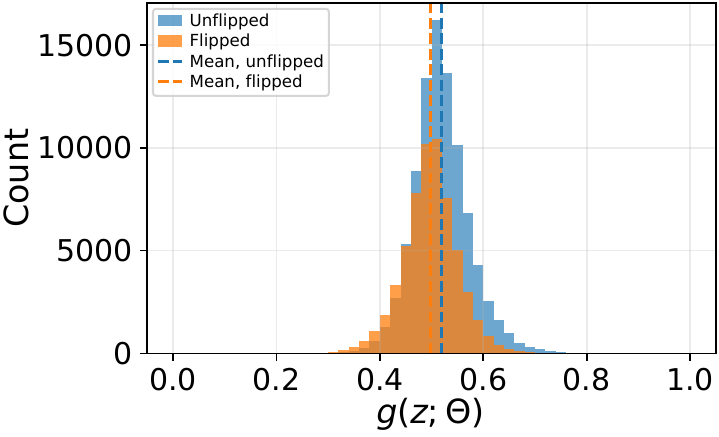}}
\caption{Learned weight distributions on Anthropic HH.}
\label{fig:hh-final-weight-hists}
\end{figure*}

As shown in \cref{tab:main-tldr}, the results show that PACMR-DPO achieves the highest win rate and win-score under all three noise rates on TL;DR. On Anthropic HH, as shown in \cref{tab:main-hh}, PACMR-DPO also performs well. Under 20\% and 30\% noise, it obtains higher observed Win-score and Win rate. Under 40\% noise, PACMR-DPO reaches a win rate of 40.63\% and a win-score of 1.2400. Although its win rate is slightly lower than rDPO's 41.38\%, its win-score is higher than that of all baselines, indicating that PACMR-DPO still has a overall comparison quality.

From the weight-distribution plots on TL;DR, as shown in \cref{fig:tldr-final-weight-hists}, the distributions of clean and noisy samples differ clearly under 20\% and 30\% noise, suggesting that the meta-network can assign different weights to samples of different reliability, thereby emphasizing relatively trustworthy preference signals and suppressing noisy labels during training. As the noise rate increases to 40\%, the proportion of noisy samples grows, and clean and noisy samples increasingly overlap in the model's implicit-reward feature space, making the task more difficult. The weight distributions still exhibit some separability, but the overlap increases, indicating that weight learning becomes harder at higher noise rates and sample discrimination weakens. Nevertheless, PACMR-DPO still outperforms all baselines under 40\% noise.

The weight-distribution plots on HH show that PACMR-DPO can also learn weight differences between clean and noisy samples under most noise rates, as shown in \cref{fig:hh-final-weight-hists}. This indicates that the meta-learning signal constructed from the implicit reward difference, implicit reward sum, and prompt augmentation consistency is also effective on HH. Similar to TL;DR, as the noise rate increases, the distributions of clean and noisy samples increasingly overlap, indicating that the task becomes harder with higher noise. In particular, under 40\% noise on HH, the weight distributions of clean and noisy samples are much less separable and become highly similar. This suggests that sample-reliability modeling for the HH dialogue preference task becomes substantially more difficult under extremely high noise. Although the meta-network can still learn some weight differences, a single-epoch training setup may not sufficient to form a strongly separable weight distribution. Even so, PACMR-DPO still performs well.

\Cref{tab:final-vnet-weights} presents the details of the weight distribution plots on both TL;DR and HH tasks, where Weight Gap denotes the difference between the average weight of unflipped pairs and that of flipped pairs. It can be observed from the table that the learned weights are able to distinguish between flipped and unflipped pairs, yet the discriminative ability diminishes as the noise rate increases.

Despite the difference between summarization and dialogue, the learned
VNet weighting functions have similar shapes at matched flip rates
(\cref{fig:tldr-final-weight-functions,fig:hh-final-weight-functions}),
suggesting that they may capture task-agnostic regularities in
preference-pair reliability.
\subsection{Ablation: Learned Weighting vs. Fixed Prior Weighting}

This subsection compares PACMR-DPO with fixed-prior weighting. The fixed-weighting variant does not learn a VNet, and instead directly uses
\begin{equation}
g(z)=\sigma(u(z))
\end{equation}
as the sample weight. This ablation answers whether a theoretically inspired fixed monotonic weight is sufficient, or whether meta-learning is necessary to learn the weighting function.

\begin{table*}[!t]
\centering
\caption{Ablation study comparing fixed sigmoid weighting with learned PACMR-DPO weighting.}
\label{tab:ablation-fixed-weight}
\small
\begin{tabular}{ll lcc}
\toprule
Task & Noise rate & Method & Win rate & Win-score \\
\midrule
TL;DR & 20\% & Fixed $\sigma(u)$ & 55.00\% & 1.4425 \\
TL;DR & 20\% & PACMR-DPO & 61.63\% & 1.4763 \\
TL;DR & 30\% & Fixed $\sigma(u)$ & 61.50\% & 1.5300 \\
TL;DR & 30\% & PACMR-DPO & 61.13\% & 1.5263 \\
TL;DR & 40\% & Fixed $\sigma(u)$ & 49.38\% & 1.3650 \\
TL;DR & 40\% & PACMR-DPO & 61.13\% & 1.4950 \\
HH & 20\% & Fixed $\sigma(u)$ & 62.63\% & 1.4500 \\
HH & 20\% & PACMR-DPO & 65.13\% & 1.4488 \\
HH & 30\% & Fixed $\sigma(u)$ & 63.88\% & 1.4938 \\
HH & 30\% & PACMR-DPO & 74.00\% & 1.6075 \\
HH & 40\% & Fixed $\sigma(u)$ & 37.50\% & 1.2000 \\
HH & 40\% & PACMR-DPO & 40.63\% & 1.2400 \\
\bottomrule
\end{tabular}
\end{table*}

As shown in \cref{tab:ablation-fixed-weight}, the fixed sigmoid weight brings some benefit in certain medium-noise settings, but is overall less stable than learned PACMR-DPO. In particular, PACMR-DPO has clearly higher win-scores on TL;DR under 20\% and 40\% noise and on HH under 30\% and 40\% noise. This indicates that using only $\sigma(u)$ as a static prior is insufficient to fully characterize sample reliability, and that learning sample weights with the outer prompt-augmentation-consistency objective is necessary.
\subsection{Sensitivity to Confidence Threshold \texorpdfstring{$\tau$}{tau}}

The confidence threshold $\tau$ controls the filtering strength for outer-level pseudo-labels. A smaller $\tau$ increases outer-sample coverage but may introduce low-confidence pseudo-labels; a larger $\tau$ can filter unreliable samples but may reduce the effective outer-level meta-signal. We consider $\tau\in\{0.55,0.60,0.70\}$.

\begin{table}[!htbp]
\centering
\caption{Sensitivity analysis of pseudo-label confidence threshold $\tau$.}
\label{tab:tau-sensitivity}
\footnotesize
\setlength{\tabcolsep}{7pt}
\renewcommand{\arraystretch}{0.93}
\begin{tabular}{@{}llccc@{}}
\toprule
Task & Noise rate & $\tau$ & Win rate & Win-score \\
\midrule
TL;DR & 20\% & 0.55 & 54.63\% & 1.4450 \\
TL;DR & 20\% & \textbf{0.60} & \textbf{61.63\%} & \textbf{1.4763} \\
TL;DR & 20\% & 0.70 & 50.88\% & 1.3425 \\
TL;DR & 30\% & 0.55 & 61.00\% & 1.5038 \\
TL;DR & 30\% & \textbf{0.60} & \textbf{61.13\%} & \textbf{1.5263} \\
TL;DR & 30\% & 0.70 & 60.88\% & 1.5150 \\
TL;DR & 40\% & 0.55 & 60.25\% & 1.4900 \\
TL;DR & 40\% & \textbf{0.60} & \textbf{61.13\%} & \textbf{1.4950} \\
TL;DR & 40\% & 0.70 & 52.63\% & 1.4088 \\
HH & 20\% & 0.55 & 60.00\% & 1.4163 \\
HH & 20\% & 0.60 & 65.13\% & 1.4488 \\
HH & 20\% & \textbf{0.70} & \textbf{68.00\%} & \textbf{1.4875} \\
HH & 30\% & 0.55 & 73.00\% & 1.5787 \\
HH & 30\% & \textbf{0.60} & \textbf{74.00\%} & \textbf{1.6075} \\
HH & 30\% & 0.70 & 65.38\% & 1.5162 \\
HH & 40\% & 0.55 & 36.88\% & 1.1738 \\
HH & 40\% & \textbf{0.60} & \textbf{40.63\%} & \textbf{1.2400} \\
HH & 40\% & 0.70 & 39.25\% & 1.2375 \\
\bottomrule
\end{tabular}
\end{table}

As shown in \cref{tab:tau-sensitivity}, on TL;DR, $\tau=0.60$ achieves the best overall performance, and the model performance first improves and then declines as $\tau$ increases. On HH, $\tau=0.60$ yields the best performance under 30\% and 40\% noise, whereas $\tau=0.70$ performs best under 20\% noise. Similar to the results on TL;DR, the performance on HH also exhibits a first-increasing-then-decreasing trend as $\tau$ increases. Overall, a threshold that is too low may introduce unreliable pseudo-labels, while a threshold that is too high may reduce the number of effective outer-level samples. Therefore, we use the compromise value $\tau=0.60$.
\subsection{Comparison with MWN-DPO Using Clean Metadata}

To evaluate whether the prompt-augmentation-consistency outer objective can replace clean meta-preference labels, this subsection compares PACMR-DPO with MWN-DPO using clean metadata. For fairness, MWN-DPO and PACMR-DPO use the same VNet input and weighting-function structure. The only difference is the outer-level objective: MWN-DPO uses the DPO loss on clean meta-preference data to update VNet, while PACMR-DPO uses the prompt-augmentation-consistency objective. Details of the clean metadata construction are given in Appendix~\ref{app:data-noise-and-clean-meta-data}, and the MWN-DPO formulation and training procedure are provided in Appendix~\ref{app:hyper-baselines} and Algorithm~\ref{alg:mwn-dpo} therein.

\begin{table}[!htbp]
\centering
\caption{Comparison of PACMR-DPO with MWN-DPO using clean metadata.}
\label{tab:clean-meta-mwn}
\small
\setlength{\tabcolsep}{6pt}
\renewcommand{\arraystretch}{0.95}
\begin{tabular}{@{}lllcc@{}}
\toprule
Task & Noise rate & Method & Win rate & Win-score \\
\midrule
TL;DR & 20\% & MWN-DPO (clean metadata) & 52.75\% & 1.4138 \\
TL;DR & 20\% & PACMR-DPO & 61.63\% & 1.4763 \\
TL;DR & 30\% & MWN-DPO (clean metadata) & 61.88\% & 1.5238 \\
TL;DR & 30\% & PACMR-DPO & 61.13\% & 1.5263 \\
TL;DR & 40\% & MWN-DPO (clean metadata) & 62.63\% & 1.5288 \\
TL;DR & 40\% & PACMR-DPO & 61.13\% & 1.4950 \\
HH & 20\% & MWN-DPO (clean metadata) & 58.88\% & 1.4100 \\
HH & 20\% & PACMR-DPO & 65.13\% & 1.4488 \\
HH & 30\% & MWN-DPO (clean metadata) & 73.38\% & 1.5900 \\
HH & 30\% & PACMR-DPO & 74.00\% & 1.6075 \\
HH & 40\% & MWN-DPO (clean metadata) & 39.00\% & 1.2288 \\
HH & 40\% & PACMR-DPO & 40.63\% & 1.2400 \\
\bottomrule
\end{tabular}
\end{table}

\Cref{tab:clean-meta-mwn} shows that PACMR-DPO can match or exceed clean-meta MWN-DPO in multiple settings without using clean meta-preference labels. This indicates that prompt augmentation consistency can serve as an effective outer-level proxy signal when clean metadata is unavailable, although its advantage does not hold for every task and noise rate.

These results may also partly reflect the practical difficulty of constructing a representative clean meta set. Clean labels alone do not guarantee that a small meta set adequately represents the full clean preference distribution. This observation does not imply that clean metadata is ineffective; rather, it highlights
that its utility depends on both label correctness and distributional representativeness.
\FloatBarrier
\section{Conclusion}

This paper studies DPO learning under noisy preference labels and analyzes how noisy preferences cause a shift in the DPO optimum from the perspectives of conditional risk and posterior shift. We prove that, under some idealized conditions, the optimal weight ratio in bilevel reweighting can make weighted DPO under the noisy training distribution recover the pointwise optimum under the clean preference distribution. Inspired by this theory, we propose PACMR-DPO, which learns sample weights through VNet and uses prompt augmentation consistency as the outer meta-objective instead of clean meta-preference data. We further use central-difference approximation and LoRA-space perturbations to reduce the computational cost of bilevel optimization in LLM settings. Experiments show that PACMR-DPO achieves performance improvements over multiple robust DPO baselines under medium and high preference-noise settings.
\clearpage
\bibliographystyle{IEEEtran}
\bibliography{references}

\clearpage
\appendix
\section*{Appendix}
\crefname{section}{Appendix}{Appendices}
\Crefname{section}{Appendix}{Appendices}

\section{Exact Meta-Gradient and Complexity Analysis of the Central-Difference Approximation}
\label{app:complexity}

\subsection{Notation and Assumptions}

This section analyzes the computational complexity of VNet parameter updates in PACMR-DPO. Let $\omega\in\mathbb{R}^{P_L}$ denote the LoRA parameters participating in policy training, and let $\Theta\in\mathbb{R}^{Q}$ denote the VNet parameters. Let $N$ be the number of training samples used for the VNet meta-update, and let $M$ be the number of meta-samples used to compute the prompt augmentation consistency meta loss. For the $i$-th training sample, denote the single-sample DPO training loss as
\begin{equation}
    \ell_i^{\mathrm{train}}(\omega).
\end{equation}
The weight output by VNet for this training sample is denoted as
\begin{equation}
    v_i(\Theta)=V(\Delta(z_i;\omega),u(z_i;\omega);\Theta),
\end{equation}
In the VNet update and virtual inner update, these input features are processed with stop-gradient.

In the LLM with LoRA setting considered in this paper, VNet is a small MLP, while the trainable parameters on the policy side lie in the LoRA parameter space of a large language model. Therefore,
\begin{equation}
    Q\ll P_L.
\end{equation}
At the same time, the training batch and meta-batch sizes are much smaller than the LoRA parameter dimension, namely
\begin{equation}
    N\ll P_L,
    \qquad
    M\ll P_L.
\end{equation}
 Let $A_{\mathrm{tr}}^{(1)}(N)$ denote the space cost of the first-order training-side backward graph, $A_{\mathrm{tr}}^{(2)}(N)$ the space cost of the higher-order training-side graph retained for bilevel differentiation, and $A_{\mathrm{meta}}^{(1)}(M)$ the space cost of the first-order meta-loss backward graph. Let $I_{\mathrm{tr}}(N)$ denote the space cost of forward-only training loss evaluation. Then, in the same configuration, we have
\begin{equation}
    I_{\mathrm{tr}}(N) \le A_{\mathrm{tr}}^{(1)}(N)\le A_{\mathrm{tr}}^{(2)}(N).
\end{equation}
Let $F_{\mathrm{tr}}(N)$ denote the time for forward-only loss evaluation on $N$ training samples.

\subsection{Exact Meta-Gradient and Computational Bottleneck}

The one-step virtual inner update used in the VNet update is
\begin{equation}
    \widehat{\omega}(\Theta)
    =\omega-
    \frac{\alpha}{N}
    \sum_{i=1}^{N}
    v_i(\Theta)
    \nabla_{\omega}\ell_i^{\mathrm{train}}(\omega).
    \label{eq:app-virtual-update}
\end{equation}
The outer-level meta-objective is denoted as
\begin{equation}
    \mathcal{L}^{\mathrm{PACMR}}(\widehat{\omega})
    =\frac{1}{M}\sum_{j=1}^{M}MR_j^{\mathrm{pac}}(\widehat{\omega}).
\end{equation}
Define the meta-gradient direction
\begin{equation}
    d_{\mathrm{meta}}
    =\nabla_{\widehat{\omega}}\mathcal{L}^{\mathrm{PACMR}}(\widehat{\omega})
    =\frac{1}{M}\sum_{j=1}^{M}
    \nabla_{\widehat{\omega}}MR_j^{\mathrm{pac}}(\widehat{\omega})
    \in\mathbb{R}^{P_L}.
\end{equation}
The reverse path from the VNet parameters $\Theta$ to the outer meta-loss is
\begin{equation}
    \Theta
    \longrightarrow
    v_i(\Theta)
    \longrightarrow
    \widehat{\omega}(\Theta)
    \longrightarrow
    \mathcal{L}^{\mathrm{PACMR}}(\widehat{\omega}(\Theta)).
    \label{eq:app-reverse-path}
\end{equation}
By the chain rule and \cref{eq:app-virtual-update}, we obtain
\begin{align}
&\nabla_{\Theta}\mathcal{L}^{\mathrm{PACMR}}(\widehat{\omega}(\Theta))
\nonumber\\[-0.25ex]
&\quad=-\frac{\alpha}{N}\sum_{i=1}^{N}
\left[d_{\mathrm{meta}}^T
\nabla_{\omega}\ell_i^{\mathrm{train}}(\omega)\right]
\nabla_{\Theta}v_i(\Theta).
\label{eq:app-exact-meta-grad}
\end{align}
Therefore, the computational bottleneck in the exact VNet meta-gradient is the coefficient corresponding to each training sample,
\begin{equation}
    c_i
    =d_{\mathrm{meta}}^T\nabla_{\omega}\ell_i^{\mathrm{train}}(\omega),
    \qquad i=1,\ldots,N.
    \label{eq:app-ci}
\end{equation}
Under $Q\ll P_L$, $\nabla_{\Theta}v_i(\Theta)$ lies in the parameter space of the small MLP and is not the dominant cost. Therefore, the following complexity analysis focuses on computing $c_i$ in \cref{eq:app-ci}.

\subsection{Exact Computation by Retaining the Computational Graph}

The first exact computation strategy is to directly retain the complete reverse path in \cref{eq:app-reverse-path}. In this case, the outer meta-loss must continue backpropagating through the virtual update $\widehat{\omega}(\Theta)$ to $v_i(\Theta)$ and $\Theta$. Since $\widehat{\omega}(\Theta)$ contains
\begin{equation}
    \nabla_{\omega}\ell_i^{\mathrm{train}}(\omega),
\end{equation}
the reverse path must retain the higher-order training-side computational graph, whose space cost is denoted by $A_{\mathrm{tr}}^{(2)}(N)$. In addition, computing $d_{\mathrm{meta}}$ requires retaining the first-order meta-loss backward graph, whose space cost is $A_{\mathrm{meta}}^{(1)}(M)$. At least one vector in the LoRA parameter space, of size $P_L$, must also be stored. Under $Q\ll P_L$, $N\ll P_L$, and $M\ll P_L$, the VNet parameters, sample weights, and sample-level scalar coefficients are lower-order terms. Therefore, the space complexity of retaining the computational graph is
\begin{equation}
    S_{\mathrm{graph}}
    =O\left(
    A_{\mathrm{tr}}^{(2)}(N)
    +A_{\mathrm{meta}}^{(1)}(M)
    +P_L
    \right).
    \label{eq:app-s-graph}
\end{equation}
This strategy computes the meta-gradient exactly, but incurs the space cost of retaining the higher-order computational graph for the inner virtual update.

\subsection{Complexity of the Central-Difference Approximation}

PACMR-DPO does not explicitly compute $\nabla_{\omega}\ell_i^{\mathrm{train}}(\omega)$. Instead, it uses central differences to approximate $c_i$ in \cref{eq:app-ci}:
\begin{align}
c_i
&=d_{\mathrm{meta}}^T\nabla_{\omega}\ell_i^{\mathrm{train}}(\omega)
\approx\widetilde{c}_i
\nonumber\\
&=\frac{1}{2\epsilon}\Big[
\ell_i^{\mathrm{train}}(\omega+\epsilon d_{\mathrm{meta}})
\nonumber\\[-0.25ex]
&\hspace{6em}-\ell_i^{\mathrm{train}}(\omega-\epsilon d_{\mathrm{meta}})
\Big].
\label{eq:app-fd-ci}
\end{align}
Central differences require storing the direction vector $d_{\mathrm{meta}}$ in the LoRA space, corresponding to space $P_L$. The training losses under positive and negative perturbations,
\begin{equation}
    \ell_i^{\mathrm{train}}(\omega+\epsilon d_{\mathrm{meta}}),
    \qquad
    \ell_i^{\mathrm{train}}(\omega-\epsilon d_{\mathrm{meta}})
\end{equation}
only require forward computation. Since $I_{\mathrm{tr}}(N)\le (A_{\mathrm{tr}}^{(1)}(N))$, forward-only evaluation does not exceed the space cost of the first-order training-side backward graph in asymptotic order. Computing $d_{\mathrm{meta}}$ requires the first-order meta-loss backward graph $A_{\mathrm{meta}}^{(1)}(M)$. Therefore, the peak space complexity of the central-difference implementation is
\begin{equation}
    S_{\mathrm{FD}}
    =O\left(
    P_L+
    \max\left\{
    A_{\mathrm{tr}}^{(1)}(N),
    A_{\mathrm{meta}}^{(1)}(M)
    \right\}
    \right).
    \label{eq:app-s-fd}
\end{equation}
Computing the central-difference coefficients requires two forward-only training-loss evaluations. Therefore, the additional time complexity for computing the central-difference directional-derivative coefficients is
\begin{equation}
    T_{\mathrm{FD\text{-}coeff}}
    =O\left(F_{\mathrm{tr}}(N)\right).
    \label{eq:app-t-fd}
\end{equation}

\subsection{Comparison with Exact Computation}

The exact implementation that retains the computational graph has space complexity
\begin{equation}
    S_{\mathrm{graph}}
    =O\left(
    A_{\mathrm{tr}}^{(2)}(N)
    +A_{\mathrm{meta}}^{(1)}(M)
    +P_L
    \right),
\end{equation}
whereas central differences have space complexity
\begin{equation}
    S_{\mathrm{FD}}
    =O\left(
    P_L+
    \max\left\{
    A_{\mathrm{tr}}^{(1)}(N),
    A_{\mathrm{meta}}^{(1)}(M)
    \right\}
    \right).
\end{equation}
Since $A_{\mathrm{tr}}^{(1)}(N)\le(A_{\mathrm{tr}}^{(2)}(N))$, central differences avoid the higher-order training graph retained for the inner virtual update in exact bilevel backpropagation.

Compared with the graph-retention implementation, the central-difference estimator introduces an explicit runtime overhead because it requires two additional evaluations of the training losses at $\omega+\epsilon d_{\mathrm{meta}}$ and $\omega-\epsilon d_{\mathrm{meta}}$. As shown in \cref{eq:app-t-fd}, this additional cost is $O(F_{\mathrm{tr}}(N))$ and consists only of two batched forward-only evaluations, without additional policy-model backward passes for computing the directional-derivative coefficients. Consequently, the additional runtime is practically acceptable in exchange for the reduction in peak memory required by bilevel differentiation.

In summary, under $Q\ll P_L$, $N\ll P_L$, and $M\ll P_L$, and when forward-only computation is cheaper than backpropagation, central differences approximate $d_{\mathrm{meta}}^T\nabla_{\omega}\ell_i^{\mathrm{train}}(\omega)$ by replacing exact differentiation through the higher-order inner-update graph with two batched forward-only loss evaluations. This changes the peak-space requirement from \cref{eq:app-s-graph} to \cref{eq:app-s-fd}, at the cost of the additional runtime in \cref{eq:app-t-fd}. The resulting memory--time trade-off is suitable for LoRA-based LLM training, where peak memory is the primary constraint.

\section{Meaning of the DPO Conditional Risk and Derivation of the Optimum}
\label{app:conditional-risk}

The conditional risk is the pointwise risk obtained by fixing the input pair $z=(x,y_a,y_b)$ and taking the expectation over the conditional distribution of the latent preference label associated with that sample. In our setting, $Y=0$ means $y_a\succ y_b$, $Y=1$ means $y_a\prec y_b$, and
\begin{equation}
    \eta(z)=\Prob(Y=0\mid Z=z).
\end{equation}
If the model gives an implicit reward difference $u(z)$ on sample $z$, then the loss for label $Y=0$ is $-\log\sigma(u(z))$, and the loss for label $Y=1$ is $-\log\sigma(-u(z))$. Therefore, the DPO conditional risk under the clean preference posterior is
\begin{equation}
    \mathcal{L}_{\mathrm{DPO}}(u,\eta)
    =-\eta\log\sigma(u)-(1-\eta)\log\sigma(-u),
\end{equation}
where the dependence on $z$ is omitted for simplicity.

Taking the derivative with respect to $u$ gives
\begin{align}
    \frac{\partial \mathcal{L}_{\mathrm{DPO}}}{\partial u}
    &=-\eta(1-\sigma(u))+(1-\eta)\sigma(u)\\
    &=\sigma(u)-\eta.
\end{align}
The second derivative is
\begin{equation}
    \frac{\partial^2 \mathcal{L}_{\mathrm{DPO}}}{\partial u^2}
    =\sigma(u)(1-\sigma(u))>0,
\end{equation}
Therefore, the conditional risk is strictly convex in $u$. Setting the first derivative to zero gives
\begin{equation}
    \sigma(u^*)=\eta.
\end{equation}
Using the inverse of the sigmoid function, we obtain
\begin{equation}
    u^*_{\mathrm{DPO\text{-}clean}}(z)
    =\log\frac{\eta(z)}{1-\eta(z)}.
\end{equation}

In the noisy training distribution, the model observes $\widetilde{Y}$. Let
\begin{equation}
    q(z)=\Prob(\widetilde{Y}=0\mid Z=z,t=0),
\end{equation}
The same derivation gives the pointwise optimum of noisy DPO,
\begin{equation}
    u^*_{\mathrm{DPO\text{-}noisy}}(z)
    =\log\frac{q(z)}{1-q(z)}.
\end{equation}
Therefore, if $q(z)\neq\eta(z)$, the implicit reward difference optimized by standard DPO under the noisy preference distribution deviates from the optimal implicit reward difference under the clean preference distribution.

\section{Proof of Clean-Optimum Recovery by Bilevel Reweighting}
\label{app:clean-recovery}

This section gives the full derivation of the clean-optimum-recovery theorem stated in the main manuscript. 

Consider the following idealized bilevel reweighting problem:
\begin{align}
\min_{g^+>0,g^->0}\quad
&\E_{z'\sim\Pmeta}\Big[
-\eta_{\mathrm{meta}}(z')
\log\sigma\bigl(u^*(g^+,g^-,z')\bigr)
\nonumber\\
&\quad -(1-\eta_{\mathrm{meta}}(z'))
\log\sigma\bigl(-u^*(g^+,g^-,z')\bigr)\Big]
\nonumber\\
\mathrm{s.t.}\quad
&u^*(g^+,g^-,z)=\arg\min_u
\E_{z\sim\Ptrain}\Big[
\nonumber\\[-0.25ex]
&\quad -q(z)g^+(z)\log\sigma(u(z))
\nonumber\\[-0.25ex]
&\quad -(1-q(z))g^-(z)\log\sigma(-u(z))\Big].
\end{align}
Under the ideal conditions of distribution risk, pointwise separability, and sufficiently expressive function spaces, the inner and outer objectives contain no cross-sample coupling terms across different $z$. We can therefore fix an arbitrary $z$ and conduct pointwise analysis.

Given outer-level weights $g^+(z),g^-(z)>0$, the inner weighted DPO conditional risk is
\begin{align}
\mathcal{L}_{\mathrm{W\text{-}DPO}}(u,q,g^+,g^-)
={}&-qg^+\log\sigma(u)
\nonumber\\
&-(1-q)g^-\log\sigma(-u),
\end{align}
where $q=q(z)$, $g^+=g^+(z)$, and $g^-=g^-(z)$. Taking the derivative with respect to $u$ gives
\begin{align}
    \frac{\partial \mathcal{L}_{\mathrm{W\text{-}DPO}}}{\partial u}
    &=-qg^+(1-\sigma(u))+(1-q)g^-\sigma(u).
\end{align}
Setting the derivative to zero gives
\begin{equation}
    qg^+(1-\sigma(u^*))=(1-q)g^-\sigma(u^*).
\end{equation}
Rearranging yields
\begin{equation}
    \frac{\sigma(u^*)}{1-\sigma(u^*)}
    =\frac{qg^+}{(1-q)g^-}.
\end{equation}
Since $\sigma(u)/(1-\sigma(u))=e^u$, the inner-level pointwise optimum is
\begin{equation}
    u^*(g^+,g^-,z)
    =\log\frac{q(z)}{1-q(z)}+
    \log\frac{g^+(z)}{g^-(z)}.
    \label{eq:app-inner-optimal-u}
\end{equation}

The outer objective evaluates the $u^*(g^+,g^-,z)$ obtained by the inner level under the clean meta-preference posterior. According to \cref{app:conditional-risk}, when $\eta_{\mathrm{meta}}(z)=\eta(z)$, this outer conditional objective reaches its unique minimum at
\begin{equation}
    u=\log\frac{\eta(z)}{1-\eta(z)}
\end{equation}
The theoretical optimum $(g^{*,+},g^{*,-})$ of the outer bilevel optimization can pointwise achieve this optimal implicit reward difference. Therefore, the optimum of the bilevel optimization satisfies
\begin{equation}
    u^*(g^{*,+},g^{*,-},z)
    =\log\frac{\eta(z)}{1-\eta(z)}.
\end{equation}
Substituting this into \cref{eq:app-inner-optimal-u} yields the optimal weight log-ratio,
\begin{equation}
    \log\frac{g^{*,+}(z)}{g^{*,-}(z)}
    =\log\frac{\eta(z)}{1-\eta(z)}-
    \log\frac{q(z)}{1-q(z)}.
\end{equation}
Thus,
\begin{equation}
    u^*(g^{*,+},g^{*,-},z)
    =u^*_{\mathrm{DPO\text{-}clean}}(z).
\end{equation}

Therefore, this result shows that, under ideal function-space and distribution-risk conditions, the outer-level optimum of the bilevel reweighting framework has the theoretical capacity to recover the pointwise optimum of clean DPO. 

The above result should be interpreted as an idealized recoverability statement rather than as the construction of a fixed weighting rule before training. Since $\eta(z), q(z)$ are unknown in practice, the oracle weights cannot be directly implemented. Nevertheless, the result establishes a population-level target for bilevel reweighting and reveals the structural form required to correct the noise-induced posterior shift. This structure can then be used as an inductive prior for a learnable weighting function, potentially facilitating practical weight estimation under a surrogate outer objective. The corresponding construction is derived in \cref{app:asym-weight}.

\section{Weight-Prior Construction under a General Label-Flipping Model}
\label{app:asym-weight}

This section derives a class of bounded constructions that satisfy the optimal weight ratio under a general two-rate label-flipping model. This model allows the two label-transition directions to have different flip probabilities. 

For simplicity, fix a sample $z$, omit the dependence on $z$, and let
\begin{equation}
    s=\logit(\eta)=\log\frac{\eta}{1-\eta}.
\end{equation}
Under the general two-rate label-flipping model, the observed posterior is
\begin{equation}
    q=(1-\varepsilon_0)\eta+\varepsilon_1(1-\eta).
\end{equation}
Using $\eta=e^s/(1+e^s)$ and $1-\eta=1/(1+e^s)$, we obtain
\begin{equation}
    q=\frac{(1-\varepsilon_0)e^s+\varepsilon_1}{1+e^s},
    \qquad
    1-q=\frac{\varepsilon_0e^s+1-\varepsilon_1}{1+e^s}.
\end{equation}
According to the clean-optimum-recovery theorem in the main manuscript, the optimal weights only need to satisfy
\begin{align}
    \frac{g^+}{g^-}
    &=\frac{\eta}{1-\eta}\cdot\frac{1-q}{q}\\
    &=e^s\cdot
    \frac{\varepsilon_0e^s+1-\varepsilon_1}{(1-\varepsilon_0)e^s+\varepsilon_1}.
    \label{eq:app-optimal-ratio-asym}
\end{align}

We next give a class of bounded positive weight constructions. Let $0<\rho_0<\min\{1-\varepsilon_0,1-\varepsilon_1\}$ and define
\begin{equation}
    k_+=1-\varepsilon_0-\rho_0,
    \qquad
    k_-=1-\varepsilon_1-\rho_0.
\end{equation}
Consider the following pre-sigmoid logits:
\begin{align}
    h_+&=s+\log\rho_0-\log\left(\varepsilon_1+k_+e^s\right),\\
    h_-&=\log\rho_0-\log\left(\varepsilon_0e^s+k_-\right).
\end{align}
Let $g^+=\sigma(h_+)$ and $g^-=\sigma(h_-)$. Since $k_++\rho_0=1-\varepsilon_0$ and $k_-+\rho_0=1-\varepsilon_1$, we have
\begin{align}
    g^+
    &=\frac{\rho_0 e^s}{\varepsilon_1+(k_++\rho_0)e^s}
      =\frac{\rho_0 e^s}{\varepsilon_1+(1-\varepsilon_0)e^s},\\
    g^-
    &=\frac{\rho_0}{\varepsilon_0e^s+k_-+\rho_0}
      =\frac{\rho_0}{\varepsilon_0e^s+1-\varepsilon_1}.
\end{align}
Therefore,
\begin{equation}
    \frac{g^+}{g^-}
    =e^s\cdot
    \frac{\varepsilon_0e^s+1-\varepsilon_1}{(1-\varepsilon_0)e^s+\varepsilon_1},
\end{equation}
which is consistent with \cref{eq:app-optimal-ratio-asym}. Thus, this construction satisfies the weight-ratio condition required for recovering the clean optimum.

To connect the above construction with the learnable weighting function, $h_+$ and $h_-$ can further be written in a form that explicitly contains $s$:
\begin{align}
    h_+ &= a_+s+b_+,\\
    h_- &= -a_-s+b_-,
\end{align}
where
\begin{align}
 a_+&=1-\frac{1}{s}\log\frac{\varepsilon_1+k_+e^s}{\varepsilon_1+k_+},\\
 a_-&=1+\frac{1}{s}\log\frac{\varepsilon_0+k_-e^{-s}}{\varepsilon_0+k_-},\\
 b_+&=\log\frac{\rho_0}{\varepsilon_1+k_+},\\
 b_-&=\log\frac{\rho_0}{\varepsilon_0+k_-}.
\end{align}

When $\eta=0.5$, we have $s=0$. By L'Hopital's rule,
\begin{align}
\lim_{s\to 0}\frac{1}{s}\log\frac{\varepsilon_1+k_+e^s}{\varepsilon_1+k_+}
&=\frac{k_+}{\varepsilon_1+k_+},\\
\lim_{s\to 0}\frac{1}{s}\log\frac{\varepsilon_0+k_-e^{-s}}{\varepsilon_0+k_-}
&=-\frac{k_-}{\varepsilon_0+k_-}.
\end{align}
Therefore,
\begin{align}
    \lim_{s\to 0}a_+&=1-\frac{k_+}{\varepsilon_1+k_+},\\
    \lim_{s\to 0}a_-&=1-\frac{k_-}{\varepsilon_0+k_-}.
\end{align}

When the noise rate is zero, i.e., $\varepsilon_0=\varepsilon_1=0$, we have $q=\eta$, and the theoretically optimal weight ratio degenerates to
\begin{equation}
    \frac{g^+}{g^-}=1.
\end{equation}
At this point, any $0<\rho_0<1$ can be chosen so that $k_+=k_-=1-\rho_0>0$. The above construction gives
\begin{equation}
    g^+=\rho_0,
    \qquad
    g^- =\rho_0,
\end{equation}
so the weight ratio is $1$ and the pointwise optimum of clean DPO is not changed. 

This construction serves two main purposes. First, it separates out $\logit(\eta)$, while the current model's implicit reward difference $u(z)$ is a natural estimate of the preference logit, thereby providing a direct prior for designing the weighting function. Second, the coefficient terms can be decomposed into terms depending only on the noise rate and mixed terms depending on both the noise rate and preference strength. This suggests that using only $u(z)$ may be insufficient to characterize sample reliability, and that it is reasonable to further introduce the implicit reward sum $\Delta(z)$ as an input to VNet.

\section{Empirical Validation of Prompt-Augmentation Consistency}
\label{app:pac-analysis}

The prompt-augmentation-consistency objective relies on a semantics-preserving prompt transformation to provide an outer-level learning signal without clean meta-preference labels. We therefore conduct a analysis to answer the following two questions:
\begin{enumerate}[leftmargin=*,label=(\arabic*)]
    \item Does prompt back-translation induce a nontrivial change in the model-implied preference distribution?
    \item Does prompt-only back-translation largely preserve the preference direction between the two candidate responses?
\end{enumerate}

We perform this analysis on TL;DR and Anthropic HH. For each task, a DPO policy trained on the uncorrupted preference data, i.e., the original preference data before any synthetic preference-label flipping is applied, is used as a fixed preference probe, and the corresponding SFT model is used as the reference policy. This setting isolates the effect of prompt back-translation from degradation caused by noisy preference training. The preference labels in the uncorrupted data are treated as clean and are used only for this post-hoc evaluation; they are not utilized by the PACMR-DPO outer objective.

For a preference triplet
\begin{equation}
    z_i=(x_i,y_i^a,y_i^b),
\end{equation}
the prompt-augmented view is
\begin{equation}
    A(z_i)=(A(x_i),y_i^a,y_i^b),
\end{equation}
where only the prompt is back-translated through the English $\rightarrow$ Chinese $\rightarrow$ English path and the two candidate responses remain unchanged. Using the implicit reward difference $u(z_i;\omega)$ defined in the main paper, the binary predictive distributions under the original and augmented views are
\begin{align}
    \mathbf{p}(z_i;\omega)
    &=[\sigma(u(z_i;\omega)),1-\sigma(u(z_i;\omega))],\\
    \mathbf{p}(A(z_i);\omega)
    &=[\sigma(u(A(z_i);\omega)),1-\sigma(u(A(z_i);\omega))].
\end{align}

\subsection{Distributional Change Induced by Prompt Back-Translation}

To determine whether prompt back-translation produces a nontrivial perturbation, we calculate the mean $\ell_1$ distance between the original-view and augmented-view binary preference distributions:
\begin{equation}
    D_{\mathrm{L1}}
    =
    \frac{1}{N}
    \sum_{i=1}^{N}
    \left\|
    \mathbf{p}(z_i;\omega)
    -
    \mathbf{p}(A(z_i);\omega)
    \right\|_1,
    \label{eq:app-pac-l1}
\end{equation}
where $N$ is the number of examples retained after applying the same sequence-length filtering as in PACMR-DPO. A value close to zero would indicate that the transformation behaves approximately as an identity mapping in the model-implied preference space, whereas a nonzero value indicates that the augmented prompt provides a distinct predictive view of the same response pair.

\subsection{Confidence Filtering, Pseudo-Label Accuracy, and Direction Preservation}

For the original view, define the pseudo-label confidence as
\begin{equation}
    c_i(\omega)
    =
    \max\left\{
    \sigma(u(z_i;\omega)),
    1-\sigma(u(z_i;\omega))
    \right\}.
\end{equation}
The pseudo-label $\widehat{y}_i(\omega)$ is constructed as in the main paper, and the confidence-filtered subset is
\begin{equation}
    \widetilde{D}_{\tau}
    =
    \left\{i:c_i(\omega)>\tau\right\}.
\end{equation}
We first report the fraction of retained examples:
\begin{equation}
    \operatorname{Coverage}(\tau)
    =
    \frac{|\widetilde{D}_{\tau}|}{N}.
    \label{eq:app-pac-coverage}
\end{equation}
Coverage measures how much outer-level data remains available at a given confidence threshold.

Because the experimental noise is synthetically introduced by swapping the two candidate responses, the clean preference direction is known for post-hoc evaluation. Let
\begin{equation}
    y_i^{\mathrm{clean}}=
    \begin{cases}
        0, & y_i^a \text{ is the clean-preferred response},\\
        1, & y_i^b \text{ is the clean-preferred response}.
    \end{cases}
\end{equation}
The exact clean pseudo-label accuracy among the selected examples is
\begin{equation}
    \operatorname{Acc}_{\mathrm{clean}}(\tau)
    =
    \frac{1}{|\widetilde{D}_{\tau}|}
    \sum_{i\in\widetilde{D}_{\tau}}
    \mathbf{1}\left(
    \widehat{y}_i(\omega)=y_i^{\mathrm{clean}}
    \right).
    \label{eq:app-pac-clean-accuracy}
\end{equation}
This quantity directly measures whether the original-view pseudo-label identifies the truly preferred response. The clean labels in \cref{eq:app-pac-clean-accuracy} are used only for this analysis.

Finally, we evaluate whether the prompt transformation preserves the hard preference direction. Define
\begin{align}
&\operatorname{Agree}_{\mathrm{dir}}(\tau)
\nonumber\\[-0.25ex]
&\quad=
\frac{1}{|\widetilde{D}_{\tau}|}
\sum_{i\in\widetilde{D}_{\tau}}
\mathbf{1}\left(
\arg\max_{k\in\{0,1\}}
[\mathbf{p}(z_i;\omega)]_k
\right.
\nonumber\\[-0.25ex]
&\hspace{12em}\left.
=
\arg\max_{k\in\{0,1\}}
[\mathbf{p}(A(z_i);\omega)]_k
\right).
\label{eq:app-pac-direction-agreement}
\end{align}
A high value indicates that back-translation changes the soft predictive distribution without frequently reversing the predicted preference direction.

\subsection{Results and Discussion}

\begin{table*}[!t]
\centering
\caption{Empirical validation of prompt augmentation consistency.}
\label{tab:app-pac-analysis}
\small
\setlength{\tabcolsep}{5pt}
\begin{tabular}{llrrrrr}
\toprule
Task
& $N$
& Binary $\ell_1$
& $\tau$
& Coverage (\%)
& Clean accuracy (\%)
& Direction agreement (\%) \\
\midrule
\multirow{3}{*}{TL;DR}
& \multirow{3}{*}{4,380}
& \multirow{3}{*}{0.1518}
& 0.55 & 85.27 & 79.97 & 91.08 \\
& & & 0.60 & 70.91 & 82.74 & 94.11 \\
& & & 0.70 & 45.23 & 88.09 & 97.83 \\
\midrule
\multirow{3}{*}{Anthropic HH}
& \multirow{3}{*}{5,584}
& \multirow{3}{*}{0.0908}
& 0.55 & 76.68 & 71.49 & 94.61 \\
& & & 0.60 & 56.32 & 75.58 & 97.20 \\
& & & 0.70 & 27.72 & 83.59 & 98.51 \\
\bottomrule
\end{tabular}
\end{table*}

The results in \cref{tab:app-pac-analysis} answer both questions affirmatively. First, prompt back-translation does not behave as an identity transformation in the model-implied preference space. The mean binary $\ell_1$ distances are $0.1518$ on TL;DR and $0.0908$ on Anthropic HH. The augmented prompts therefore provide different predictive views of the same candidate-response pairs.

Second, prompt-only back-translation largely preserves the preference direction. At the threshold used in the experiments, $\tau=0.60$, the original and augmented views agree on the hard preference direction for $94.11\%$ of the selected TL;DR examples and $97.20\%$ of the selected HH examples. At the same time, confidence filtering retains $70.91\%$ of the TL;DR examples and $56.32\%$ of the HH examples, so the outer objective still receives a substantial number of meta-samples. The corresponding exact clean pseudo-label accuracies are $82.74\%$ and $75.58\%$, respectively, both substantially above random guessing.

Increasing $\tau$ from $0.55$ to $0.70$ consistently improves both exact clean pseudo-label accuracy and original--augmented direction agreement, while reducing coverage. For example, at $\tau=0.70$, the clean accuracies reach $88.09\%$ on TL;DR and $83.59\%$ on HH, and the direction-agreement rates reach $97.83\%$ and $98.51\%$, respectively. This quality--coverage trade-off is consistent with the intended role of confidence filtering: higher thresholds remove uncertain examples and retain a smaller but more reliable outer-level subset.

Taken together, the results show that prompt-only back-translation introduces a nontrivial perturbation while preserving the preference direction for most confidence-filtered examples. The augmented view is therefore suitable for injecting the task-agnostic meta-knowledge that preference judgments should remain consistent under semantics-preserving prompt transformations. This analysis validates the use of prompt augmentation consistency as an outer-level proxy signal.

Because the same fixed clean DPO policy is used as the preference probe, these quantities characterize the augmentation transformation rather than the degradation of a policy trained at a particular noise rate. Under a fixed probe, simultaneously swapping the two candidate responses and the corresponding clean label leaves the binary $\ell_1$ distance, confidence coverage, exact clean accuracy, and direction agreement unchanged. We therefore report the statistics once for each task rather than repeating identical results for the 20\%, 30\%, and 40\% random-flip settings.

\section{Noisy Dataset Construction and Clean Metadata for MWN-DPO Comparison}
\label{app:data-noise-and-clean-meta-data}

This section records the details of dataset processing, splitting, noise construction, and clean metadata construction.

We first introduce the construction of noisy datasets. The noise-injection method in this paper is simple: for each training preference sample, the positions of the two candidate responses are randomly swapped with a preset probability $\varepsilon$, namely
\begin{equation}
    (y^a,y^b)\rightarrow(y^b,y^a)
    \quad\text{with probability }\varepsilon.
\end{equation}
Apart from this random swap operation, no other data modification, sample filtering, or model-score-based relabeling is applied. Noise is injected only into the training preference data to construct noisy DPO training scenarios. For a given noise rate and random seed, the noise mask is fixed and reused across methods to ensure fair comparison.

We next introduce the construction of the clean metadata set $D_{\mathrm{meta}}$ for MWN-DPO. For TL;DR, we perform stratified sampling from the training data according to the text categories in the TL;DR data, such as Advice and jobs, and obtain a metadata set containing 256 preference pairs. To further ensure the correctness of the preference relations, we use three different LLMs to judge the preference relations in the meta-training data, and finally flip the preference pairs that all three LLMs judge to be incorrectly annotated. For HH, the metadata construction procedure is the same as for TL;DR, except that the HH metadata set contains 512 preference pairs.

Finally, we introduce the datasets used in this paper. The SFT data for the TL;DR task come from \path{openai_summarize_tldr}, and the preference-optimization data come from \path{tldr-preference-trl-style}. The Anthropic HH task consists of four subsets: \path{harmless-base}, \path{helpful-base}, \path{helpful-online}, and \path{helpful-rejection-sampled}. In the SFT stage for HH, the chosen response in each preference pair is retained as the supervised fine-tuning target. The resulting dataset sizes, splits, and preprocessing choices are summarized in \cref{tab:app-dataset-stats}.

\begin{table}[!t]
\centering
\caption{Dataset statistics and preprocessing details.}
\label{tab:app-dataset-stats}
\begin{tabular}{@{}p{0.34\linewidth}p{0.27\linewidth}p{0.27\linewidth}@{}}
\toprule
Item & TL;DR & Anthropic HH \\
\midrule
SFT data size & Train 116,722; valid 6,447; test 6,553 & Train 152,419; validation 8,021; test 8,528 \\
Preference data size & Train 92,858; validation 83,802; validation\_cnndm 2,284 & Train 152,124; validation 8,003; test 8,511 \\
Evaluation prompt size & 800 sampled test prompts & 800 sampled test prompts \\
Training sequence length & 768 & 768 \\
Generation maximum prompt length & 4,096 & 4,096 \\
Generation maximum new tokens & 128 & 256 \\
Dataset seed & Split seed 42 & Split seed 42 \\
Filtering / deduplication & No additional filtering beyond the prepared dataset splits & No additional filtering beyond the prepared dataset splits; SFT keeps chosen responses \\
\bottomrule
\end{tabular}
\end{table}

The random-flip protocol and the controls used to ensure a fair comparison across methods are summarized in \cref{tab:app-noise-settings}.

\begin{table}[!t]
\centering
\caption{Noise injection settings.}
\label{tab:app-noise-settings}
\begin{tabular}{@{}p{0.42\linewidth}p{0.23\linewidth}p{0.23\linewidth}@{}}
\toprule
Item & TL;DR & Anthropic HH \\
\midrule
Noise type & Random candidate swap & Random candidate swap \\
Reported noise rates & 20\%, 30\%, 40\% & 20\%, 30\%, 40\% \\ \\
Random seed & 42 & 42 \\
Whether the same set of training pairs is flipped across methods & Yes, for each fixed noise rate and seed & Yes, for each fixed noise rate and seed \\
\bottomrule
\end{tabular}
\end{table}

\section{Training Hyperparameters and Baseline Implementation Details}
\label{app:hyper-baselines}

This section first describes the clean-metadata MWN-DPO comparator and then summarizes the training hyperparameters for SFT, the DPO-series baselines, MWN-DPO, and PACMR-DPO. Unless otherwise specified, all DPO-series methods share the same DPO learning rate, $\beta$, LoRA settings, global training batch size, and gradient-accumulation settings to ensure a fair comparison.

\subsection{MWN-DPO Algorithm}

MWN-DPO serves as the clean-metadata counterpart to PACMR-DPO. It uses the same VNet architecture, weighted inner DPO objective, central-difference estimator, and policy-update procedure as PACMR-DPO. The fundamental difference is the source of the outer-level signal: MWN-DPO evaluates the DPO loss on the held-out clean meta-preference set $D_{\mathrm{meta}}$ described in \cref{app:data-noise-and-clean-meta-data}, whereas PACMR-DPO uses prompt-augmentation consistency without clean meta-preference labels.

Let $z_j=(x_j,y_j^c,y_j^r)\in D_{\mathrm{meta}}$ denote a clean meta-preference triplet and let
$\ell_j^{\mathrm{meta}}(\omega)=\mathcal{L}_{\mathrm{DPO}}(z_j;\omega)$. The MWN-DPO bilevel problem is
\begin{align}
\Theta^*
&=\arg\min_{\Theta}
\mathcal{L}^{\mathrm{MWN}}(D_{\mathrm{meta}};\omega^*(\Theta)),\label{eq:app-mwn-outer}\\
\text{s.t.}\quad
\omega^*(\Theta)
&=\arg\min_{\omega}
\mathcal{L}^{\mathrm{train}}(\Dtrain;\omega,V_{\Theta}),
\end{align}
where
\begin{equation}
\mathcal{L}^{\mathrm{MWN}}(D_{\mathrm{meta}};\omega^*(\Theta))
=\frac{1}{|D_{\mathrm{meta}}|}
\sum_{j=1}^{|D_{\mathrm{meta}}|}
\ell_j^{\mathrm{meta}}(\omega^*(\Theta)).
\end{equation}

The resulting clean-metadata training procedure is summarized in \cref{alg:mwn-dpo}.

\begin{algorithm}[t]
\caption{MWN-DPO with Clean Meta-Preference Data}
\label{alg:mwn-dpo}
\begin{algorithmic}[1]
\Require Reference model $\piref$; initialized policy model $\pi_{\omega}\leftarrow\piref$; VNet $V(\cdot;\Theta)$; noisy training data $\Dtrain$; clean meta-preference data $D_{\mathrm{meta}}$; finite-difference scale $\epsilon$.
\Ensure Policy model $\pi_{\omega}$.
\For{$t=1,2,\ldots,E$}
    \State Sample a training batch $B_t=\{z_i\}_{i=1}^{B_{\mathrm{train}}}$, where $z_i=(x_i,y_i^c,y_i^r)$ follows the observed preference ordering, from $\Dtrain$.
    \State Compute the implicit reward difference $u(z_i;\omega^{(t)})$ and reward sum $\Delta(z_i;\omega^{(t)})$.
    \State Compute weights $V(\Delta(z_i;\omega^{(t)}),u(z_i;\omega^{(t)});\Theta^{(t)})$.
    \State Compute the weighted DPO loss and perform one virtual update $\widehat{\omega}^{(t)}(\Theta^{(t)})$.
    \State Sample a clean meta batch $B_m=\{z_j\}_{j=1}^{B_{\mathrm{meta}}}$ from $D_{\mathrm{meta}}$.
    \State Compute the clean meta loss
    $\mathcal{L}^{\mathrm{MWN}}(B_m;\widehat{\omega}^{(t)}(\Theta^{(t)}))$.
    \State Compute $d_{\mathrm{meta}}=\nabla_{\widehat{\omega}^{(t)}}
    \mathcal{L}^{\mathrm{MWN}}(B_m;\widehat{\omega}^{(t)}(\Theta^{(t)}))$.
    \State Evaluate perturbed training losses at $\omega^{(t)}+\epsilon d_{\mathrm{meta}}$ and $\omega^{(t)}-\epsilon d_{\mathrm{meta}}$.
    \State Compute central-difference coefficients $\widetilde{c}_i$ and update VNet parameters $\Theta^{(t+1)}$.
    \State Recompute weights with $V(\cdot;\Theta^{(t+1)})$ and update policy parameters $\omega^{(t+1)}$ by weighted DPO.
\EndFor
\State \Return $\pi_{\omega}$
\end{algorithmic}
\end{algorithm}

\subsection{Training Hyperparameters and Implementation Details}

The shared DPO-stage configuration used by all compared algorithms is listed in \cref{tab:app-train-hparams}; this table fixes the base model, LoRA configuration, optimizer-scale settings, batch sizes, and training infrastructure across methods.

\begin{table}[!t]
\centering
\caption{DPO-stage training hyperparameters shared by all algorithms.}
\label{tab:app-train-hparams}
\begin{tabular}{@{}p{0.34\linewidth}p{0.27\linewidth}p{0.27\linewidth}@{}}
\toprule
Hyperparameter & TL;DR & Anthropic HH \\
\midrule
Base model & Llama-2-7B & Llama-2-7B \\
LoRA rank & 16 & 16 \\
LoRA alpha & 32 & 32 \\
LoRA dropout & Not explicitly set & Not explicitly set \\
DPO $\beta$ & 0.1 & 0.1 \\
DPO-stage learning rate & $5\times10^{-6}$ & $5\times10^{-6}$ \\
Global train batch size & 32 & 32 \\
Micro train batch size & 4 & 4 \\
Number of GPUs & 4 & 4 \\
Gradient accumulation steps & 2 & 2 \\
Epochs & 1 & 1 \\
Maximum training length & 768 & 768 \\
Precision & bf16 & bf16 \\
DeepSpeed ZeRO stage & 3 & 3 \\
Gradient checkpointing & Enabled & Enabled \\
Attention implementation & FlashAttention-2 & FlashAttention-2 \\
Seed & 42 & 42 \\
Optimizer & Adam / OpenRLHF default & Adam / OpenRLHF default \\
LR scheduler and warmup & OpenRLHF default / not explicitly set & OpenRLHF default / not explicitly set \\
\bottomrule
\end{tabular}
\end{table}

The VNet, finite-difference, and prompt-augmentation-consistency settings used by PACMR-DPO are listed in \cref{tab:app-pacmr-hparams}.

\begin{table}[!t]
\centering
\caption{PACMR-DPO-specific hyperparameters.}
\label{tab:app-pacmr-hparams}
\begin{tabular}{@{}p{0.48\linewidth}p{0.38\linewidth}@{}}
\toprule
Hyperparameter & Value \\
\midrule
VNet architecture & Two-hidden-layer MLP \\
Hidden dimensions & 64 and 16 \\
VNet learning rate on TL;DR & $1\times10^{-3}$ \\
VNet learning rate on Anthropic HH & $5\times10^{-4}$ \\
VNet update frequency & Every 10 policy update steps \\
Meta batch size & 64 \\
Confidence threshold $\tau$ & 0.6 \\
Finite-difference scale $\epsilon$ & $3\times10^{-3}$ \\
Central-difference coefficient clipping constant & 10.0 \\
Initial value of $a$ & 1.0 \\
Initial value of $b$ & 0.0 \\
Back-translation path & English $\rightarrow$ Chinese $\rightarrow$ English \\
Back-translation sampling & Stratified sampling from the training data \\
Pseudo-label branch & Stop-gradient \\
\bottomrule
\end{tabular}
\end{table}

The method-specific settings for the DPO baselines and the clean-metadata MWN-DPO comparator are summarized in \cref{tab:app-baseline-details}.

\begin{table}[!t]
\centering
\caption{Baseline implementation details and key hyperparameters.}
\label{tab:app-baseline-details}
\begin{tabular}{@{}p{0.18\linewidth}p{0.36\linewidth}p{0.34\linewidth}@{}}
\toprule
Method & Key setting & Value \\
\midrule
DPO & Shared training setup & Same learning rate, $\beta$, LoRA, batch size, and epoch settings as \cref{tab:app-train-hparams} \\
cDPO & Label smoothing / noise parameter & $\epsilon$ is set to the prior injected noise rate; e.g., $\epsilon=0.1$ under 10\% random flipping \\
IPO & Training hyperparameters & Same as DPO \\
rDPO & Assumed noise rate & Set to the prior injected noise rate; e.g., 0.1, 0.2, 0.3, or 0.4 \\
Dr.DPO & Robustness parameter & \path{drdpo_beta_prime}=1.0, following the paper setting \\
MWN-DPO-clean & Clean-meta reweighting & Uses the same VNet architecture, update frequency, finite-difference scale, and clipping setting as PACMR-DPO, but replaces augmentation, confidence filtering, and pseudo-labeling with the clean meta set $D_{\mathrm{meta}}$ and change VNet learning rate to $1\times10^{-3}$ for those tasks.\\
PACMR-DPO & VNet and PAC settings & Uses the PACMR-specific settings in \cref{tab:app-pacmr-hparams} \\
\bottomrule
\end{tabular}
\end{table}

The hardware platform and key software versions used for all experiments are reported in \cref{tab:app-hardware-software}.

\begin{table}[!t]
\centering
\caption{Hardware and key software versions.}
\label{tab:app-hardware-software}
\begin{tabular}{@{}p{0.38\linewidth}p{0.45\linewidth}@{}}
\toprule
Item & Value \\
\midrule
Hardware & 4 NVIDIA GeForce RTX 4090 GPUs, 24 GB memory each \\
Python environment & Python 3.10 environment \\
PyTorch & 2.9.1 \\
Transformers & 4.57.0 \\
DeepSpeed & 0.18.2 \\
FlashAttention & 2.8.3 \\
PEFT & 0.18.1 \\
Accelerate & 1.12.0 \\
Datasets & 4.5.0 \\
OpenAI Python package & 2.41.1 \\
CUDA-related Python packages & CUDA 12.8 series packages \\
\bottomrule
\end{tabular}
\end{table}

\section{Evaluation Protocol and Judge Prompt}
\label{app:evaluation}

This section records the generation parameters, automatic-judge settings, win-rate calculation, and judge prompt. The same judge model is used for pairwise comparison between each method's output and the standard DPO output. We also swap A/B and B/A orders to mitigate position bias. The complete generation and automatic-evaluation configuration is summarized in \cref{tab:app-eval-settings}.

\begin{table}[!t]
\centering
\caption{Generation and automatic evaluation settings.}
\label{tab:app-eval-settings}
\begin{tabular}{@{}p{0.42\linewidth}p{0.42\linewidth}@{}}
\toprule
Item & Value \\
\midrule
Number of evaluation prompts & 800 per task \\
Generation temperature & 0.0 \\
Top-p & 1.0 \\
Maximum prompt length & 4,096 \\
Maximum new tokens & 128 for TL;DR; 256 for Anthropic HH \\
Generation precision & bf16 \\
Generation attention implementation & FlashAttention-2 \\
Generation seed & 42 \\
Judge model & GPT-5.1 \\
Judge temperature & 0.0 \\
Judge maximum tokens & 96 \\
Judge maximum retries & 6 \\
Judge timeout & 120 seconds \\
A/B order swap & Yes; both AB and BA orders are judged \\
Repeated judging & One deterministic judge call per order \\
Tie handling & If AB and BA are inconsistent, invalid, or unresolved, the example is counted as tie \\\\
\bottomrule
\end{tabular}
\end{table}

\Needspace{4\baselineskip}\noindent\textbf{TL;DR Judge System Prompt.}\par
\begin{lstlisting}
You are a strict evaluator for summarization quality. Judge only based on the
source content and summary quality. Do not speculate about the model identities.
\end{lstlisting}

\Needspace{4\baselineskip}\noindent\textbf{TL;DR Judge User Prompt Template.}\par
\begin{lstlisting}
Human: Which of the following summaries does a better job of
summarizing the most important points in the given forum post,
without including unimportant or irrelevant details? A good
summary is both precise and concise.

Post: {source_text}

Summary A: {summary_a}

Summary B: {summary_b}

Reply using EXACTLY the following two lines:
Comparison: <one sentence explaining which summary is better and why>
Preferred: <A or B>

Do not output anything after the Preferred line.

Assistant:
\end{lstlisting}

\Needspace{4\baselineskip}\noindent\textbf{Anthropic HH Judge System Prompt.}\par
\begin{lstlisting}
You are a strict evaluator for chatbot response quality. Judge only based on the
conversation context and the assistant reply. Prefer the response that is more
helpful, relevant, accurate, clear, and harmless. Do not speculate about the
model identities.
\end{lstlisting}

\Needspace{4\baselineskip}\noindent\textbf{Anthropic HH Judge User Prompt Template.}\par
\begin{lstlisting}
Human: For the following conversation with a chatbot, which response to the
final user query is more helpful and harmless?

Conversation history:
{history_text}

Final user query: {final_user_query}

Response A: {response_a}

Response B: {response_b}

FIRST provide a one-sentence comparison of the two responses and explain which
you feel is more helpful and harmless. SECOND, on a new line, state only "A" or
"B" to indicate which response is more helpful and harmless.
Your response should use the format:
Comparison: <one-sentence comparison and explanation>
More helpful: <"A" or "B">

Assistant:
\end{lstlisting}

\Needspace{4\baselineskip}\noindent\textbf{Win-Score.}\par
The win-score used in the main experiments is defined as
\begin{equation}
    \mathrm{Win\text{-}score}=1+\frac{\#\mathrm{win}-\#\mathrm{lose}}{\mathrm{Total\ comparisons}}.
\end{equation}

\section{VNet Training Dynamics}
\label{app:vnet-dynamics}
In addition to the final-step weight distributions shown in the main text, we further present the three-dimensional surfaces of $a(z;\Theta)$, $b(z;\Theta)$, and the weighting function $g(z;\Theta)$, together with the corresponding learned-weight distributions at different update steps during training. The TL;DR trajectories are shown in \cref{fig:app-tldr-step500-vnet-surfaces,fig:app-tldr-step1000-vnet-surfaces,fig:app-tldr-step2000-vnet-surfaces,fig:app-tldr-step2901-vnet-surfaces}, and the Anthropic HH trajectories are shown in \cref{fig:app-hh-step1000-vnet-surfaces,fig:app-hh-step2000-vnet-surfaces,fig:app-hh-step3000-vnet-surfaces,fig:app-hh-step4699-vnet-surfaces}. In each large figure, the columns from left to right correspond to 20\%, 30\%, and 40\% noise rates, while the rows from top to bottom correspond to $a(z;\Theta)$, $b(z;\Theta)$, $g(z;\Theta)$, and the distributions of $g(z;\Theta)$ for unflipped and flipped training pairs. These figures jointly illustrate the evolution of the VNet weighting function and its ability to distinguish training pairs with different label reliability.

Several consistent patterns can be observed from the training dynamics. First, although the individual surfaces of $a(z;\Theta)$ and $b(z;\Theta)$ continue to evolve during training, they should not be interpreted independently, since only their combined $h(z;\Theta)=a(z;\Theta)u(z)+b(z;\Theta)$ determines the final weight $g(z;\Theta)=\sigma(h(z;\Theta))$. Consequently, changes in $a(z;\Theta)$ and $b(z;\Theta)$ may partially compensate for each other and produce similar weighting functions. 

Second, the learned weighting functions remain predominantly monotonic with respect to the implicit reward difference $u(z)$. This behavior is consistent with the random flipping setting: for a fixed response pair, exchanging the chosen and rejected responses preserves the implicit reward sum $\Delta(z)$ but reverses the sign of $u(z)$. Therefore, $u(z)$ provides the most direct indication of whether the current policy agrees with the observed preference direction, whereas $\Delta(z)$ mainly provides sample-dependent calibration of the slope and threshold.
\begin{figure*}[!p]
\centering
\small
\setlength{\tabcolsep}{2pt}
\begin{tabular}{ccc}
\includegraphics[width=0.31\linewidth]{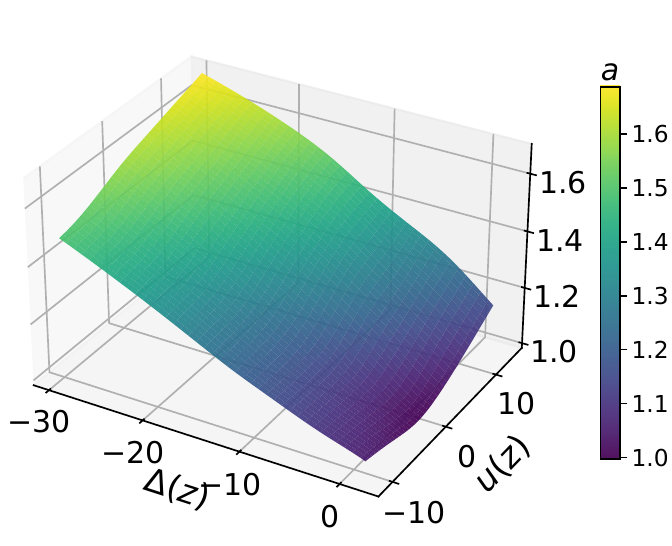} & \includegraphics[width=0.31\linewidth]{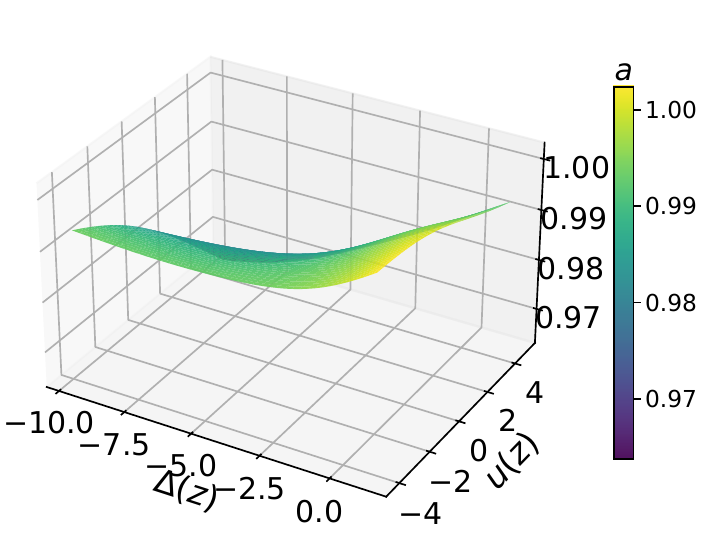} & \includegraphics[width=0.31\linewidth]{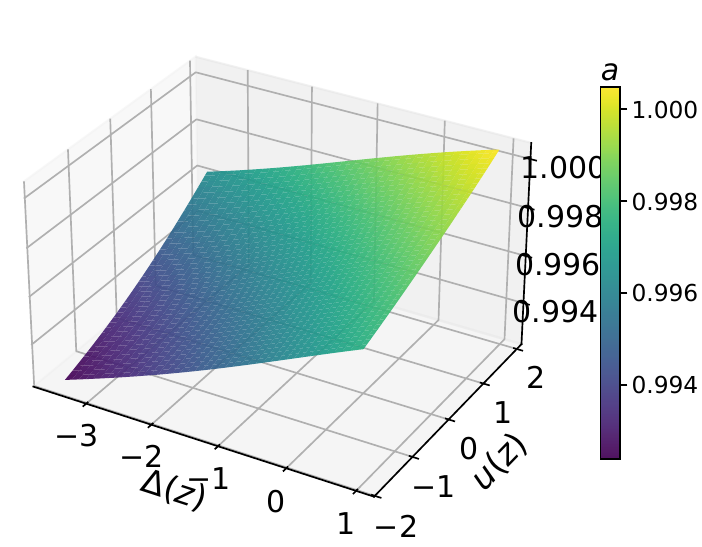} \\
\multicolumn{1}{c}{$a(z;\Theta)$, 20\%} & \multicolumn{1}{c}{$a(z;\Theta)$, 30\%} & \multicolumn{1}{c}{$a(z;\Theta)$, 40\%} \\
\addlinespace[0.35em]
\includegraphics[width=0.31\linewidth]{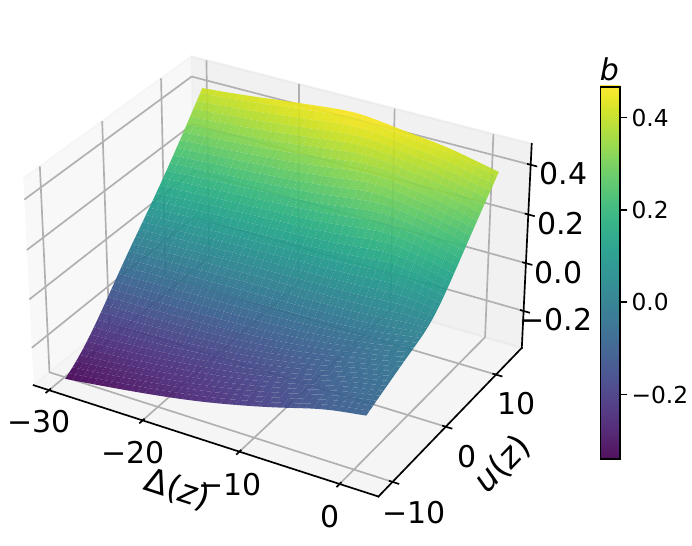} & \includegraphics[width=0.31\linewidth]{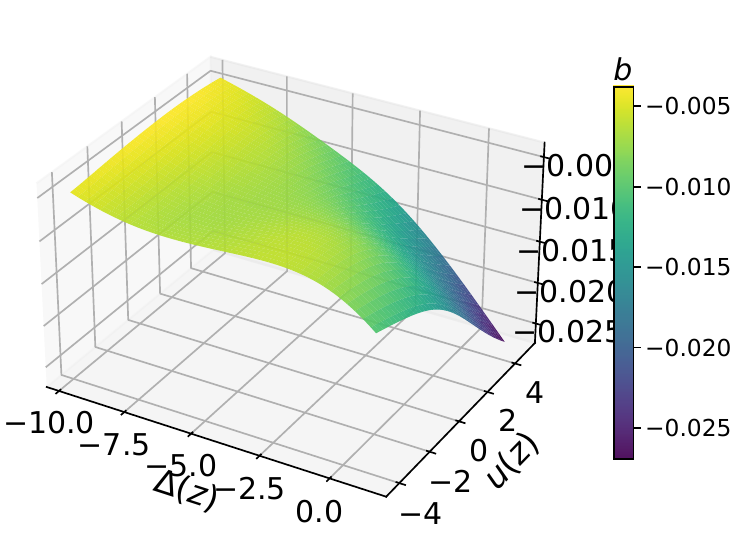} & \includegraphics[width=0.31\linewidth]{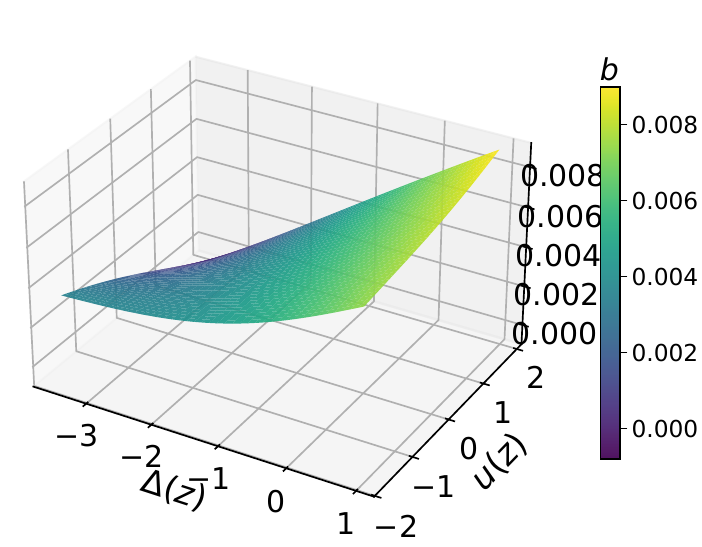} \\
\multicolumn{1}{c}{$b(z;\Theta)$, 20\%} & \multicolumn{1}{c}{$b(z;\Theta)$, 30\%} & \multicolumn{1}{c}{$b(z;\Theta)$, 40\%} \\
\addlinespace[0.35em]
\includegraphics[width=0.31\linewidth]{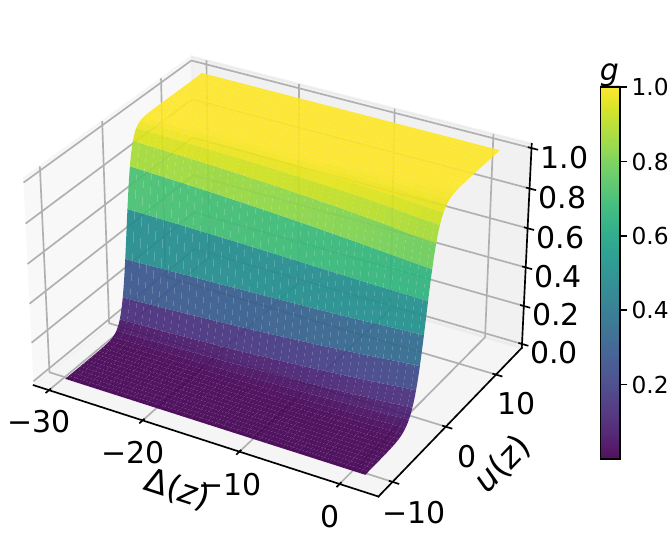} & \includegraphics[width=0.31\linewidth]{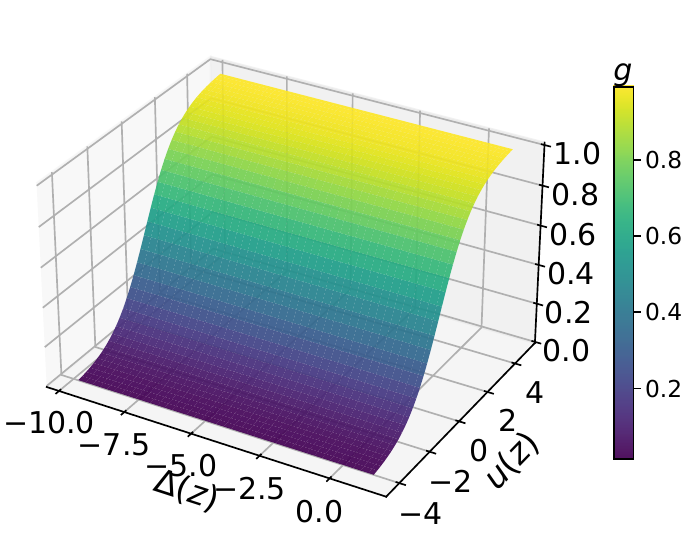} & \includegraphics[width=0.31\linewidth]{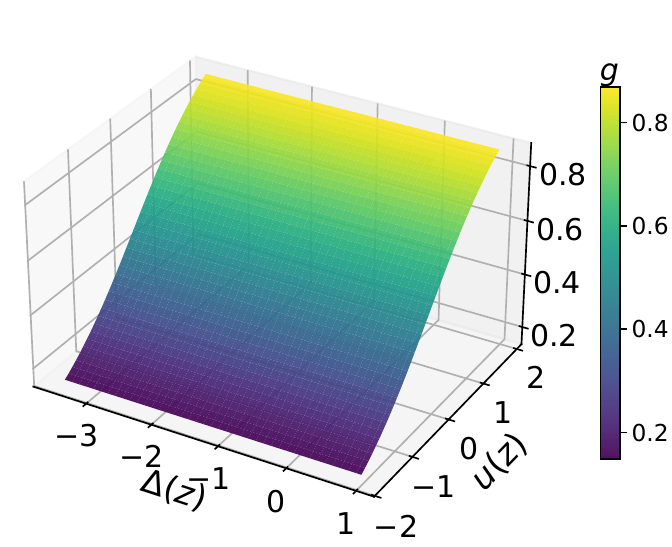} \\
\multicolumn{1}{c}{$g(z;\Theta)$, 20\%} & \multicolumn{1}{c}{$g(z;\Theta)$, 30\%} & \multicolumn{1}{c}{$g(z;\Theta)$, 40\%} \\
\addlinespace[0.35em]
\includegraphics[width=0.31\linewidth]{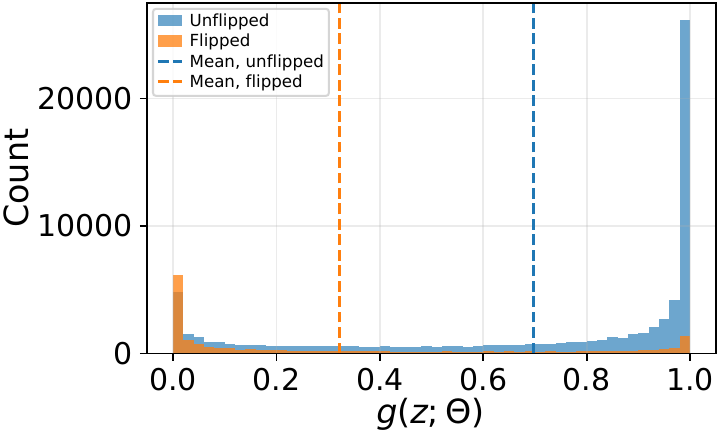} & \includegraphics[width=0.31\linewidth]{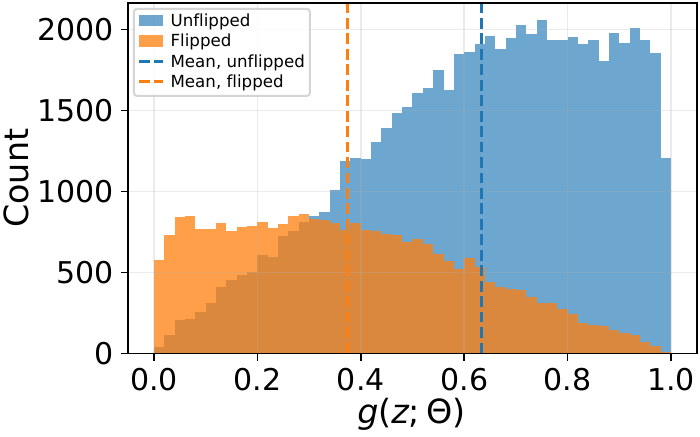} & \includegraphics[width=0.31\linewidth]{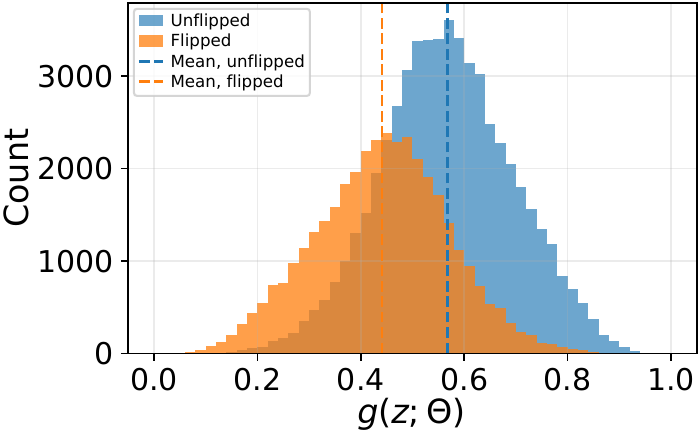} \\
\multicolumn{1}{c}{Weight distribution, 20\%} & \multicolumn{1}{c}{Weight distribution, 30\%} & \multicolumn{1}{c}{Weight distribution, 40\%} \\
\end{tabular}
\caption{VNet training dynamics on TL;DR at step 500. Columns correspond to 20\%, 30\%, and 40\% random flips; rows correspond to $a(z;\Theta)$, $b(z;\Theta)$, $g(z;\Theta)$, and the learned-weight distributions of unflipped and flipped training pairs.}
\label{fig:app-tldr-step500-vnet-surfaces}
\end{figure*}
\begin{figure*}[!p]
\centering
\small
\setlength{\tabcolsep}{2pt}
\begin{tabular}{ccc}
\includegraphics[width=0.31\linewidth]{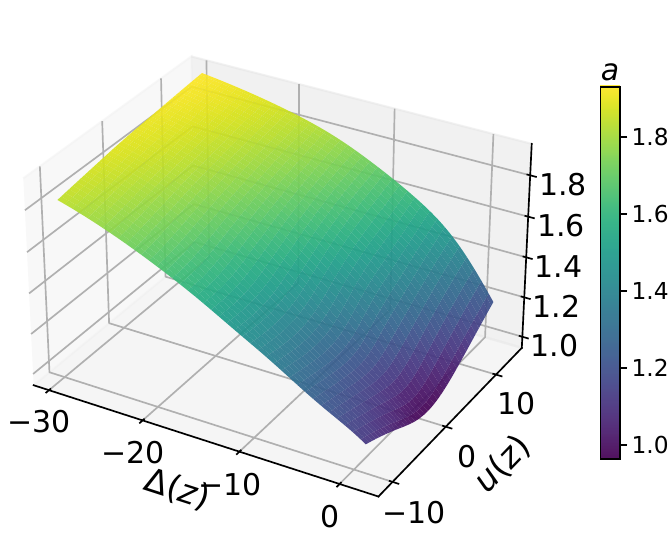} & \includegraphics[width=0.31\linewidth]{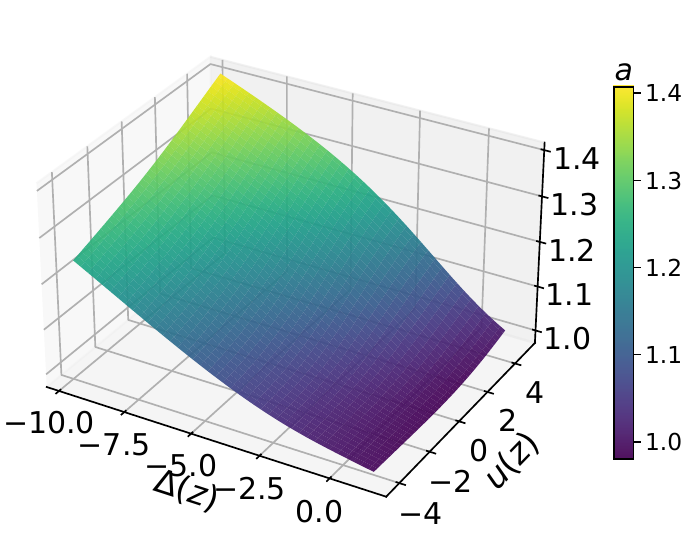} & \includegraphics[width=0.31\linewidth]{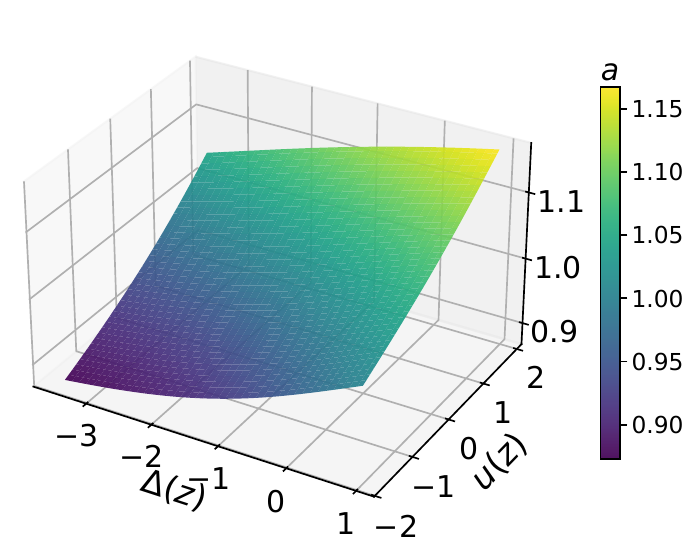} \\
\multicolumn{1}{c}{$a(z;\Theta)$, 20\%} & \multicolumn{1}{c}{$a(z;\Theta)$, 30\%} & \multicolumn{1}{c}{$a(z;\Theta)$, 40\%} \\
\addlinespace[0.35em]
\includegraphics[width=0.31\linewidth]{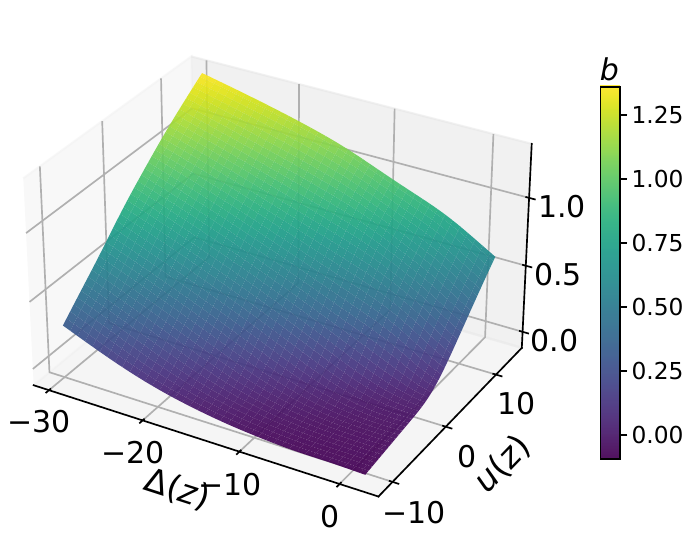} & \includegraphics[width=0.31\linewidth]{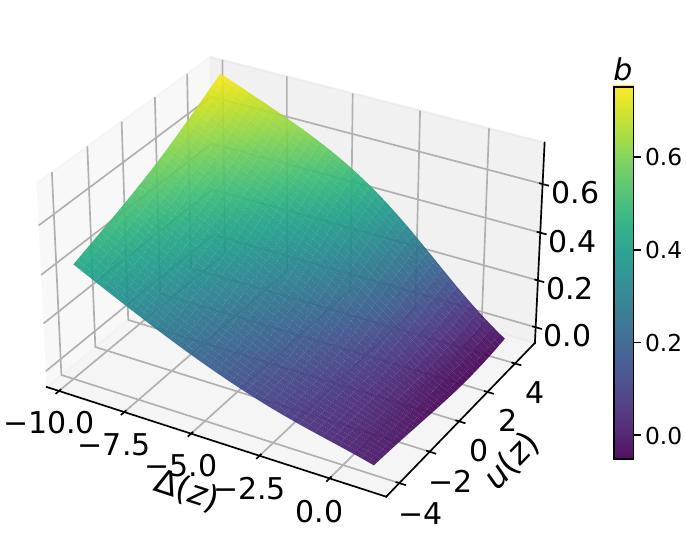} & \includegraphics[width=0.31\linewidth]{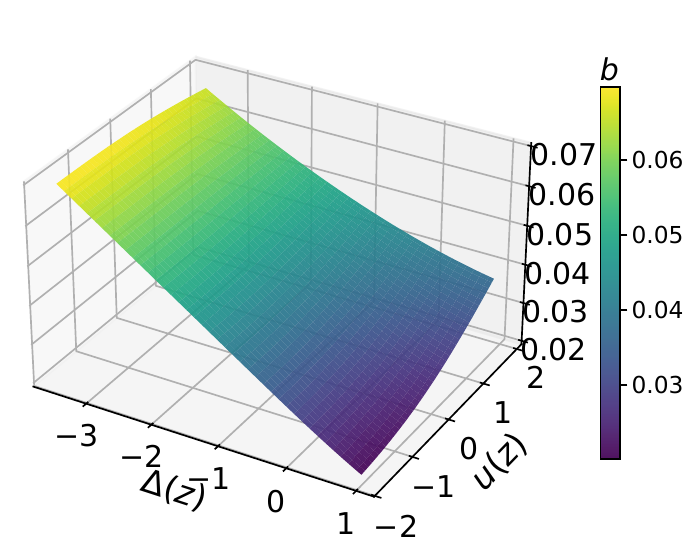} \\
\multicolumn{1}{c}{$b(z;\Theta)$, 20\%} & \multicolumn{1}{c}{$b(z;\Theta)$, 30\%} & \multicolumn{1}{c}{$b(z;\Theta)$, 40\%} \\
\addlinespace[0.35em]
\includegraphics[width=0.31\linewidth]{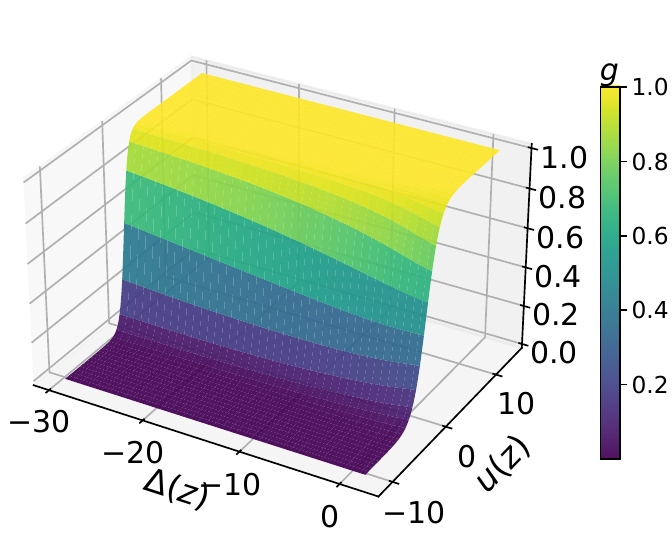} & \includegraphics[width=0.31\linewidth]{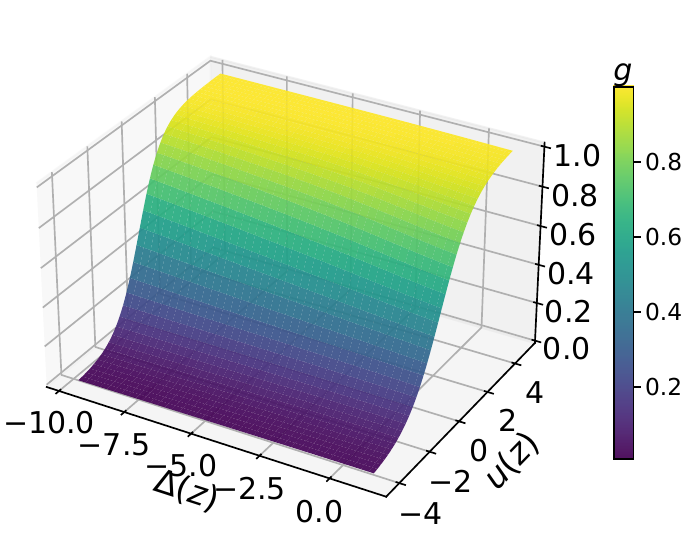} & \includegraphics[width=0.31\linewidth]{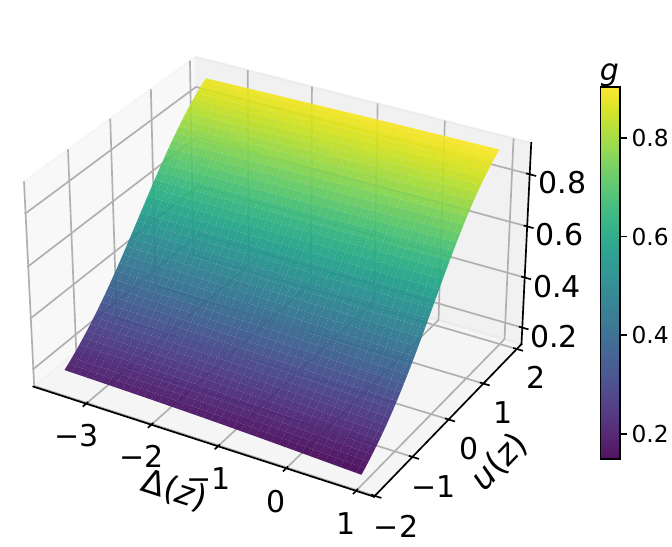} \\
\multicolumn{1}{c}{$g(z;\Theta)$, 20\%} & \multicolumn{1}{c}{$g(z;\Theta)$, 30\%} & \multicolumn{1}{c}{$g(z;\Theta)$, 40\%} \\
\addlinespace[0.35em]
\includegraphics[width=0.31\linewidth]{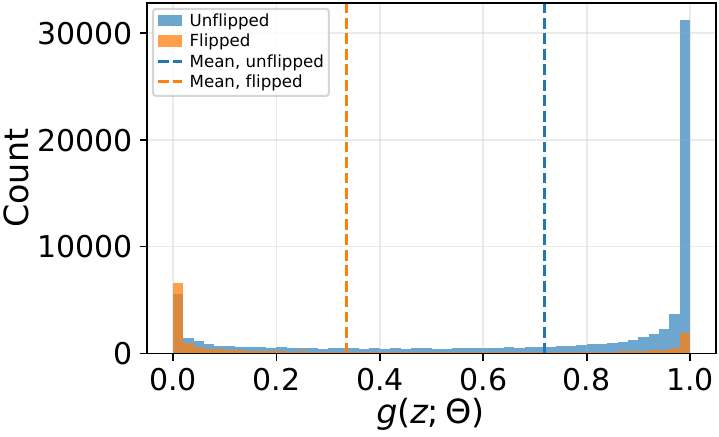} & \includegraphics[width=0.31\linewidth]{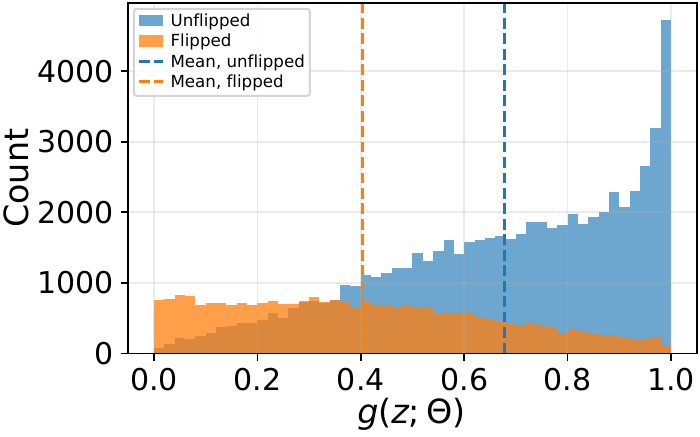} & \includegraphics[width=0.31\linewidth]{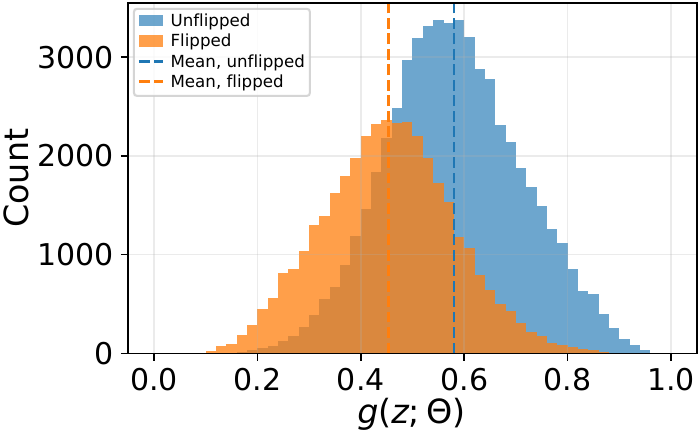} \\
\multicolumn{1}{c}{Weight distribution, 20\%} & \multicolumn{1}{c}{Weight distribution, 30\%} & \multicolumn{1}{c}{Weight distribution, 40\%} \\
\end{tabular}
\caption{VNet training dynamics on TL;DR at step 1000. Columns correspond to 20\%, 30\%, and 40\% random flips; rows correspond to $a(z;\Theta)$, $b(z;\Theta)$, $g(z;\Theta)$, and the learned-weight distributions of unflipped and flipped training pairs.}
\label{fig:app-tldr-step1000-vnet-surfaces}
\end{figure*}
\begin{figure*}[!p]
\centering
\small
\setlength{\tabcolsep}{2pt}
\begin{tabular}{ccc}
\includegraphics[width=0.31\linewidth]{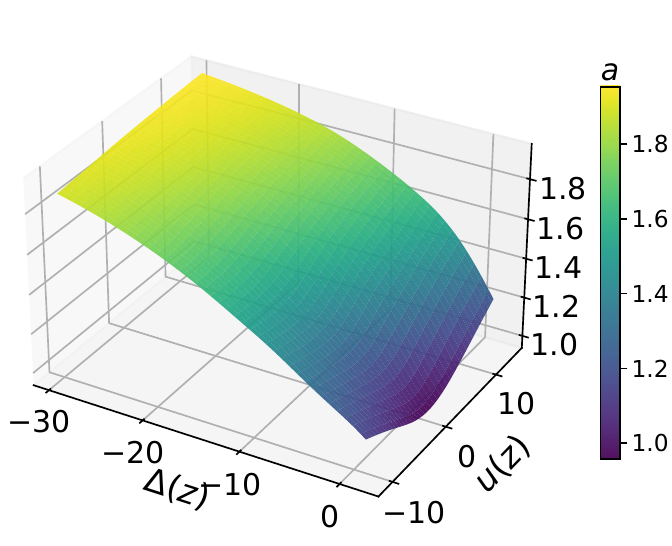} & \includegraphics[width=0.31\linewidth]{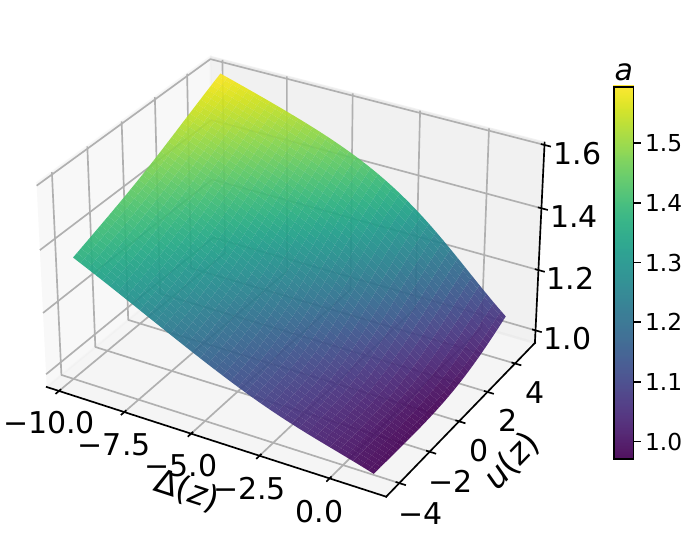} & \includegraphics[width=0.31\linewidth]{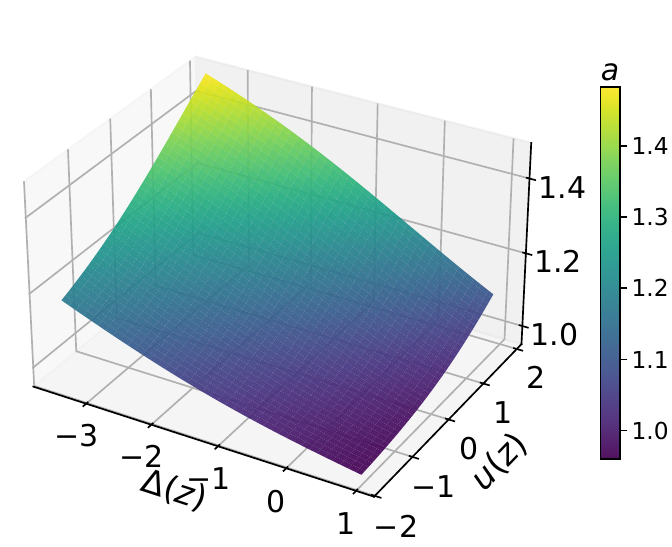} \\
\multicolumn{1}{c}{$a(z;\Theta)$, 20\%} & \multicolumn{1}{c}{$a(z;\Theta)$, 30\%} & \multicolumn{1}{c}{$a(z;\Theta)$, 40\%} \\
\addlinespace[0.35em]
\includegraphics[width=0.31\linewidth]{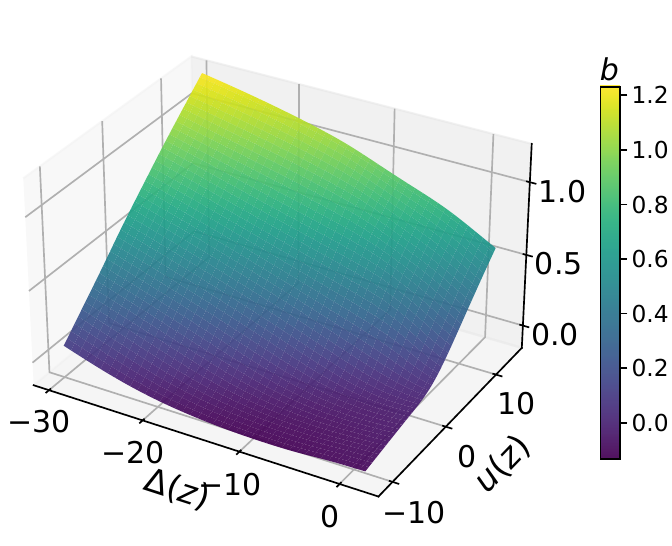} & \includegraphics[width=0.31\linewidth]{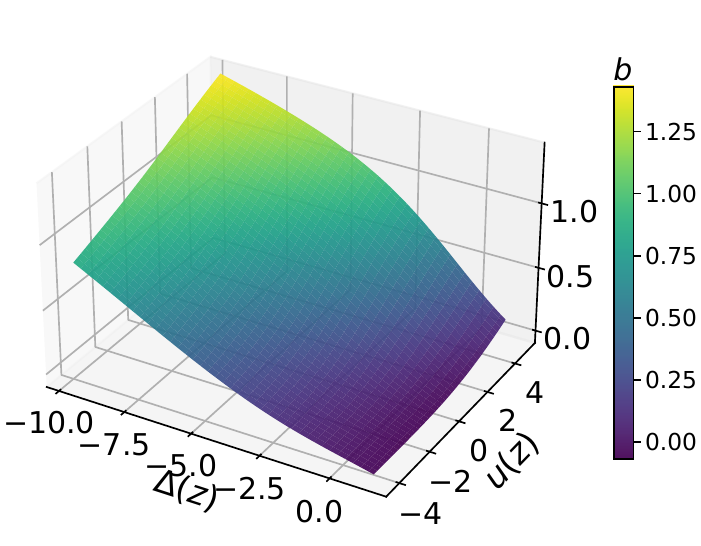} & \includegraphics[width=0.31\linewidth]{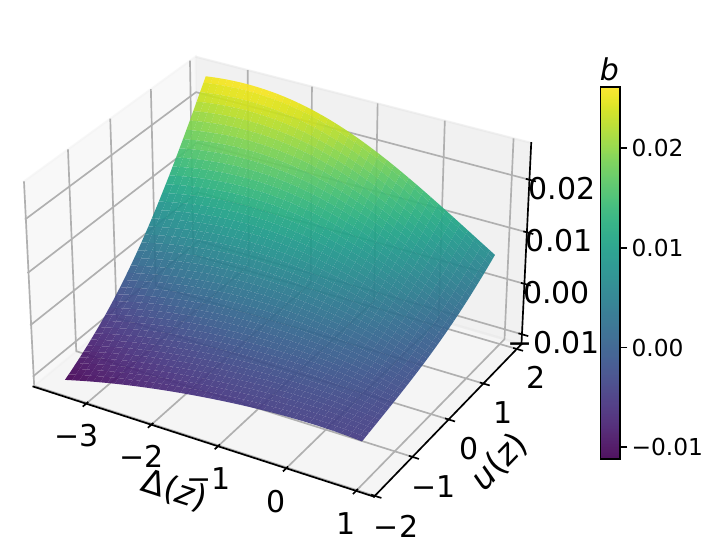} \\
\multicolumn{1}{c}{$b(z;\Theta)$, 20\%} & \multicolumn{1}{c}{$b(z;\Theta)$, 30\%} & \multicolumn{1}{c}{$b(z;\Theta)$, 40\%} \\
\addlinespace[0.35em]
\includegraphics[width=0.31\linewidth]{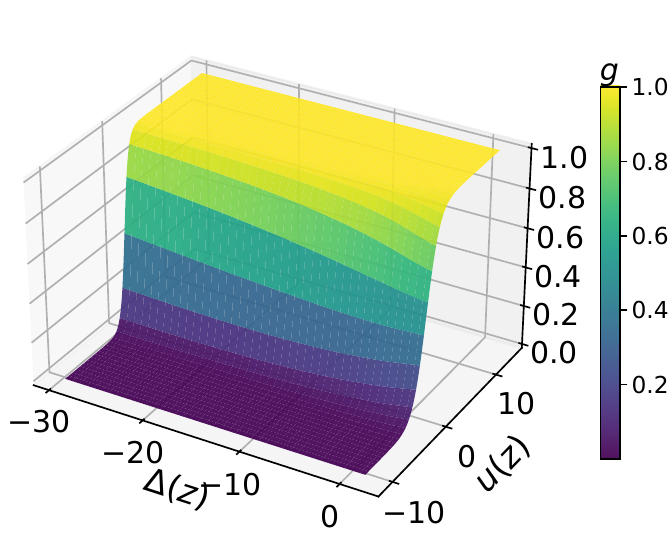} & \includegraphics[width=0.31\linewidth]{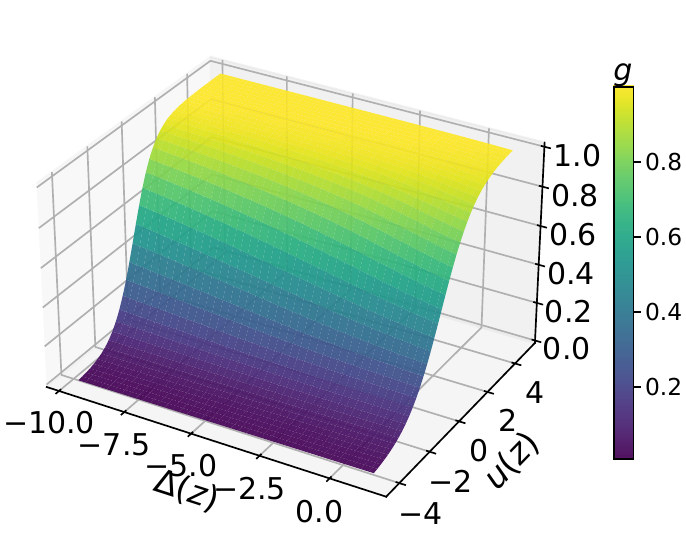} & \includegraphics[width=0.31\linewidth]{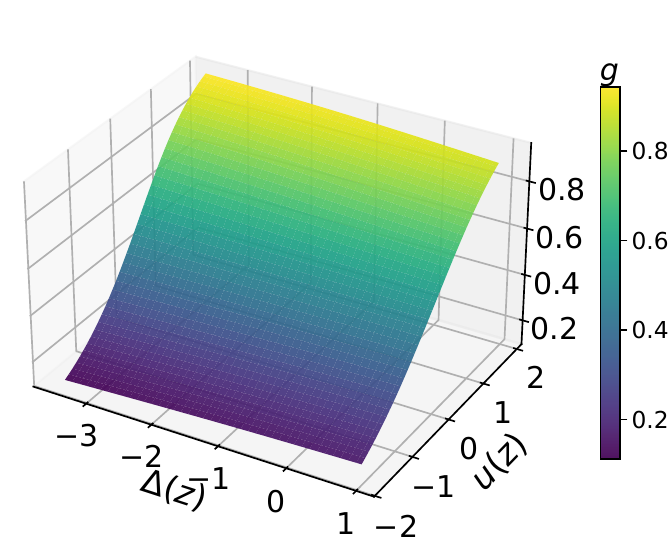} \\
\multicolumn{1}{c}{$g(z;\Theta)$, 20\%} & \multicolumn{1}{c}{$g(z;\Theta)$, 30\%} & \multicolumn{1}{c}{$g(z;\Theta)$, 40\%} \\
\addlinespace[0.35em]
\includegraphics[width=0.31\linewidth]{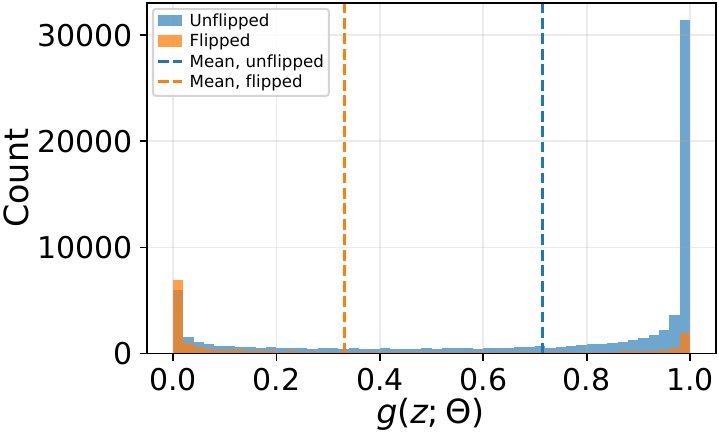} & \includegraphics[width=0.31\linewidth]{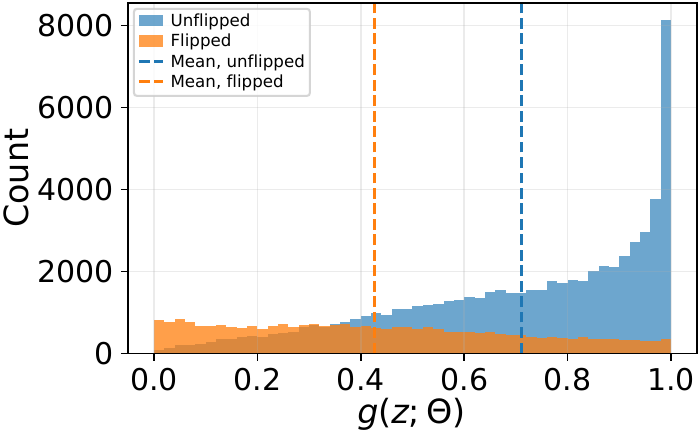} & \includegraphics[width=0.31\linewidth]{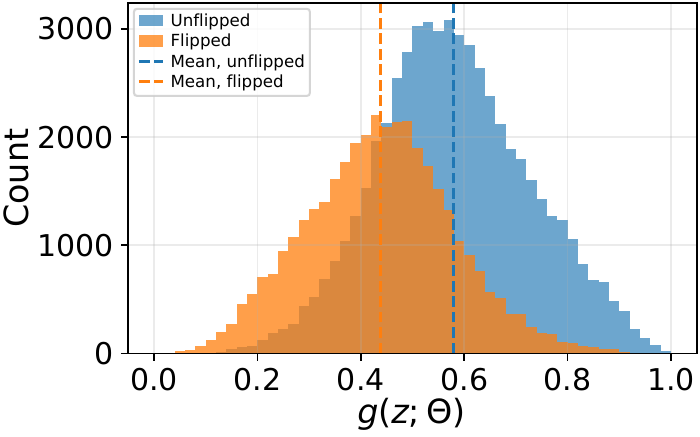} \\
\multicolumn{1}{c}{Weight distribution, 20\%} & \multicolumn{1}{c}{Weight distribution, 30\%} & \multicolumn{1}{c}{Weight distribution, 40\%} \\
\end{tabular}
\caption{VNet training dynamics on TL;DR at step 2000. Columns correspond to 20\%, 30\%, and 40\% random flips; rows correspond to $a(z;\Theta)$, $b(z;\Theta)$, $g(z;\Theta)$, and the learned-weight distributions of unflipped and flipped training pairs.}
\label{fig:app-tldr-step2000-vnet-surfaces}
\end{figure*}
\begin{figure*}[!p]
\centering
\small
\setlength{\tabcolsep}{2pt}
\begin{tabular}{ccc}
\includegraphics[width=0.31\linewidth]{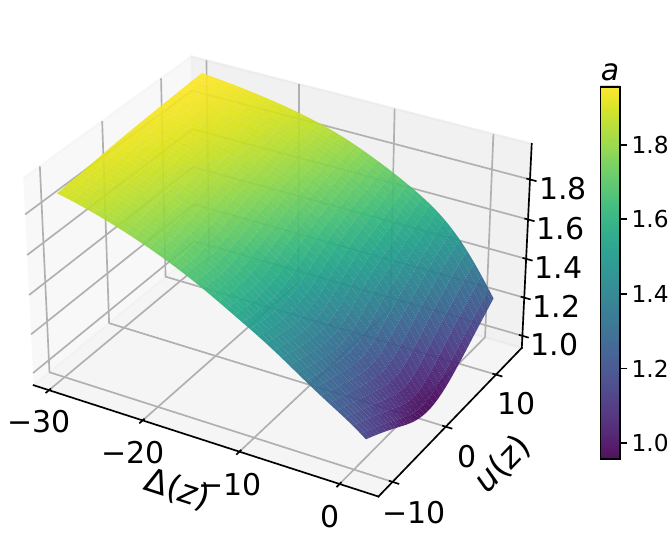} & \includegraphics[width=0.31\linewidth]{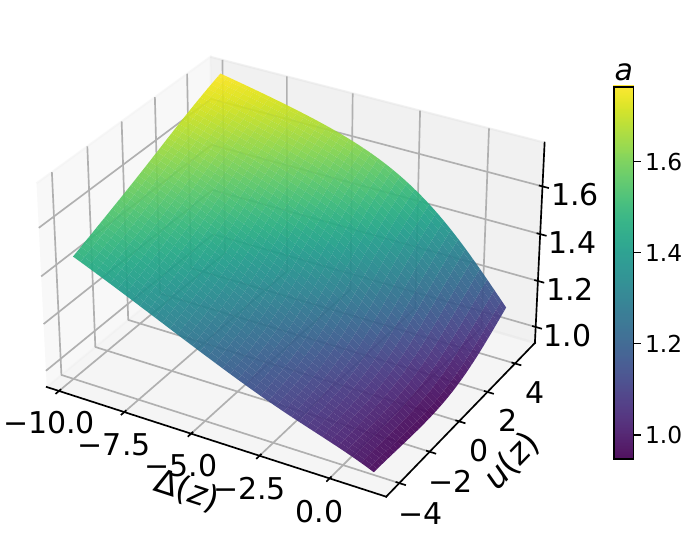} & \includegraphics[width=0.31\linewidth]{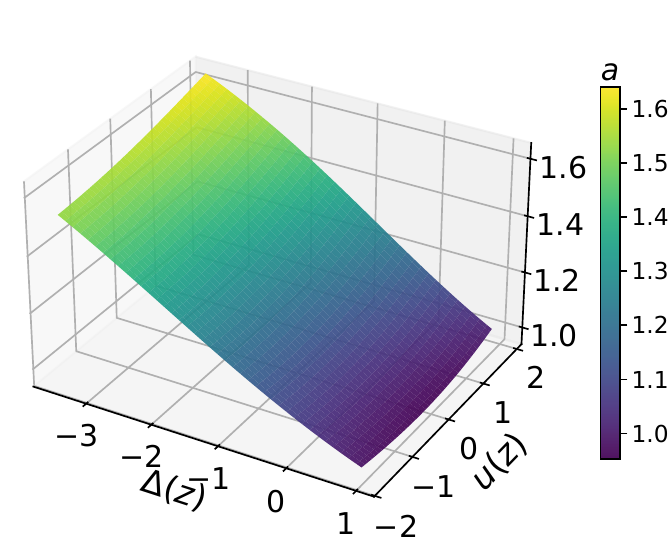} \\
\multicolumn{1}{c}{$a(z;\Theta)$, 20\%} & \multicolumn{1}{c}{$a(z;\Theta)$, 30\%} & \multicolumn{1}{c}{$a(z;\Theta)$, 40\%} \\
\addlinespace[0.35em]
\includegraphics[width=0.31\linewidth]{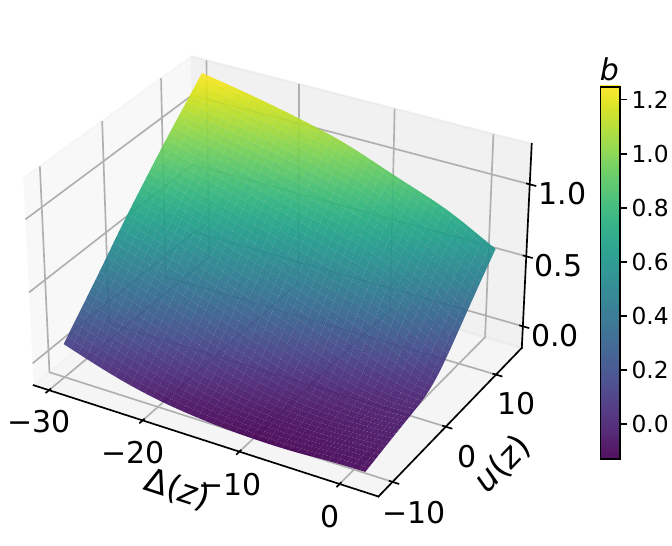} & \includegraphics[width=0.31\linewidth]{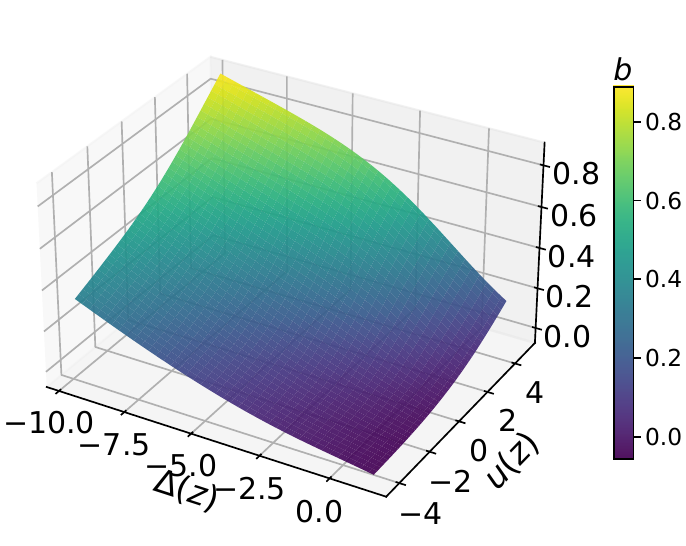} & \includegraphics[width=0.31\linewidth]{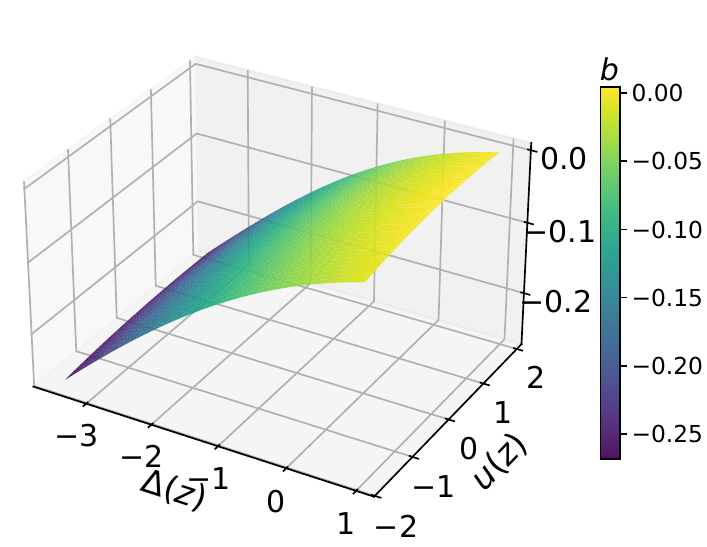} \\
\multicolumn{1}{c}{$b(z;\Theta)$, 20\%} & \multicolumn{1}{c}{$b(z;\Theta)$, 30\%} & \multicolumn{1}{c}{$b(z;\Theta)$, 40\%} \\
\addlinespace[0.35em]
\includegraphics[width=0.31\linewidth]{figures/vnet/tldr_rf20_step2901_function.pdf} & \includegraphics[width=0.31\linewidth]{figures/vnet/tldr_rf30_step2901_function.pdf} & \includegraphics[width=0.31\linewidth]{figures/vnet/tldr_rf40_step2901_function.pdf} \\
\multicolumn{1}{c}{$g(z;\Theta)$, 20\%} & \multicolumn{1}{c}{$g(z;\Theta)$, 30\%} & \multicolumn{1}{c}{$g(z;\Theta)$, 40\%} \\
\addlinespace[0.35em]
\includegraphics[width=0.31\linewidth]{figures/vnet/tldr_rf20_step2901_hist.pdf} & \includegraphics[width=0.31\linewidth]{figures/vnet/tldr_rf30_step2901_hist.pdf} & \includegraphics[width=0.31\linewidth]{figures/vnet/tldr_rf40_step2901_hist.pdf} \\
\multicolumn{1}{c}{Weight distribution, 20\%} & \multicolumn{1}{c}{Weight distribution, 30\%} & \multicolumn{1}{c}{Weight distribution, 40\%} \\
\end{tabular}
\caption{VNet training dynamics on TL;DR at step 2901. Columns correspond to 20\%, 30\%, and 40\% random flips; rows correspond to $a(z;\Theta)$, $b(z;\Theta)$, $g(z;\Theta)$, and the learned-weight distributions of unflipped and flipped training pairs.}
\label{fig:app-tldr-step2901-vnet-surfaces}
\end{figure*}
\begin{figure*}[!p]
\centering
\small
\setlength{\tabcolsep}{2pt}
\begin{tabular}{ccc}
\includegraphics[width=0.31\linewidth]{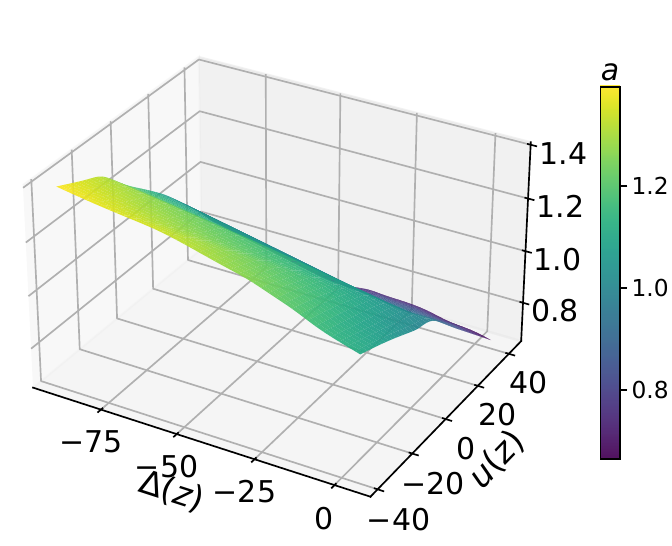} & \includegraphics[width=0.31\linewidth]{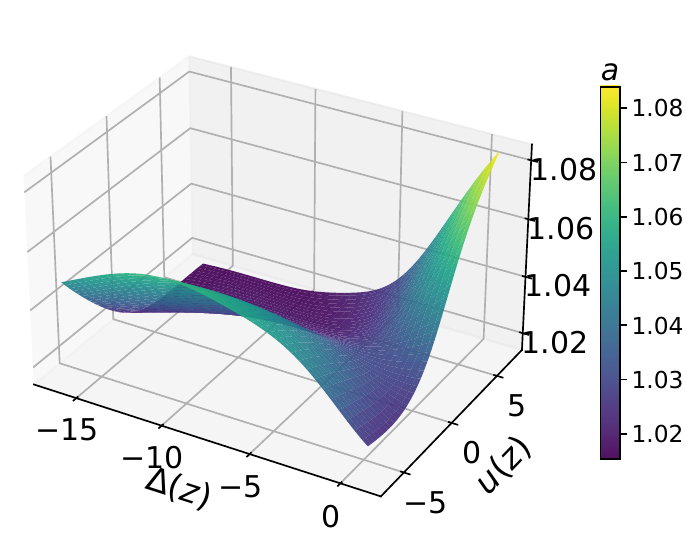} & \includegraphics[width=0.31\linewidth]{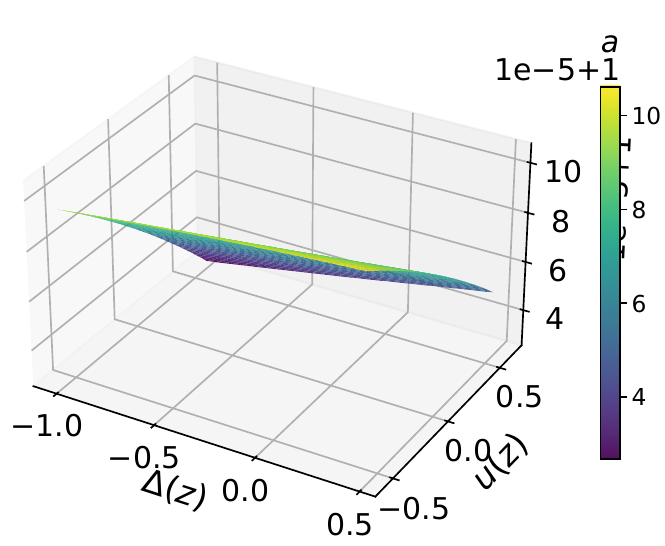} \\
\multicolumn{1}{c}{$a(z;\Theta)$, 20\%} & \multicolumn{1}{c}{$a(z;\Theta)$, 30\%} & \multicolumn{1}{c}{$a(z;\Theta)$, 40\%} \\
\addlinespace[0.35em]
\includegraphics[width=0.31\linewidth]{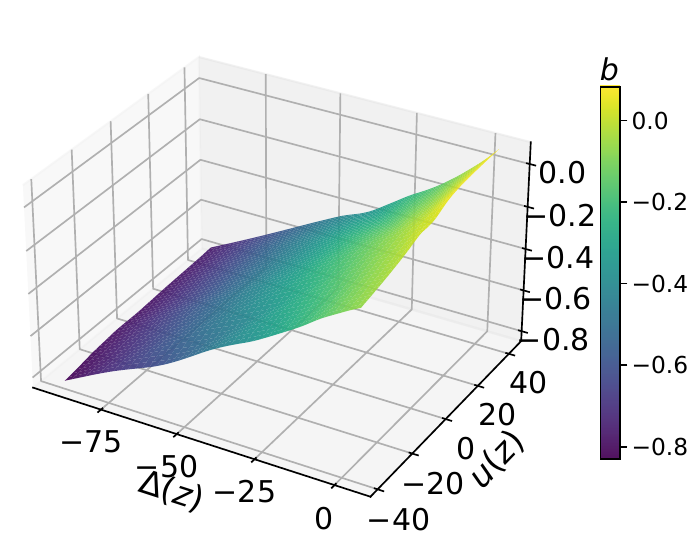} & \includegraphics[width=0.31\linewidth]{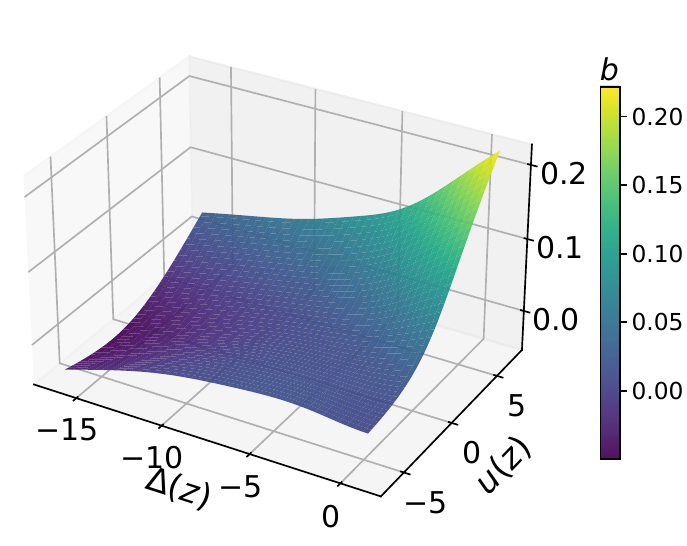} & \includegraphics[width=0.31\linewidth]{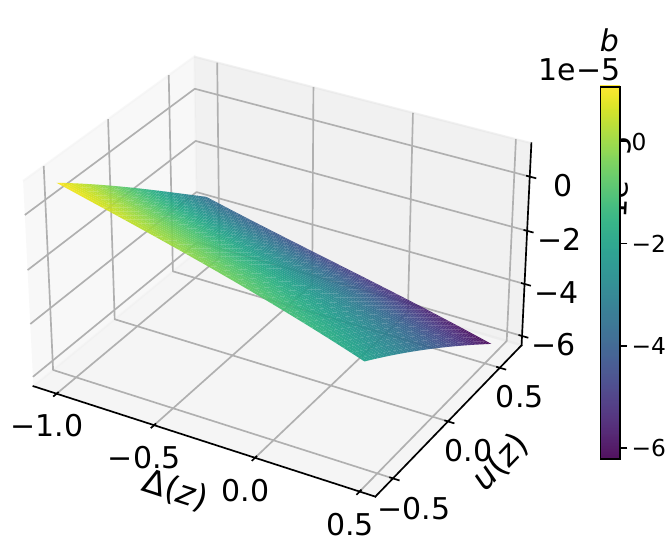} \\
\multicolumn{1}{c}{$b(z;\Theta)$, 20\%} & \multicolumn{1}{c}{$b(z;\Theta)$, 30\%} & \multicolumn{1}{c}{$b(z;\Theta)$, 40\%} \\
\addlinespace[0.35em]
\includegraphics[width=0.31\linewidth]{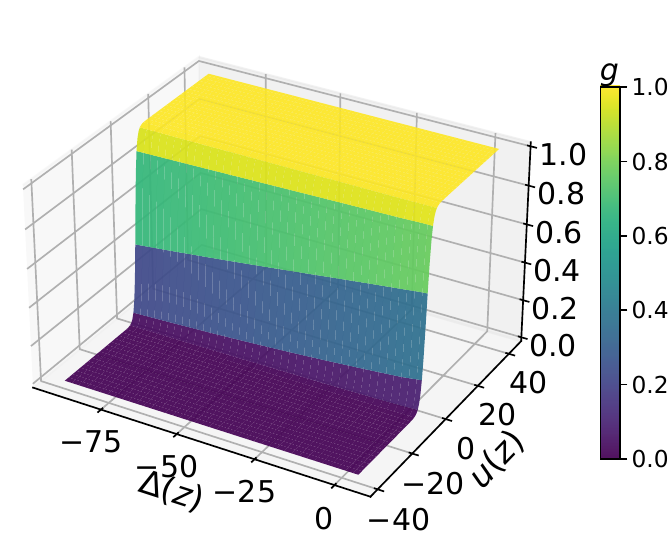} & \includegraphics[width=0.31\linewidth]{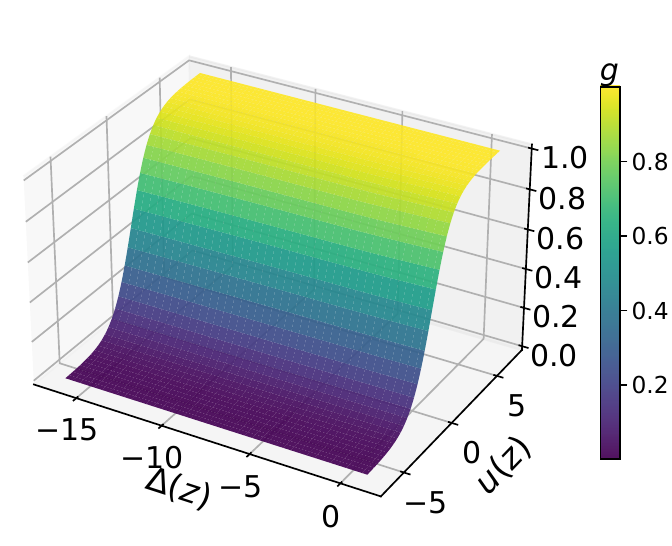} & \includegraphics[width=0.31\linewidth]{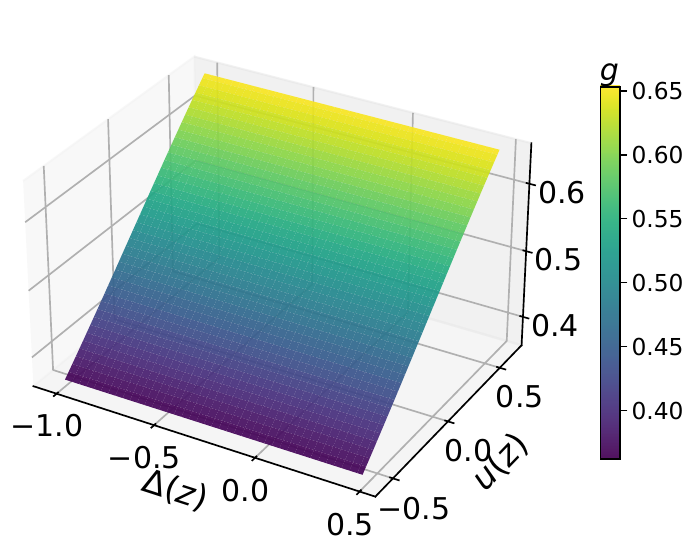} \\
\multicolumn{1}{c}{$g(z;\Theta)$, 20\%} & \multicolumn{1}{c}{$g(z;\Theta)$, 30\%} & \multicolumn{1}{c}{$g(z;\Theta)$, 40\%} \\
\addlinespace[0.35em]
\includegraphics[width=0.31\linewidth]{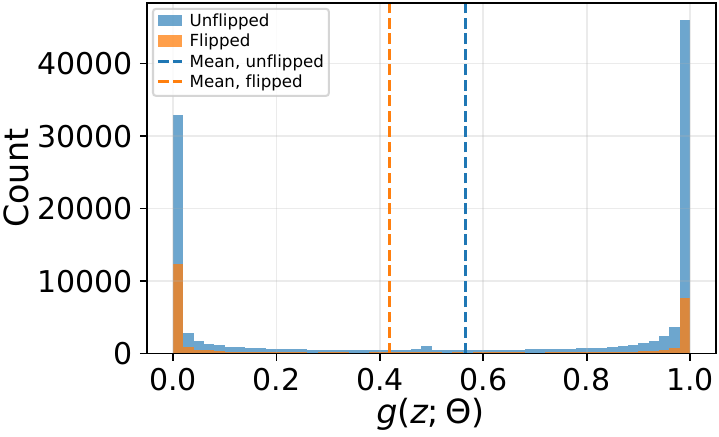} & \includegraphics[width=0.31\linewidth]{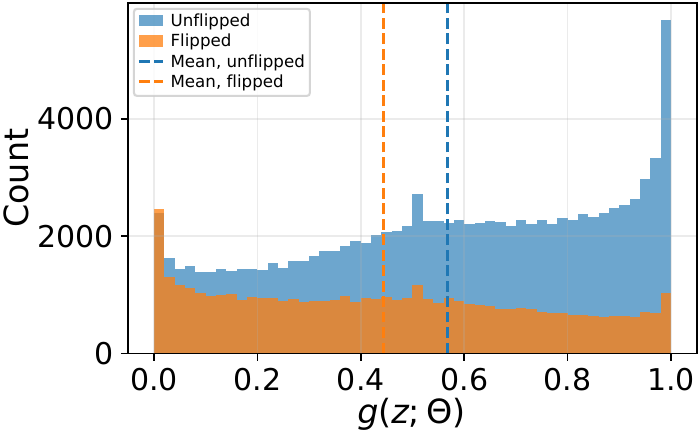} & \includegraphics[width=0.31\linewidth]{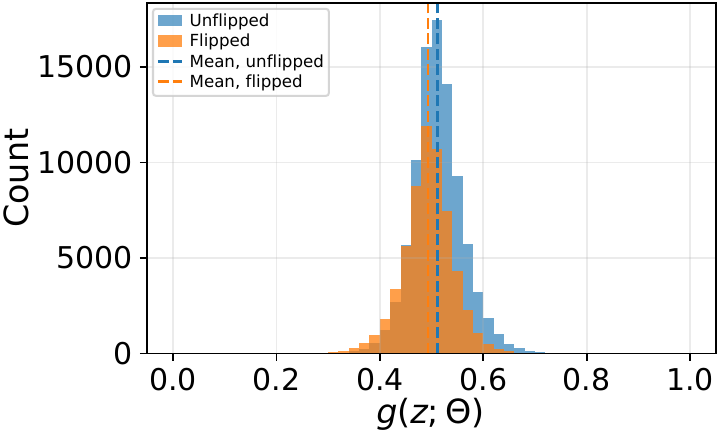} \\
\multicolumn{1}{c}{Weight distribution, 20\%} & \multicolumn{1}{c}{Weight distribution, 30\%} & \multicolumn{1}{c}{Weight distribution, 40\%} \\
\end{tabular}
\caption{VNet training dynamics on Anthropic HH at step 1000. Columns correspond to 20\%, 30\%, and 40\% random flips; rows correspond to $a(z;\Theta)$, $b(z;\Theta)$, $g(z;\Theta)$, and the learned-weight distributions of unflipped and flipped training pairs.}
\label{fig:app-hh-step1000-vnet-surfaces}
\end{figure*}
\begin{figure*}[!p]
\centering
\small
\setlength{\tabcolsep}{2pt}
\begin{tabular}{ccc}
\includegraphics[width=0.31\linewidth]{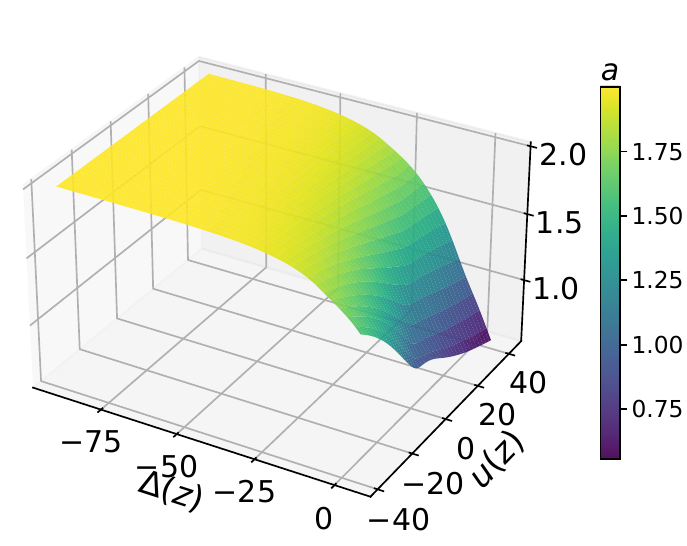} & \includegraphics[width=0.31\linewidth]{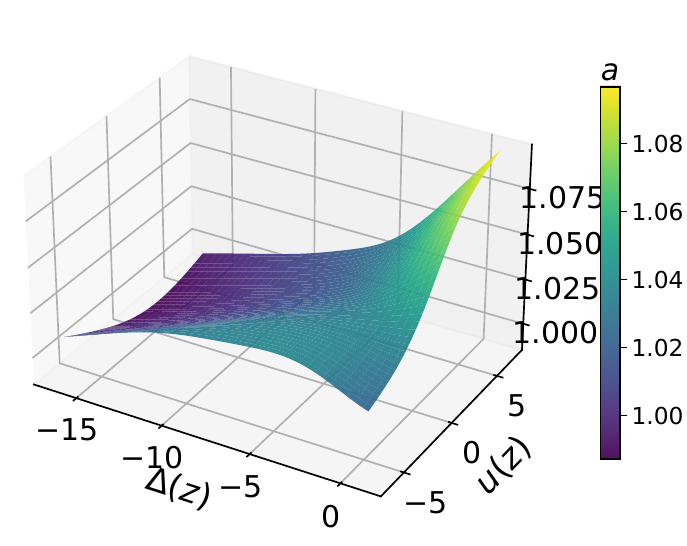} & \includegraphics[width=0.31\linewidth]{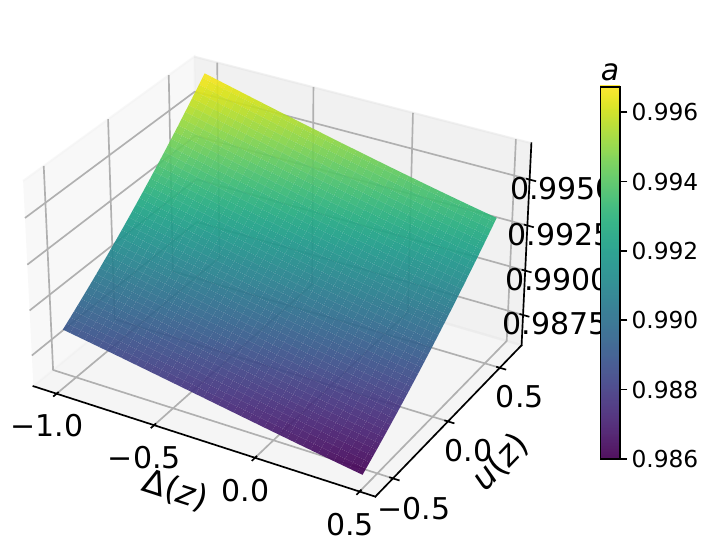} \\
\multicolumn{1}{c}{$a(z;\Theta)$, 20\%} & \multicolumn{1}{c}{$a(z;\Theta)$, 30\%} & \multicolumn{1}{c}{$a(z;\Theta)$, 40\%} \\
\addlinespace[0.35em]
\includegraphics[width=0.31\linewidth]{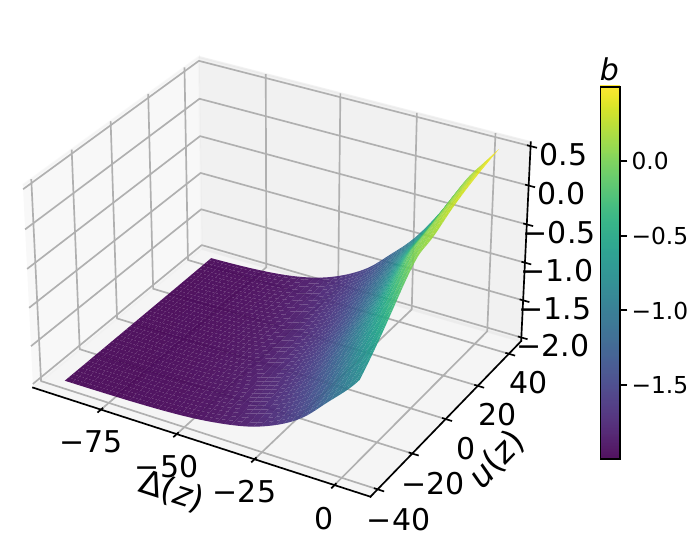} & \includegraphics[width=0.31\linewidth]{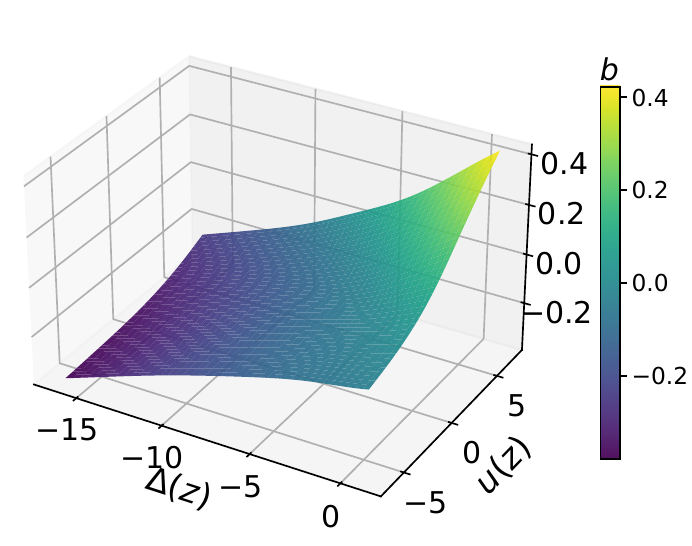} & \includegraphics[width=0.31\linewidth]{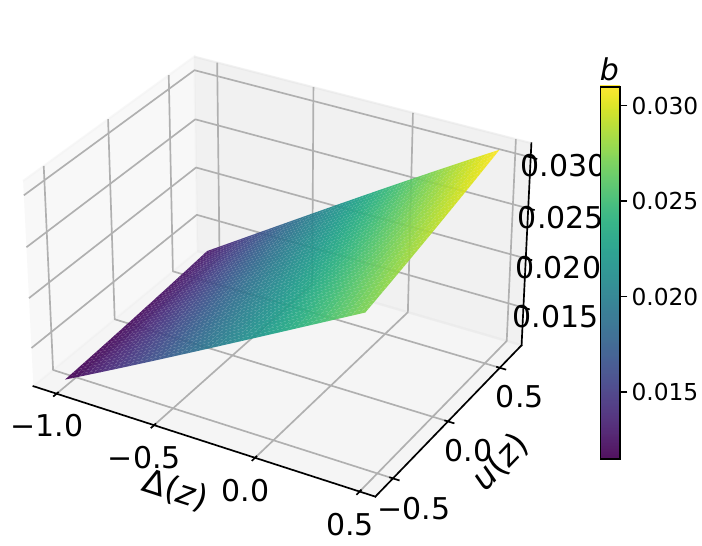} \\
\multicolumn{1}{c}{$b(z;\Theta)$, 20\%} & \multicolumn{1}{c}{$b(z;\Theta)$, 30\%} & \multicolumn{1}{c}{$b(z;\Theta)$, 40\%} \\
\addlinespace[0.35em]
\includegraphics[width=0.31\linewidth]{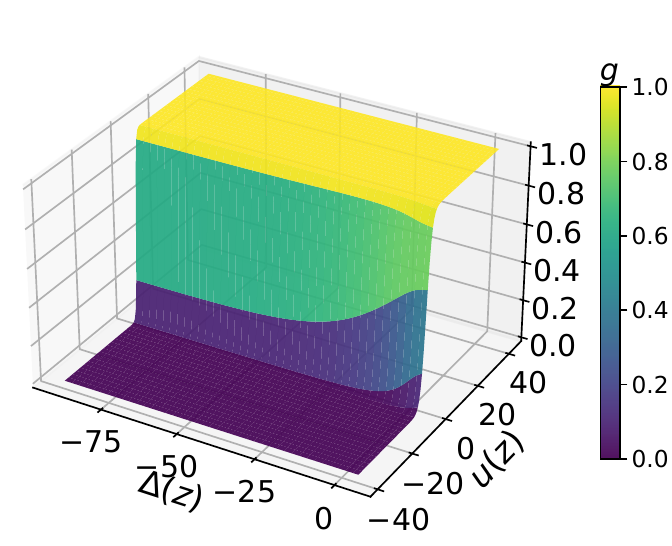} & \includegraphics[width=0.31\linewidth]{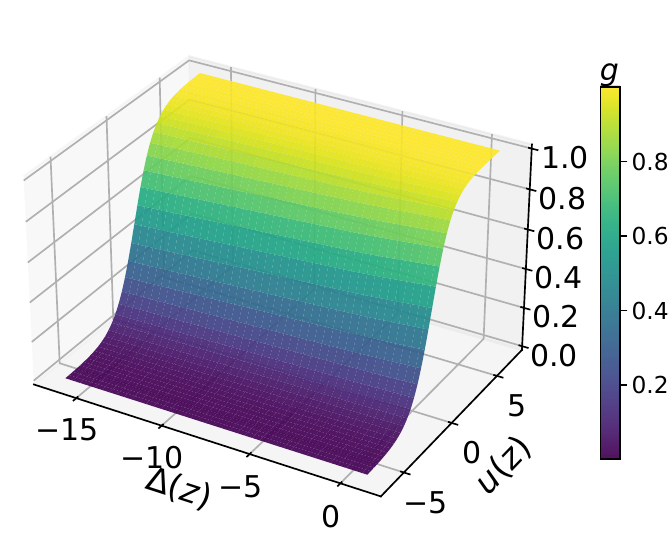} & \includegraphics[width=0.31\linewidth]{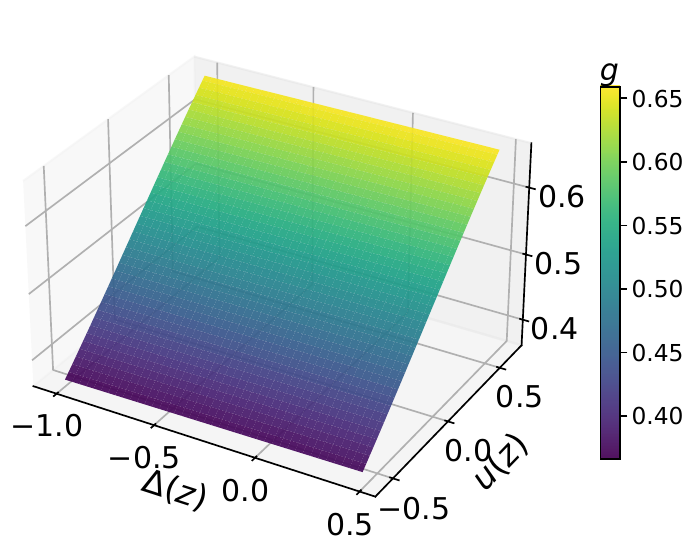} \\
\multicolumn{1}{c}{$g(z;\Theta)$, 20\%} & \multicolumn{1}{c}{$g(z;\Theta)$, 30\%} & \multicolumn{1}{c}{$g(z;\Theta)$, 40\%} \\
\addlinespace[0.35em]
\includegraphics[width=0.31\linewidth]{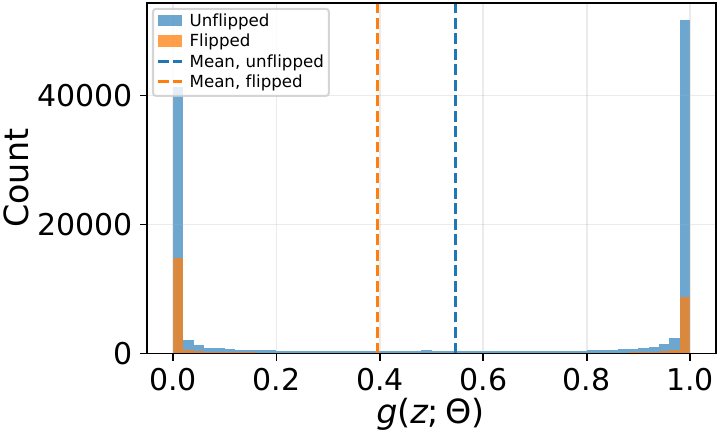} & \includegraphics[width=0.31\linewidth]{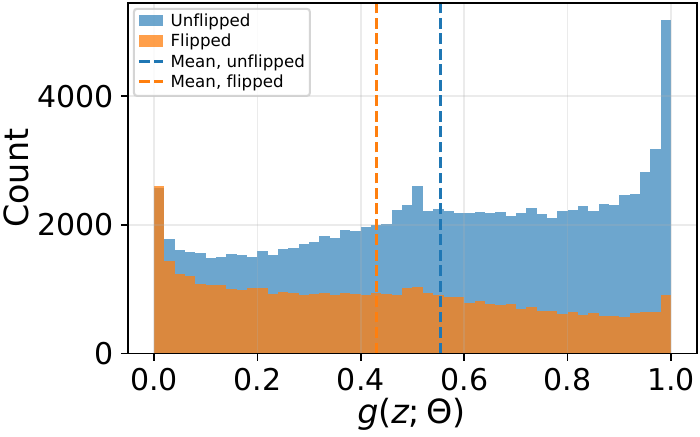} & \includegraphics[width=0.31\linewidth]{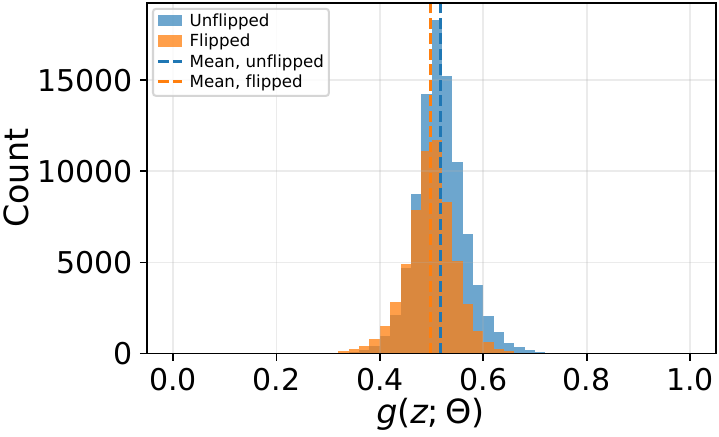} \\
\multicolumn{1}{c}{Weight distribution, 20\%} & \multicolumn{1}{c}{Weight distribution, 30\%} & \multicolumn{1}{c}{Weight distribution, 40\%} \\
\end{tabular}
\caption{VNet training dynamics on Anthropic HH at step 2000. Columns correspond to 20\%, 30\%, and 40\% random flips; rows correspond to $a(z;\Theta)$, $b(z;\Theta)$, $g(z;\Theta)$, and the learned-weight distributions of unflipped and flipped training pairs.}
\label{fig:app-hh-step2000-vnet-surfaces}
\end{figure*}
\begin{figure*}[!p]
\centering
\small
\setlength{\tabcolsep}{2pt}
\begin{tabular}{ccc}
\includegraphics[width=0.31\linewidth]{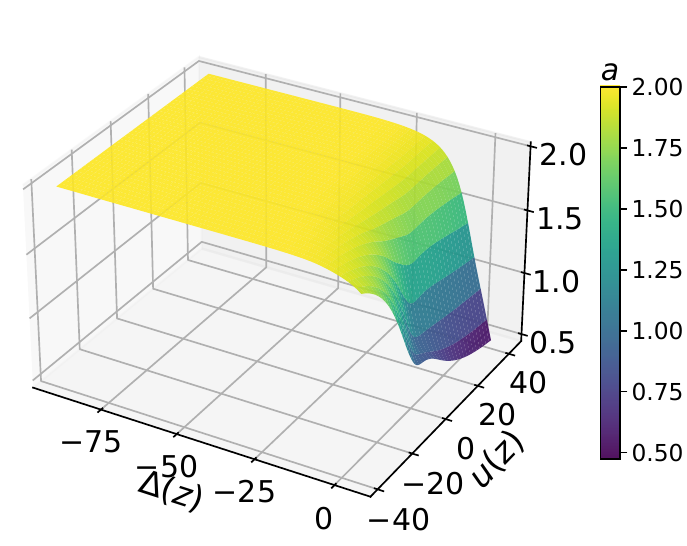} & \includegraphics[width=0.31\linewidth]{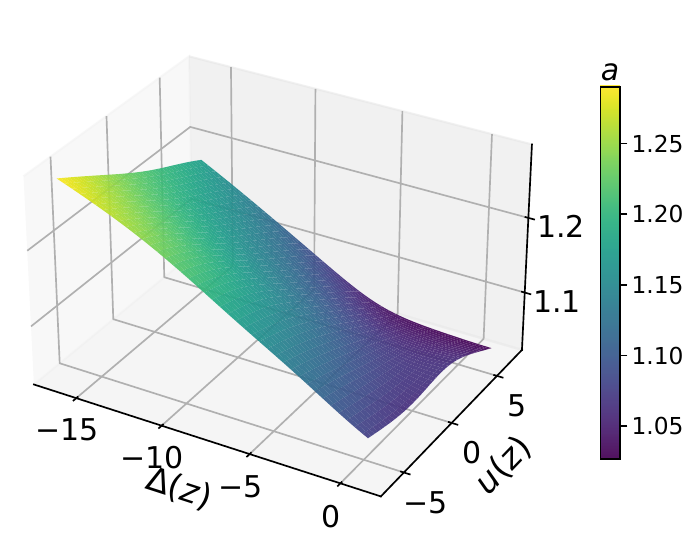} & \includegraphics[width=0.31\linewidth]{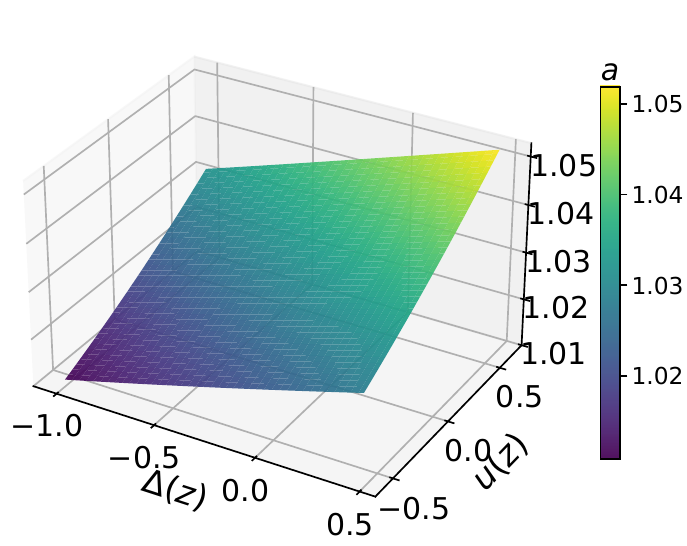} \\
\multicolumn{1}{c}{$a(z;\Theta)$, 20\%} & \multicolumn{1}{c}{$a(z;\Theta)$, 30\%} & \multicolumn{1}{c}{$a(z;\Theta)$, 40\%} \\
\addlinespace[0.35em]
\includegraphics[width=0.31\linewidth]{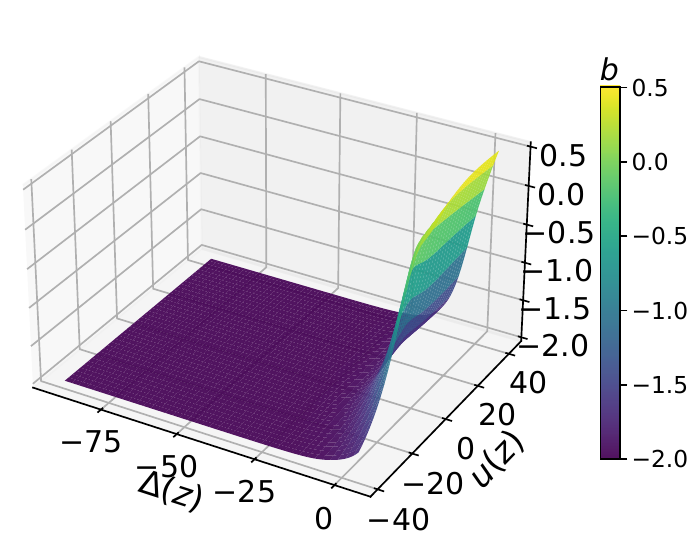} & \includegraphics[width=0.31\linewidth]{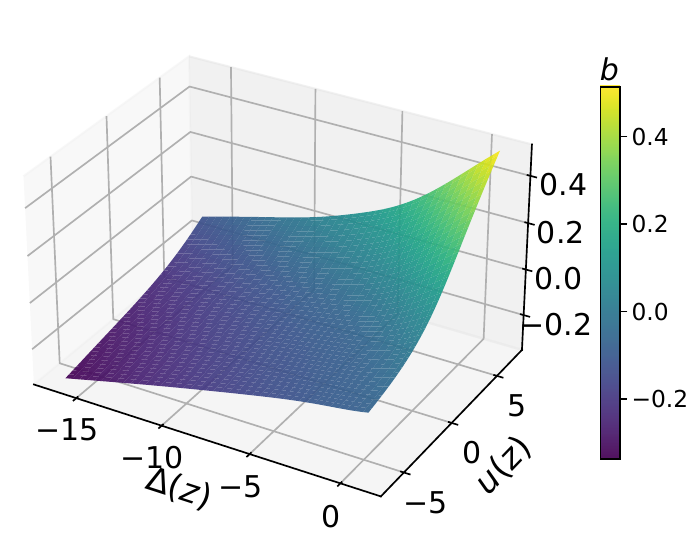} & \includegraphics[width=0.31\linewidth]{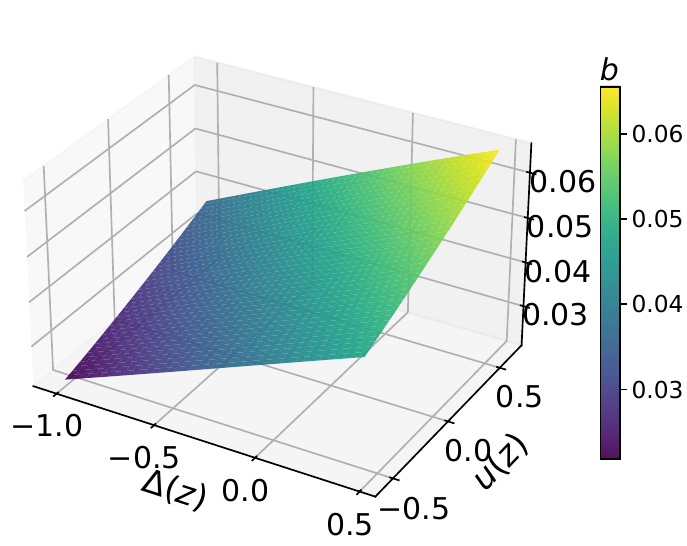} \\
\multicolumn{1}{c}{$b(z;\Theta)$, 20\%} & \multicolumn{1}{c}{$b(z;\Theta)$, 30\%} & \multicolumn{1}{c}{$b(z;\Theta)$, 40\%} \\
\addlinespace[0.35em]
\includegraphics[width=0.31\linewidth]{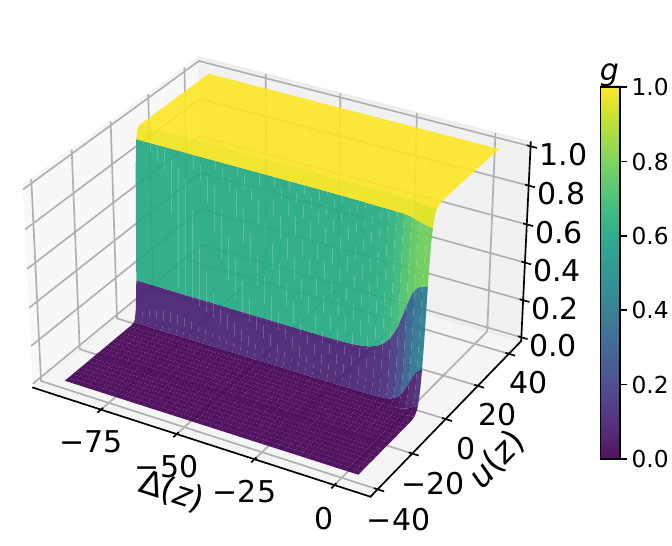} & \includegraphics[width=0.31\linewidth]{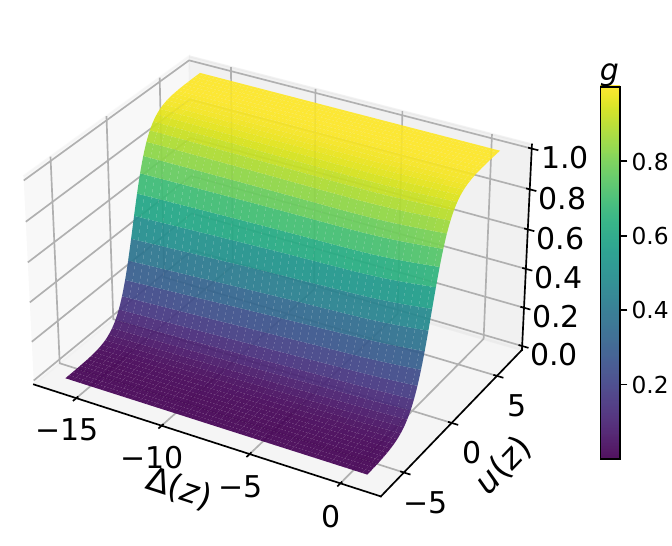} & \includegraphics[width=0.31\linewidth]{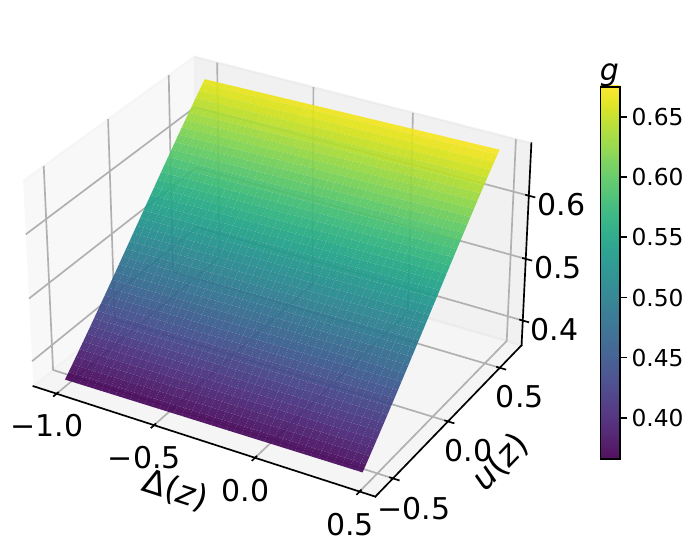} \\
\multicolumn{1}{c}{$g(z;\Theta)$, 20\%} & \multicolumn{1}{c}{$g(z;\Theta)$, 30\%} & \multicolumn{1}{c}{$g(z;\Theta)$, 40\%} \\
\addlinespace[0.35em]
\includegraphics[width=0.31\linewidth]{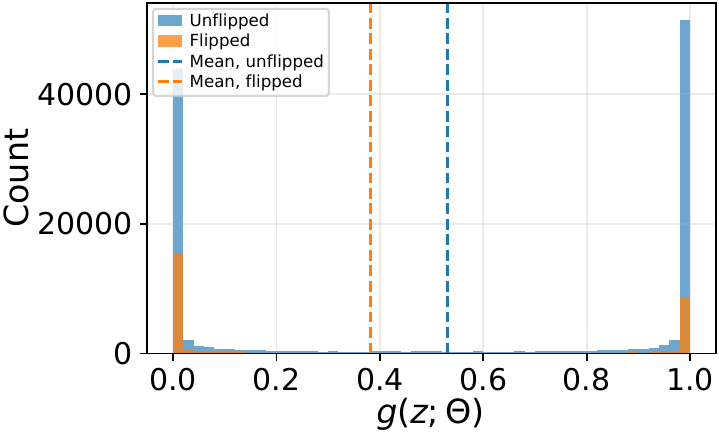} & \includegraphics[width=0.31\linewidth]{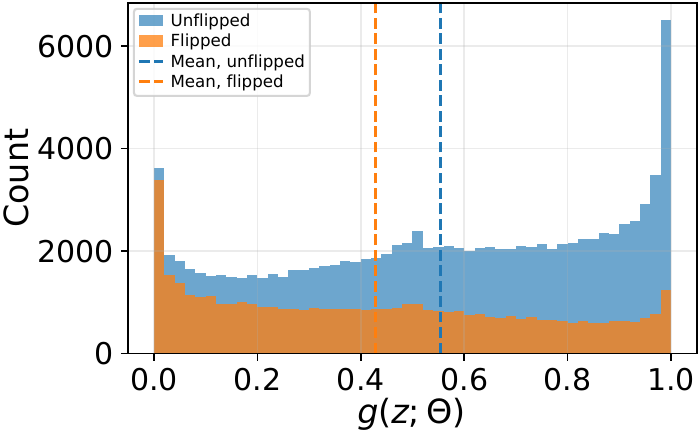} & \includegraphics[width=0.31\linewidth]{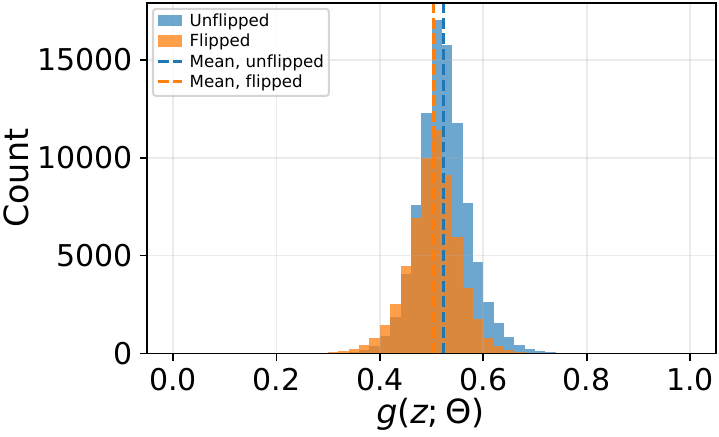} \\
\multicolumn{1}{c}{Weight distribution, 20\%} & \multicolumn{1}{c}{Weight distribution, 30\%} & \multicolumn{1}{c}{Weight distribution, 40\%} \\
\end{tabular}
\caption{VNet training dynamics on Anthropic HH at step 3000. Columns correspond to 20\%, 30\%, and 40\% random flips; rows correspond to $a(z;\Theta)$, $b(z;\Theta)$, $g(z;\Theta)$, and the learned-weight distributions of unflipped and flipped training pairs.}
\label{fig:app-hh-step3000-vnet-surfaces}
\end{figure*}
\begin{figure*}[!p]
\centering
\small
\setlength{\tabcolsep}{2pt}
\begin{tabular}{ccc}
\includegraphics[width=0.31\linewidth]{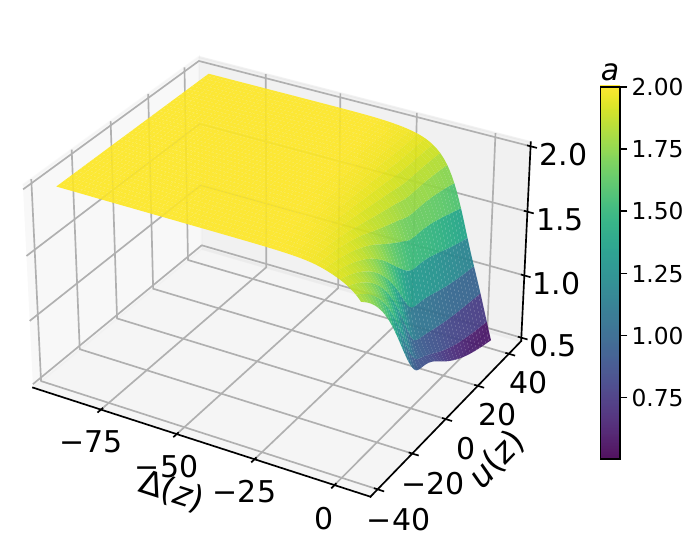} & \includegraphics[width=0.31\linewidth]{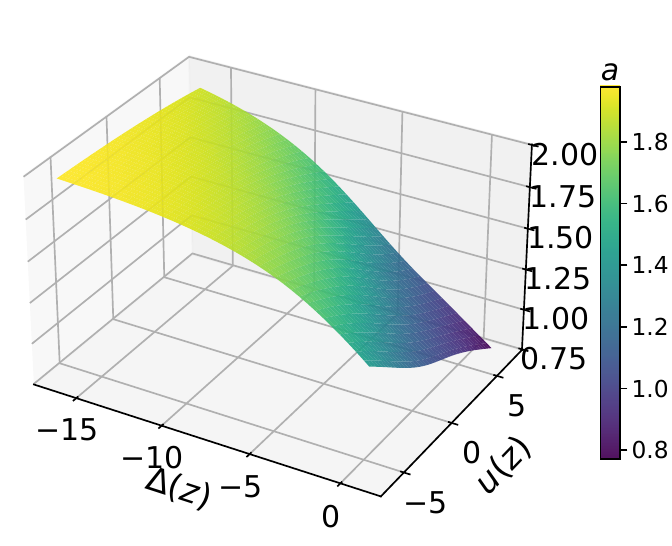} & \includegraphics[width=0.31\linewidth]{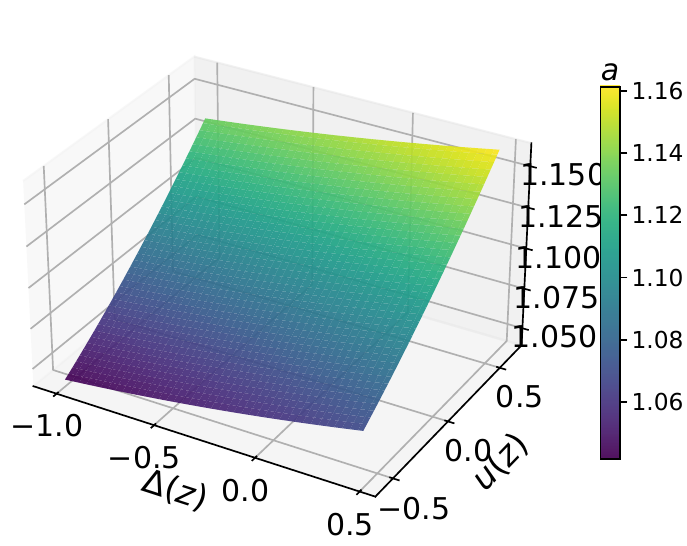} \\
\multicolumn{1}{c}{$a(z;\Theta)$, 20\%} & \multicolumn{1}{c}{$a(z;\Theta)$, 30\%} & \multicolumn{1}{c}{$a(z;\Theta)$, 40\%} \\
\addlinespace[0.35em]
\includegraphics[width=0.31\linewidth]{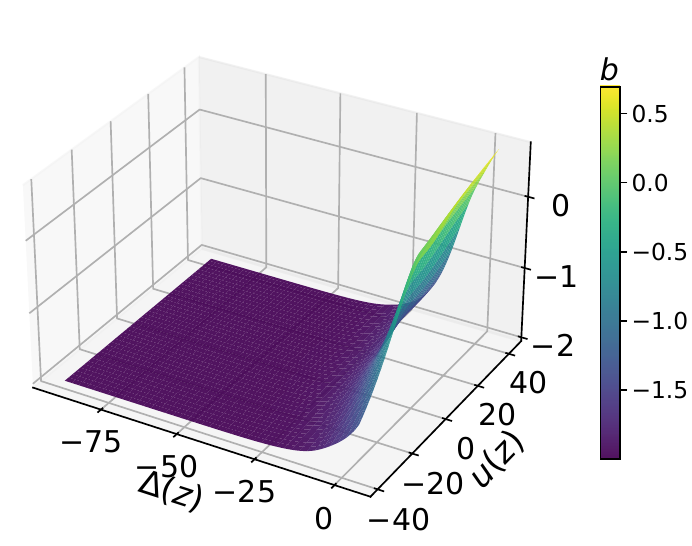} & \includegraphics[width=0.31\linewidth]{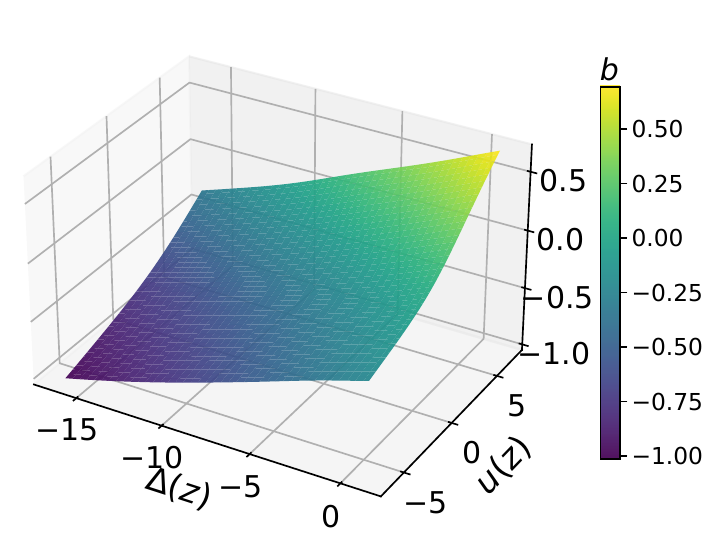} & \includegraphics[width=0.31\linewidth]{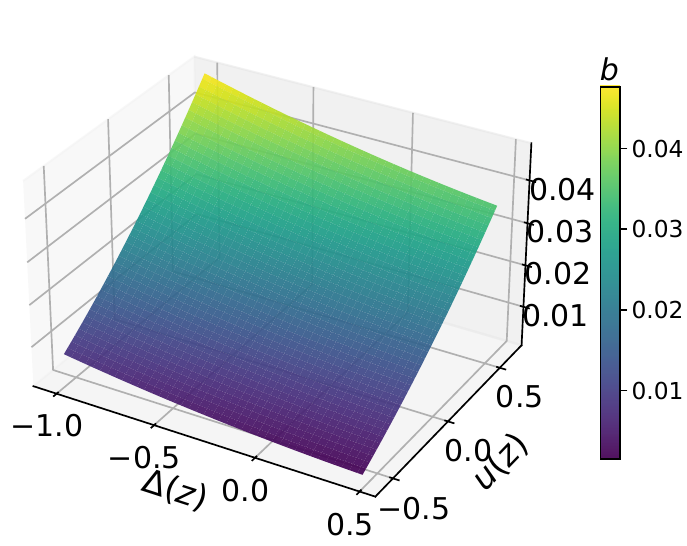} \\
\multicolumn{1}{c}{$b(z;\Theta)$, 20\%} & \multicolumn{1}{c}{$b(z;\Theta)$, 30\%} & \multicolumn{1}{c}{$b(z;\Theta)$, 40\%} \\
\addlinespace[0.35em]
\includegraphics[width=0.31\linewidth]{figures/vnet/hh_rf20_step4699_function.pdf} & \includegraphics[width=0.31\linewidth]{figures/vnet/hh_rf30_step4699_function.pdf} & \includegraphics[width=0.31\linewidth]{figures/vnet/hh_rf40_step4699_function.pdf} \\
\multicolumn{1}{c}{$g(z;\Theta)$, 20\%} & \multicolumn{1}{c}{$g(z;\Theta)$, 30\%} & \multicolumn{1}{c}{$g(z;\Theta)$, 40\%} \\
\addlinespace[0.35em]
\includegraphics[width=0.31\linewidth]{figures/vnet/hh_rf20_step4699_hist.pdf} & \includegraphics[width=0.31\linewidth]{figures/vnet/hh_rf30_step4699_hist.pdf} & \includegraphics[width=0.31\linewidth]{figures/vnet/hh_rf40_step4699_hist.pdf} \\
\multicolumn{1}{c}{Weight distribution, 20\%} & \multicolumn{1}{c}{Weight distribution, 30\%} & \multicolumn{1}{c}{Weight distribution, 40\%} \\
\end{tabular}
\caption{VNet training dynamics on Anthropic HH at step 4699. Columns correspond to 20\%, 30\%, and 40\% random flips; rows correspond to $a(z;\Theta)$, $b(z;\Theta)$, $g(z;\Theta)$, and the learned-weight distributions of unflipped and flipped training pairs.}
\label{fig:app-hh-step4699-vnet-surfaces}
\end{figure*}

\end{document}